\documentclass[11pt]{article}

\usepackage{geometry}
\usepackage{verbatim}
\usepackage{enumitem}
\usepackage{amsmath}
\usepackage{graphicx}
\usepackage{amsthm}
\usepackage{amssymb}
\usepackage{hyperref} 
\usepackage{color}
\usepackage{thmtools}
\usepackage{bbm}
\usepackage{amsthm}
\usepackage{tikz-cd}
\usepackage{subcaption}

\makeatletter
\renewcommand{\paragraph}{%
  \@startsection{paragraph}{4}%
  {\z@}{2.25ex \@plus 1ex \@minus .2ex}{-1em}%
  {\normalfont\normalsize\bfseries}%
}
\makeatother

\newcommand\restr[2]{{
  \left.\kern-\nulldelimiterspace 
  #1 
  \vphantom{\big|} 
  \right|_{#2} 
  }}

\newtheorem{theorem}{Theorem}[section]
\newtheorem{lemma}[theorem]{Lemma}

\newtheorem{proposition}[theorem]{Proposition} 
\newtheorem{corollary}[theorem]{Corollary} 

\theoremstyle{definition}
\newtheorem{definition}[theorem]{Definition}
\newtheorem{question}[theorem]{Question}

\theoremstyle{plain}

\DeclareMathOperator{\Var}{Var}
\DeclareMathOperator{\Span}{span}
\newcommand{\norm}[1]{\left \lVert #1 \right \rVert}

\mathchardef\mhyphen="2D

\newtheoremstyle{break}
  {\topsep}{\topsep}%
  {}{}%
  {\bfseries}{}%
  {\newline}{}%
\theoremstyle{break}

\DeclareMathOperator*{\argmin}{argmin}

\newcommand{\RR}{\mathbb{R}}

\newcommand{\NN}{\mathbb{N}}

\DeclareMathOperator{\supp}{supp}

\DeclareMathOperator{\sign}{sign}

\newcommand{\EE}{\mathbb{E}}

\DeclareMathOperator{\rspan}{rspan} 
\DeclareMathOperator{\vspan}{span} 
\DeclareMathOperator{\Proj}{Proj}

\newcommand{\Sigmahat}{\hat{\Sigma}}
\newcommand{\EEE}{\hat{\mathbb{E}}}

\DeclareMathOperator{\dist}{dist}

\newcommand{\pinv}{\dagger}

\newcommand{\Thetat}{\tilde{\Theta}}
\DeclareMathOperator*{\argmax}{arg\,max}
\DeclareMathOperator{\cond}{cond}
\DeclareMathOperator{\poly}{poly}

\DeclareMathOperator{\tw}{tw}
\DeclareMathOperator{\bone}{\mathbbm{1}}

\newcommand{\note}[1]{{\color{red} #1}}

\theoremstyle{remark}
\newtheorem{remark}{Remark}
\newtheorem{example}{Example}
\newcommand{\ignore}[1]{{}}

\title{On the Power of Preconditioning in Sparse Linear Regression}
\author{Jonathan A. Kelner\thanks{\texttt{kelner@mit.edu}. This work was supported in part by NSF Large CCF-1565235, NSF Medium CCF-1955217, and NSF TRIPODS 1740751.} \\ MIT \and Frederic Koehler\thanks{\texttt{fkoehler@mit.edu}. This work was supported in part by NSF CAREER Award CCF-1453261, NSF Large CCF-1565235, A. Moitra's ONR Young Investigator Award and E. Mossel's Vannevar Bush Faculty Fellowship ONR-N00014-20-1-2826.} \\ MIT \and Raghu Meka\thanks{\texttt{raghum@cs.ucla.edu}. This work was supported in part by NSF CAREER Award CCF-1553605 and NSF Small CCF-2007682} \\ UCLA \and Dhruv Rohatgi\thanks{\texttt{drohatgi@mit.edu}. This work was supported in part by NSF Large CCF-1565235, NSF Medium CCF-1955217, and the MIT UROP Office.} \\ MIT}

\begin{document}

\maketitle
\begin{abstract}
Sparse linear regression is a fundamental problem in high-dimensional statistics, but strikingly little is known about how to efficiently solve it without restrictive conditions on the \emph{design matrix}. We consider the (correlated) random design setting, where the covariates are independently drawn from a multivariate Gaussian $N(0,\Sigma)$, for some $n \times n$ positive semi-definite matrix $\Sigma$, and seek estimators $\hat{w}$ minimizing $(\hat{w}-w^*)^T\Sigma(\hat{w}-w^*)$, where $w^*$ is the $k$-sparse ground truth. Information theoretically, one can achieve strong error bounds with only $O(k \log n)$ samples for arbitrary $\Sigma$ and $w^*$; however, no efficient algorithms are known to match these guarantees even with $o(n)$ samples, without further assumptions on $\Sigma$ or $w^*$.

Yet there is little evidence for this gap in the random design setting: computational lower bounds are only known for worst-case design matrices. To date, random-design instances (i.e. specific covariance matrices $\Sigma$) have only been proven hard against the Lasso program and variants. More precisely, these ``hard" instances can often be solved by Lasso after a simple change-of-basis (i.e. preconditioning).

In this work, we give both upper and lower bounds clarifying the power of preconditioning as a tool for solving sparse linear regression problems. On the one hand, we show that the preconditioned Lasso can solve a large class of sparse linear regression problems nearly optimally: it succeeds whenever the \emph{dependency structure} of the covariates, in the sense of the Markov property, has low treewidth --- even if $\Sigma$ is highly ill-conditioned. This upper bound builds on ideas from the wavelet and signal processing literature. As a special case of this result, we give an algorithm for sparse linear regression with covariates from an autoregressive time series model, where we also show that the (usual) Lasso provably fails. 

On the other hand, we construct (for the first time) random-design instances which are provably hard even for an optimally preconditioned Lasso. In fact, we complete our treewidth classification by proving that for \emph{any} treewidth-$t$ graph, there exists a Gaussian Markov Random Field on this graph such that the preconditioned Lasso, with any choice of preconditioner, requires $\Omega(t^{1/20})$ samples to recover $O(\log n)$-sparse signals when covariates are drawn from this model.
\end{abstract}

\thispagestyle{empty}
\newpage
\pagenumbering{arabic} 

\section{Introduction}\label{section:introduction}
In this paper, we study the fundamental statistical problem of \emph{sparse linear regression with (correlated) random design}. In the simplest form of this problem, the learning algorithm is given access to $m$ independent and identically distributed samples $(X_1,Y_1),\ldots,(X_m,Y_m)$ of the form
\begin{equation}\label{eqn:linear-model}
Y_i = \langle w^*, X_i \rangle + \xi_i 
\end{equation}
where each \emph{covariate} $X_i \sim N(0,\Sigma)$ is a Gaussian random vector in $\mathbb{R}^n$, the noise $\xi_i \sim N(0,\sigma^2)$ is independent, and the true \emph{coefficient vector} $w^*$ is $k$-sparse, i.e. $w^*$ has at most $k$ nonzero entries. The goal of the learning algorithm is to output a vector $w$ such that the \emph{out-of-sample prediction error}
\begin{equation}\label{eqn:oos-prediction-error}
\EE[(Y_0 - \langle w, X_0 \rangle)^2] = (w-w^*)^T \Sigma (w-w^*) + \sigma^2
\end{equation}
is as small as possible (i.e. close to $\sigma^2$, the error achieved by $w^*$), where $(X_0,Y_0)$ is a fresh sample from the model.

There is a rich and vast body of work on sparse linear regression with $\ell_1$-regularized approaches such as the Lasso and Dantzig selector (see for example \cite{tibshirani1996regression,candes2007dantzig,bickel2009simultaneous,van2009conditions}) ubiquitous in many domains and applied sciences (see e.g. \cite{levy1981reconstruction,santosa1986linear,wu2009genome,fan2011sparse}). The problem is also extensively studied in the signal processing, compressed sensing and sketching communities (e.g. \cite{donoho1989uncertainty,blumensath2009iterative,baraniuk2010model,hassanieh2012simple,candes2005decoding,donoho2006compressed,candes2006robust,rudelson2006sparse}) where the measurements $X$ are not necessarily Gaussian but may come from other structured distributions. 

It is well known that {\em information theoretically} (see Section \ref{section:preliminaries}), it is possible to achieve error $(1 + \epsilon)\sigma^2$
with $m = O(k \log(n)/\epsilon)$ samples; 
note that the dependence on the ambient dimension $n$ is logarithmic and that there is no dependence on the covariance matrix $\Sigma$. Unfortunately, despite a tremendous amount of work on sparse linear regression, we still do not know an efficient algorithm that for general $(\Sigma,w^*)$ can get a small error (say $O(\sigma^2)$) even with up to $o(n)$ many samples. This limitation holds even when there is no noise (i.e., $\sigma = 0$).

\ignore{
As we explain in Preliminaries,
this is the same goal as estimating the vector $w^*$ in the intrinsic \emph{Mahalanobis norm}
\[ \|w - w^*\|_{\Sigma}^2 := \langle w - w^*, \Sigma (w - w^*) \rangle. \]}



The classical algorithmic results for this problem assume that the covariates satisfy some kind of well-conditioning property such as \emph{incoherence} \cite{donoho1989uncertainty}, or a variant such as
the \emph{Restricted Isometry Property} (RIP) \cite{candes2005decoding},  the \emph{Restricted Eigenvalue Condition} \cite{bickel2009simultaneous}, or the \emph{Compatibility Condition} \cite{van2009conditions}, and achieve up to constants the optimal statistical guarantee described above.
See \cite{van2009conditions} for an extensive discussion of these assumptions\footnote{In the fixed design setting, these conditions are placed on the empirical covariance matrix (or equivalently, the design matrix). The results of \cite{raskutti2010restricted,zhou2009restricted} shows the analogous conditions on the population covariance are inherited by the empirical covariance matrix 
 in the random design setting.}. 
The simplest to state version of these conditions, the RIP property, requires that all small submatrices of the covariance matrix are spectrally close to the identity matrix. When $\Sigma$ is the identity matrix or has a bounded \emph{condition number}\footnote{The ratio of the largest eigenvalue to the smallest eigenvalue.}, the restricted eigenvalue condition holds, and we can solve sparse linear regression with $O(k \log n)$ samples \cite{raskutti2010restricted}. 
However, the above methods leave wide open what happens for general $\Sigma$. Since the population covariance $\Sigma$ is given to us by nature in most statistical applications (for example, if the covariates $X_i$ correspond to answers to survey questions, or observations from a complex scientific experiment), what happens for general $\Sigma$ is a question of significant practical interest.
This was one of the main motivations for studying weaker versions of the RIP property such as the Restricted Eigenvalue condition (see e.g. discussion in \cite{bickel2009simultaneous,raskutti2010restricted,jain2014iterative}) and compatibility condition \cite{van2009conditions,van2013lasso}, and understanding how well the Lasso performs (well or not) with correlated design matrices remains an active area of research (see e.g. \cite{dalalyan2017prediction,van2013lasso,koltchinskii2014l_1,zhang2017optimal,bellec2018noise}).

While there are a few exceptions, such as settings where submodularity holds (e.g. \cite{das2008algorithms,das2011submodular,elenberg2018restricted}), we do not have good algorithms for dealing with ill-conditioned $\Sigma$.
On the other hand, the state-of-the-art computational lower bounds for sparse linear regression \cite{natarajan1995sparse,zhang2014lower,foster2015variable,har2016approximate} apply only to the \emph{fixed-design} setting with worst-case vectors $X_i$. It's unclear that extending these results to the random design setting is even possible, given 
various barriers to proving hardness of average case problems (see, for example, \cite{applebaum2008basing} \footnote{This paper discusses obstacles to improper learning; however, in random-design sparse linear regression where $\Sigma$ is known, an improper learning algorithm can be converted into a proper learning algorithm (by using the former to generate artificial samples, and then running (ordinary) linear regression).}). Indeed, in the random design setting, the state-of-the-art lower bounds are simply against the Lasso \cite{wainwright2009sharp,foygelfast} or related classes of algorithms, such as linear regression with a coordinate-separable regularizer \cite{zhang2017optimal} or local search procedures\footnote{We note that this last work is focused on understanding a constant factor gap in the isotropic setting, a related but fairly different goal vs. understanding the landscape for general $\Sigma$.} \cite{gamarnik2017sparse}.  Such lower bounds by no means imply that the instances are computationally hard: even if the covariance matrix is ill-conditioned and Lasso fails, the sparse linear regression problem may still be tractable. Indeed, there are numerous examples \cite{foygelfast, dalalyan2017prediction, zhang2017optimal, kelner2019learning} of hard instances for Lasso which become solvable after a simple \emph{change-of-basis}, and (to our knowledge) no examples of random designs which provably cannot be solved by Lasso after such a change-of-basis.

\subsection{Preconditioned Lasso}
\emph{Preconditioning} is a powerful and extremely well-studied technique for solving linear systems. 
In that literature, there are two types of preconditioning: from the left, or from the right \cite{saad2003iterative}.

In the vast literature on $\ell_1$ methods for sparse linear regression, 
we are aware of no works that systematically study the power of ``right" preconditioning, i.e. an initial change-of-basis in parameter space (see Section~\ref{section:related} for further discussion, including the substantial differences between ``left" and ``right" preconditioning in sparse linear regression). Are there natural classes of sparse linear regression problems which can be solved by
a preconditioned $\ell_1$ method but not by classical methods? 
Are there examples of designs which provably cannot be helped by appropriate preconditioning? In this paper, we initiate a systematic study of these questions. To formalize the notion of an initial change-of-basis, we define the following large and natural class of convex programs, which we call the \emph{preconditioned Lasso}.

\begin{definition}[Preconditioned Lasso]\label{definition:preconditioned-lasso}
Let $S \in \RR^{n \times s}$ be a matrix. The \emph{$S$-preconditioned Lasso} on samples $(X_i,Y_i)_{i=1}^m$ with tuning parameter $\lambda$ is the program 
\begin{equation} 
\argmin_{w \in \RR^n} \norm{Y - Xw}_2^2 + \lambda\norm{S^T w}_1
\label{eq:preconditioned-lasso}
\end{equation}
where $X : m \times n$ is the design matrix with rows $X_i$.
Taking $\lambda \to 0$, as is done for noiseless samples, yields the \emph{$S$-preconditioned Basis Pursuit (BP)}:
\begin{equation}
\argmin_{w \in \RR^n: Xw = Y} \norm{S^T w}_1.
\label{eq:preconditioned-bp}
\end{equation}
Programs~\ref{eq:preconditioned-lasso} and~\ref{eq:preconditioned-bp} are convex, so can be solved in time $\poly(n,s,m)$. If $S$ is the identity matrix, then they 
are just the well-studied Lasso and Basis Pursuit programs 
(see, e.g. \cite{wainwright2019high}). 
\end{definition}

Program~\ref{eq:preconditioned-lasso} has been previously studied in the literature, under various names including the \emph{generalized Lasso} \cite{tibshirani2011solution}. However, in most applications of the generalized Lasso the motivation is different: the matrix $S$ is introduced into the program because the signal is not sparse in the original basis, but in a different one (e.g. for piecewise constant signals, $S$ is chosen to give the \emph{total variation} norm which penalizes the discrete derivative \cite{tibshirani2011solution,needell2013stable}).
In contrast, we are only interested in recovering signals \emph{sparse in the original basis}, and we seek to choose $S$ based on the design matrix to improve the performance of the Lasso.
To avoid confusion, we therefore refer to this program as the ``preconditioned Lasso'' in this paper.
This should not be confused with a different and largely unrelated 
terminology introduced in prior work \cite{wauthier2013comparative, jia2015preconditioning}; we expand on this distinction in Section~\ref{section:related}.

As the name suggests, the above class of programs essentially corresponds to solving the Lasso after first performing an appropriate change of basis. Indeed, if $S$ is an $n \times n$ invertible matrix, then the $S$-preconditioned Lasso is equivalent to Lasso with a (right) preconditioned design:
\[ \argmin_{u \in \RR^n} \|Y - X(S^T)^{-1} u\|_2^2 + \lambda \|u\|_1. \]
This is a natural class since, as previously remarked, the ability to change basis is powerful enough to fix the Lasso in several examples where it is otherwise known to fail (for instance, see examples in \cite{foygelfast, dalalyan2017prediction, zhang2017optimal, kelner2019learning}). 

As we will explain further, in this paper we present both 
upper and lower bounds for this class of programs. Our results are closely tied to a standard notion of \emph{graphical structure} for the covariate distribution. Before explaining the conditions in general, we start with a motivating example: estimating a sparse linear functional of a simple random walk, studied in \cite{koltchinskii2014l_1}.


\subsection{A motivating example: Random walk/Brownian motion}\label{section:introduction-example}
Suppose we have a sequence of random variables $R_1,\dots,R_n$ where each $R_i$ is generated from $R_{i-1}$ and some independent noise. That is, $Z_1,\dots,Z_n \sim N(0,1)$ are independent Gaussian random variables, with $R_1 = Z_1$ and $$R_i = R_{i-1} + Z_i$$ for $i>1$. This describes a simple random walk, one of the simplest forms of time-series data. If each covariate vector $X_i$ is an i.i.d. copy of $(R_1,\ldots,R_n)$, then the covariance matrix $\Sigma$ 
is just $\Sigma_{ij} = \min(i,j)$. More importantly, $\Sigma$ is quite ill-conditioned and existing guarantees (e.g., restricted eigenvalue etc.) do not
seem useful in this scenario \cite{koltchinskii2014l_1}. In the work \cite{koltchinskii2014l_1}, the authors gave upper bounds on the performance of the Lasso for this version of sparse linear regression, which did not match the performance of the information-theoretically optimal algorithm. 

One of the technical innovations in the present paper is a general and relatively easy-to-use method for proving \emph{lower bounds} on the performance of the Lasso in random design problems. As a simple application of our general result, we clarify the behavior of the Lasso in this model by proving a strong negative result:
\begin{theorem}[Informal version of Theorem~\ref{theorem:random-walk-lb}]\label{thm:random-walk-lb-intro}
For any $k \geq 2$, there is a $k$-sparse signal $w^* \in \RR^n$ such that the Lasso and Basis Pursuit
require at least $m = \Omega(\sqrt{n})$ samples to exactly recovery $w^*$ from noiseless observations $(X_i, Y_i)_{i = 1}^m$, when the covariates $X_i$ are independently drawn from the Gaussian random walk $N(0,\Sigma)$ and $Y_i = \langle w^*, X_i \rangle$. The same holds if the coordinates of the covariates are normalized to all have variance $1$.
\end{theorem}
The above theorem shows that the most popular algorithmic approach for sparse linear regression problems, $\ell_1$-regularized least squares, 
performs poorly in this problem; in fact, its sample complexity is exponentially sub-optimal in the ambient dimension $n$ (this was also observed experimentally in \cite{kelner2019learning}). This brings up an obvious question, which to the best of our knowledge, was unanswered even in this particular case --- can \emph{any} polynomial time algorithm achieve nearly optimal performance (or even $o(\sqrt{n})$ sample complexity) in this example?

\begin{figure}
    \centering
    \includegraphics[scale=0.5]{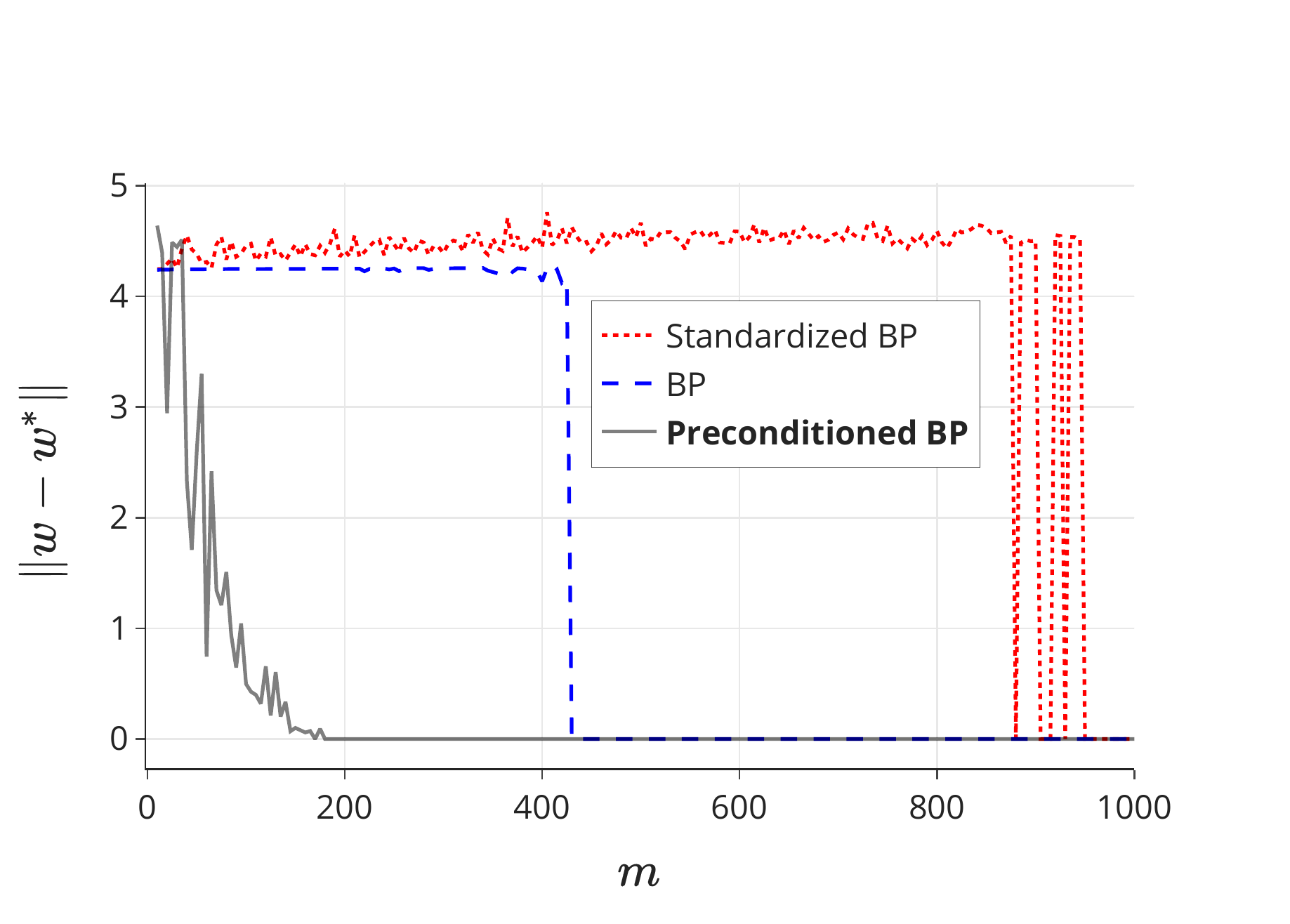}    
    \caption{Basis Pursuit (BP) vs Preconditioned BP vs Standardized BP in the case of the simple random walk with $n = 32768$ ($2^{15}$) variables.
    The $x$-axis is the number of samples $m$ (incremented in steps of $5$) and the $y$-axis is the error in recovering $w^*$ in Euclidean norm in a single independent run of the algorithms. Consistent with our theory (Theorem~\ref{thm:random-walk-lb-intro} and Theorem~\ref{thm:upper-bound-intro}), preconditioned BP succeeds at exact recovery with significantly fewer samples than normal BP or standardized BP (BP with the coordinates of $X_i$ standardized to variance 1, which is a common preprocessing step). The ground truth covariates $X_i$ are i.i.d. copies of a simple random walk with Gaussian steps, and the ground truth labels are $Y_i = 3[(X_i)_{32768} - (X_i)_{32767}]$, i.e. $w^* = (0,\ldots,0,-3,3)$. 
    }
    \label{fig:experiment}
\end{figure}

We show the answer is \emph{yes} --- sparse linear regression on a random walk can be solved efficiently. While Lasso run in the usual way fails, it turns out that if we first make an appropriate change-of-basis, i.e. precondition, Lasso will succeed. More formally, there is a sparse ``preconditioner" matrix $S \in \RR^{n \times n}$, 
such that the $S$-preconditioned Lasso 
recovers any $k$-sparse rule $w^* \in \RR^n$ from $m = O(k\log^3 n)$ samples $Y_i = \langle w^*, X_i \rangle + \xi_i$, where $\xi_i$ is independent Gaussian noise. In the present case we actually choose $S$ to be an invertible matrix, so this can be interpreted simply as a change of basis.
\begin{theorem}[Special case of Theorem~\ref{thm:upper-bound-intro}]\label{thm:special-case}
Suppose that $X_1,\ldots,X_m$ are independent copies of $(R_1,\ldots,R_n)$ and $Y_i = \langle w^*, X_i \rangle + \xi_i$ with $\xi_i \sim N(0,\sigma^2)$ independent and $\sigma^2 \ge 0$.
There is a polynomial-time algorithm (preconditioned Lasso/BP) which outputs $\hat{w}$ such that with high probability,
\begin{equation}\label{eqn:guarantee-upper-bound-intro}
(\hat{w} - w^*)^T \Sigma (\hat{w} - w^*) = O\left(\frac{\sigma^2 k \log^2(n)}{m}\right) 
\end{equation}
provided $m = \Omega(k\log^3(n))$. (When $\sigma = 0$, the rhs is zero.)
\end{theorem}
The preconditioner $S^T$ which works is quite simple, and builds upon inspiration from other settings in the compressed sensing and signal processing literature (see  \cite{mallat1999wavelet,needell2013stable,hutter2016optimal} and Appendix~\ref{apdx:example}). The inverse $(S^T)^{-1}$ is the composition of the differencing operator $D[x]_i = x_i - x_{i - 1}$ with the \emph{discrete Haar wavelet transform} \cite{haar1911theorie,mallat1999wavelet}. The first transformation (differencing) transforms the sequence $R_1,\ldots,R_n$ back into the steps $Z_1,\ldots,Z_n$, which completely fixes the ill-conditioning of the basis. However, the first step destroys sparsity: $\langle w^*, R \rangle$ is not a sparse linear functional of $Z$, but instead a dense linear functional with \emph{piecewise constant} coefficients. The second step, the Haar transform, is orthogonal --- hence preserves well-conditioning --- and it restores sparsity, because a piecewise constant signal is sparse in the Haar basis \cite{mallat1999wavelet}. Given this, Theorem~\ref{thm:special-case} follows from classical results on sparse linear regression with isotropic covariates from the compressed sensing literature (e.g. \cite{candes2007dantzig,bickel2009simultaneous}). 



It is remarkable that the random walk model admits a preconditioning matrix $S$ which combines so well with the Lasso. This suggests the question: 
\begin{question}\label{questionmain}
 \emph{For which $\Sigma$ can we construct a preconditioner 
    such that the preconditioned Lasso solves sparse linear regression for covariates drawn from $N(0, \Sigma)$?}
\end{question}

\subsection{Preconditioning and Dependency Graphs}
The answer, as it turns out, depends on the \emph{conditional independence structure} of $X \sim N(0,\Sigma)$, or equivalently the \emph{dependency graph} of the distribution. We first introduce this notion. 

Fix a distribution $D$ on $\mathbb{R}^n$, and a graph $G$ on $n$ vertices. We say $D$ satisfies the \emph{Markov property} with respect to $G$ if the following holds for $X \sim D$: whenever $i,j$ are not adjacent in $G$, $X_i, X_j$ are independent conditioned on $(X_k: \text{$k$ a neighbor of $i$ in $G$})$. That is, $G$ is a \emph{dependency graph} for the distribution $D$. The study of dependency graphs of distributions has a rich history and vast literature within statistics and machine learning, under the general area of \emph{graphical models} --- see e.g. \cite{lauritzen1996graphical,dempster1972covariance,bishop2006pattern,wainwright2008graphical,pearl2009causality}. 

For a multivariate Gaussian distribution $N(0,\Sigma)$ with invertible $\Sigma$, there is a clean characterization of the dependency graph. Let $\Theta = \Sigma^{-1}$ be its \emph{precision matrix}. The dependency graph of $N(0,\Sigma)$ is precisely the graph whose adjacency matrix is the support of $\Theta$ \cite{lauritzen1996graphical,dempster1972covariance}.
Reconsidering the example of random walks, a key property of the distribution of $(R_1,\ldots,R_n)$ as defined above is the following \emph{Markov property}: conditional on $R_i$, the past variables $R_1,\ldots,R_{i - 1}$ are independent of the future variables $R_{i + 1},\ldots,R_n$. Equivalently (see e.g., \cite{lauritzen1996graphical}), the precision matrix $\Theta = \Sigma^{-1}$ of $R_1,\dots,R_n$ is supported on the adjacency matrix of the path graph.  

Our main contribution is an essentially complete answer to Question \ref{questionmain} in terms of the corresponding dependency graph of $\Sigma$. At a high level, we show that whenever the dependency graph of $\Sigma$ has small \emph{treewidth}, then there is a preconditioner such that Lasso succeeds. Conversely, we show that for any graph $G$ with high treewidth, there is a Gaussian distribution with $G$ as the dependency graph on which \emph{no} preconditioning can make Lasso succeed. This shows that treewidth, long used as a natural complexity measure in graphical models, e.g. in the context of the celebrated junction tree algorithm \cite{lauritzen1988local}, also determines the difficulty of solving a sparse linear regression with preconditioned Lasso.
We formally state our results next. 

\subsection{Main Results}

\paragraph{Preconditioning for small treewidth.}


We show that whenever the dependency graph of the covariate distribution has low treewidth, say $t$, there exists a choice of preconditioner which makes the Lasso succeed with $\gg k t \log^3 n$ samples. Furthermore, such a preconditioner can be constructed efficiently without exact knowledge of $\Sigma$: just knowing the dependency graph allows us to efficiently construct the preconditioner based off of the samples. Formally, we show the following:

\begin{theorem}\label{thm:upper-bound-intro}
Let $G$ be a graph on $[n]$, and let $\Theta \in \RR^{n\times n}$ be a positive-definite matrix supported on $G$. Suppose that $G$ has treewidth at most $t$. Let $\sigma \ge 0$. Then there is a polynomial-time algorithm which outputs $\hat{w}$ such that with high probability,
\begin{equation}\label{eqn:guarantee-upper-bound-intro}
(\hat{w} - w^*)^T \Sigma (\hat{w} - w^*) = O\left(\frac{\sigma^2 kt\log^{1/2}(t) \log^2(n)}{m}\right) 
\end{equation}
from (1) knowledge of the graph $G$, and (2) $m = \Omega(kt\log^3 n)$ independent samples $(X_i, Y_i)$, where $Y_i = \langle w^*, X_i \rangle + \xi_i$ with $w^*$ a $k$-sparse vector, and independently $X_i \sim N(0,\Theta^{-1})$ and $\xi_i \sim N(0,\sigma^2).$
\end{theorem}
\begin{remark}
The extra $\log^{1/2}(t)$ factor arises from the approximation algorithm of \cite{feige2008improved} and can be eliminated if the optimal tree decomposition is given as input. Also, in Theorem~\ref{thm:iht-sst} we show how to shave the extra $\log(n)$ factor from \eqref{eqn:guarantee-upper-bound-intro} by combining our preconditioner with model-based Iterative Hard Thresholding instead of the Lasso (cf. \cite{baraniuk2010model}). Finally, we note the results generalize straightforwardly to subgaussian data and noise, in which case the sparsity pattern of $\Theta$ may differ from the graphical structure according to the Markov property.
\end{remark}

Besides giving a characterization of dependency structures which enable the success of preconditioned Lasso, the above also covers several important cases that arise in practice. The simplest case is the random walk discussed above, where the dependency graph is a path and therefore has treewidth $1$. Especially if we are regressing on time series data, the path graph may sometimes be a reasonable assumption on the dependency structure of the covariates. However, even in the specific context of time series, one often has multiple interacting time series and/or longer range interactions (consider e.g. an $AR(2)$ model \cite{brockwell2009time}) which fundamentally change the graph structure. In these situations, the treewidth is bounded by the length of the interactions, and thus may naturally be small. More generally, sparse graphical structure is often a natural assumption in practice and plays, for example, a very important role in causal inference and reasoning \cite{pearl2009causality,peters2017elements}.


\paragraph{Failure of preconditioning for high treewidth.}
We complement our upper bound with a sample complexity lower bound for high-treewidth graphs: 
for \emph{any} graph $G$ with treewidth $t$, there is a multivariate Gaussian distribution
with dependency graph $G$ such that
\emph{for any preconditioner} $S$, the $S$-preconditioned Lasso fails (with high probability) unless the number of samples is $\Omega(t^c)$ for an absolute constant $c > 0$. The preconditioner is allowed to depend on the distribution, and the lower bound result holds in the (easiest) noiseless setting,
where the corresponding notion of success requires exact recovery of the ground truth. 

\begin{theorem}\label{theorem:lower-bound-intro}
Pick $n, t, s \in \NN$, and suppose that $G$ is a graph on $[n]$ with treewidth at least $t$. Then there exists $k = O(\log n)$ and some positive-definite precision matrix $\Theta$, supported on $G$, with condition number $\poly(n)$, such that the following holds: for every preconditioner $S \in \RR^{n \times s}$, the $S$-preconditioned basis pursuit requires $m = \Omega(t^{1/20})$ samples $(X_i,Y_i)$ to exactly recover a $k$-sparse coefficient vector $w^*$ from covariates $X_1,\ldots,X_m$ drawn i.i.d. from $N(0,\Theta^{-1})$ and noiseless responses $Y_i = \langle w^*,  X_i \rangle$, with probability better than $1/t^{1/400}$.

\end{theorem}

To the best of our knowledge, this result provides the first class of examples of random design problems where a change of basis \emph{provably cannot} fix the performance of the Lasso. To prove this result, we develop an easy-to-use machinery for proving lower bounds on the performance of the Lasso in random design settings, which is of independent interest. 

\ignore{

\subsection{Our Contribution}
The previous example suggests the following natural question about the power of preconditioning in sparse linear regression:

\begin{center}
    \emph{For which graphical structures of $\Theta = \Sigma^{-1}$ can we construct a preconditioner 
    such that the preconditioned Lasso solves sparse linear regression for covariates drawn from $N(0, \Sigma)$?}
\end{center}

Our main contribution is essentially a complete resolution of this question. 
In a word, we show that the graphical property that governs the existence of good preconditioners is \emph{treewidth}\footnote{Treewidth famously appears in graphical models via the junction tree algorithm \cite{lauritzen1988local} which allows for computationally efficient inference for graphical models on low-tree-width graphs. However, in the present context (Gaussians), inference is already computationally efficient so the connection seems limited.}. On the one hand, we show that if $\Theta$ is supported on a graph with treewidth $t$, then there is a preconditioner such that preconditioned Lasso succeeds with $O(kt\log(n)/\epsilon^2)$ samples, and this preconditioner can be efficiently computed from samples given only \emph{knowledge of the underlying graph structure} and without requiring any other knowledge of $\Sigma$. This means the result applies in the practically important settings where we have high-level knowledge of the (graphical) structure of the data but do not know the exact details of the data generating process (see e.g. \cite{pearl2009causality} for further discussion).

On the other hand, we show that for any graph with treewidth $t$, there is a precision matrix $\Theta$ supported on this graph, such that no preconditioner can enable the success of preconditioned Lasso when covariates are drawn from $N(0,\Theta^{-1})$, unless the number of samples is $\Omega(t^\epsilon)$ for a constant $\epsilon>0$. The latter lower bound holds even in the simple noiseless setting with $\sigma = 0$, and provides (to the best of our knowledge) the first class of examples of random design problems where a change of basis \emph{provably cannot} fix the performance of the Lasso. To prove this result, we develop an easy-to-use machinery for proving lower bounds on the performance of the Lasso in random design settings, which is of independent interest. 

\begin{itemize}
    \item Start with sparse linear regression on random walk/time-series data.
    \item Lasso does not work, but Lasso in a different basis does. Lead up to preconditioned Lasso.
    \item Small tree-width.
    \item Limitations of preconditioning. 
\end{itemize}
}

\section{Overview of Techniques}\label{section:overview}

\subsection{Algorithms for Low-treewidth}
\paragraph{Known $\Sigma$ setting.} First, we describe the simplified version of the low-treewidth algorithm, which assumes knowledge of the population covariance matrix $\Sigma$. The algorithm generalizes the one for the path described earlier.
We show that there exists a sparse matrix $S$ such that $\Sigma = SS^T$ (i.e., a \emph{sparse Cholesky factorization}) and such that both $S^T$ and $(S^T)^{-1}$ are \emph{sparsity preserving} (within $poly(t,\log(n))$ factors, where $t$ is the treewidth). This enables us to precondition the Lasso exactly, transforming from the original problem to a new problem where the covariates have identity covariance, and the unknown signal is transformed but still sparse --- the ideal setting to apply classical results on sparse linear regression from compressed sensing literature.

The preconditioner $S$ we use is closely related to the Haar wavelet transform from signal processing and its natural generalization to trees (see, e.g., \cite{sharpnack2013detecting}). We extend this transform to work with low-treewidth graphs. A variant of Haar wavelets also appears in the recent work \cite{dong2020nearly} on solving linear programming instances with low treewidth; however, the details and the motivation are different from ours. 


Concretely, the tree we start with is given by computing a \emph{tree decomposition} of our low-treewidth graph, using for example the algorithm of \cite{feige2008improved}. This tree provides a natural hierarchical decomposition of the graph, because we can always break a tree into roughly equal size pieces by removing its \emph{centroid} \cite{guibas1987linear}. We can then exploit the Markov property to reduce the problem of preconditioning the entire model to preconditioning each of the smaller pieces. Recursing, we get a sparse block Cholesky factorization of $\Sigma$ that we use as the preconditioner.


Because our preconditioner has a natural tree structure, we also show that we can use algorithmic tools from the area of \emph{model-based compressed sensing} (see, e.g., \cite{baraniuk2010model}) to shave an extra log factor from the rate that arises when using the preconditioned Lasso. 
This algorithm, based on a version of Iterative Hard Thresholding \cite{blumensath2009iterative}, allows us to recover the information-theoretically optimal $O(\sigma^2 k \log(n)/m)$ rate in the bounded treewidth setting.

\paragraph{Data-Dependent Sparse Preconditioner.}
It's often the case that the true population covariance matrix $\Sigma$ is unknown to the algorithm. Furthermore, since we assume access to only a small number of samples $X_i$ from the covariate distribution, the empirical covariance matrix cannot stand in as a suitable replacement for the true matrix $\Sigma$ (for example, the empirical covariance matrix may not be invertible even if the true $\Sigma$ is). 

We show how to overcome these difficulties under the more realistic assumption that the graphical structure of the distribution, i.e., the support of the precision matrix $\Theta = \Sigma^{-1}$, is known. This kind of modeling assumption is prevalent in the causal inference and graphical models literature (see, e.g., \cite{pearl2009causality}) and is generally more plausible than knowing the exact matrix $\Sigma$. Also, even if initially unknown, the graph structure may be recoverable from a small number of samples using GGM learning algorithms (e.g., \cite{meinshausen2006high,friedman2008sparse,kelner2019learning}).

Algorithmically, we build our preconditioner by performing an approximate \emph{block Cholesky factorization} of the empirical covariance matrix $\hat \Sigma$, following the tree structure described above. By using the Markov property, we can handle the poor approximation quality of the empirical covariance matrix $\hat \Sigma$ to $\Sigma$ by zeroing out all of the entries of various Schur complements which arise during the Cholesky factorization and which must be zero due to the Markov property. If $A$ is the centroid from the tree decomposition and $P, Q$ are the resulting subforests given by removing $A$, the approximate block factorization is given by
\[
S := \begin{bmatrix}
\tilde{\Sigma}_{AA}^{1/2} & 0 & 0 \\
\tilde{\Sigma}_{PA} \tilde{\Sigma}_{AA}^{-1/2} & S_P & 0 \\
\tilde{\Sigma}_{QA} \tilde{\Sigma}_{AA}^{-1/2} & 0 & S_Q.
\end{bmatrix} \]
where $\tilde \Sigma$ is the empirical covariance matrix, and the bottom right is a (recursively defined) approximate block Cholesky factorization of the Schur complement $\tilde{\Sigma}/A$ with zerod out bottom-left and top-right sub-blocks. This factorization is \emph{not} a spectral approximation of the empirical covariance matrix $\tilde \Sigma$ (which may be rank degenerate); instead, we show that changing basis by $S$ results in a new Lasso problem which satisfies the \emph{Restricted Isometry Property} \cite{candes2005decoding}. The proof of this fact is quite involved, as we need to precisely track the accumulation of errors in the factorization from the perspective of a sparse test vector.

\paragraph{Aside: sparse linear regression with sparse covariance.} The assumption that $\Theta$ is sparse is very common and natural from a modeling perspective. That being said, we also consider what happens when $\Sigma$, instead of $\Theta$, is sparse. In Appendix~\ref{sec:sparse-covariance}, we give an algorithm for $k$-sparse linear regression with runtime roughly $d^k \cdot poly(n)$, where the rows of $\Sigma$ are $d$-sparse. It's again based on preconditioning the Lasso but uses a randomized preconditioner based on a \emph{site percolation} process on the graph (see, e.g., \cite{krivelevich2015phase}), where each vertex of the graph is kept with probability $p$. 
\subsection{Impossibility of preconditioning in high-treewidth models}

In this section, we outline the proof of Theorem~\ref{theorem:lower-bound-intro}, the sample complexity lower bound for high-treewidth graphs. There are three main elements to the proof:
\begin{enumerate}
    \item Identifying conditions on a precision matrix $\Theta = \Sigma^{-1}$ and preconditioner $S$, under which the $S$-preconditioned Lasso will fail (for some sparse signal, with covariates drawn from $N(0,\Sigma)$)
    \item Constructing a precision matrix on (a slight variant of) the \emph{grid graph} which satisfies these conditions for any preconditioner
    \item Extending the lower bound for the grid graph variant, in a black-box manner, to a lower bound for any high-treewidth graph
\end{enumerate}


\paragraph{Conditions under which (preconditioned) Lasso fails.} There are two distinct reasons why classical Lasso might fail to recover some signal: either the covariates are ill-conditioned, or the ground truth is not sparse. For preconditioned Lasso, the situation is \emph{roughly} analogous, and we have two cases: if the preconditioned covariates are ill-conditioned, then recovery should intuitively fail; and if the preconditioner has dense rows, meaning that the ground truth may be dense in the preconditioned basis, then recovery should intuitively fail. While making these statements precise requires additional assumptions, this intuition is accurate in spirit.

We first formalize the first case above. For a fixed design matrix $X$, the standard KKT conditions can determine whether Lasso/Basis Pursuit succeed at exact recovery (see, e.g., Theorem 7.8 of \cite{wainwright2019high}). However, to show the Lasso fails in random design, we need a condition on $\Sigma$ that guarantees failure with a high probability over $X$.
Despite a vast literature on conditions for success and failure of Lasso, we are not aware of a broad, sufficient condition on the covariance matrix in the random design setting under which there must exist some sparse signal that causes
Lasso to fail (even in the more straightforward non-preconditioned setting). We rectify this gap by introducing the \emph{Weak ($S$-Preconditioned) Compatibility Condition}. This condition is defined analogously to stronger compatibility conditions (cf. \cite{raskutti2010restricted,van2009conditions}), which are sufficient for Lasso's success. That is, it roughly states that $SS^T$ (identity in the case of unpreconditioned Lasso) approximates $\Sigma$. However, unlike classical compatibility conditions, the condition we introduce is \emph{necessary}\footnote{An interesting result with related motivation is Theorem 3.1 of \cite{bellec2018noise}, which shows that if the Lasso succeeds for arbitrary sparse signals while (a variant of) the $\ell_1$-eigenvalue/compatibility constant of the design matrix is large, then the regularization parameter $\lambda$ must be small. However, it leaves open the possibility that Lasso may succeed with an appropriately small choice of $\lambda$.} as opposed to sufficient: if it is not satisfied, then the $S$-preconditioned Lasso will fail with high probability on some sparse signal. 

Of course, it's easy to construct a preconditioner $S$ such that $SS^T$ does approximate (or even equals) $\Sigma$; it just might not be sparsity-preserving (i.e., $S^T w^*$ need not be sparse even for sparse $w^*$). We formalize the intuition that a preconditioner with dense rows should also cause preconditioned Lasso to fail (i.e., the second mode of failure we alluded to above). This is surprisingly challenging. The main obstacle is that this is false without additional restrictions: for example, replacing the preconditioner $S$ with the column-wise concatenation $[S; S]$ doubles the size of the support of $S^T w^*$ but doesn't affect the output of the $S$-preconditioned basis pursuit, so sample complexity cannot be directly tied to $|\supp(S^T w^*)|$.

More technically, even if $S^T w^*$ is dense, we cannot simply change the basis and use the fact that classical Lasso fails on dense signals\footnote{Recall that $S$-preconditioned Lasso can be intuitively viewed as standard Lasso where the actual signal is $S^T w^*$.}. The issue is that $S^T$ may map to a higher-dimensional space, so changing basis introduces a new subspace constraint into the program, which could make KKT optimality conditions easier to satisfy. Instead, to find violations of the KKT optimality conditions, we must use additional structural properties of $\Theta$ and $S$ (e.g., that every \emph{column} of $S$ is either dense or has a very small norm). In the next paragraph, we discuss a specific framework for constructing $\Theta$; it is under this framework that we can show that a dense preconditioner causes recovery to fail.

\paragraph{Framework for constructing $\Theta$.} How do we construct a positive-definite matrix $\Theta$ such that any preconditioner $S$ is either dense or poorly approximates $\Theta$ (i.e. $SS^T$ is spectrally far from $\Sigma = \Theta^{-1}$)? Obviously, $\Theta$ must be ill-conditioned. Taking this intuition to the extreme, we can consider a nearly-degenerate matrix $\Theta = \Thetat + \epsilon I$, where $\Thetat$ is PSD with an $r$-dimensional kernel, and $\epsilon$ is arbitrarily small. If the preconditioner $S$ satisfies $SS^T \approx \Sigma$ in an appropriate sense, then it can be seen that every column of $S$ lies arbitrarily near $\ker \Thetat$ (as $\epsilon \to 0$) and that the columns must, in the limit, span $\ker \Theta$. So for $S$ to necessarily be dense, it's enough that $\ker\Thetat$ is high-dimensional and contains no sparse vectors. This is the key insight in understanding what properties a hard instance should have: 
\begin{center} \emph{To construct a precision matrix $\Theta = \Thetat+\epsilon I$ which is hard to precondition, it suffices to show that $\ker \Thetat$ is high-dimensional and dense, i.e., contains no sparse vectors}.\footnote{In a way, this property of $\Thetat$ resembles the Restricted Isometry Property, which is a property of the covariance matrix $\Sigma$ that enables the \emph{success} of Lasso. However, in our case, the condition is placed on the inverse covariance matrix $\Theta$, which means it obstructs preconditioning the large eigendirections of $\Sigma$.}\end{center}

Of course, this ``story" has several issues. First, the assumption that $SS^T$ spectrally approximates $\Sigma$ is too strong because the converse does not imply that $S$-Preconditioned Lasso fails. Instead, we only assume that the $S$-Preconditioned Weak Compatibility Condition holds, which introduces new difficulties in proving that $S$ is dense. Notably, a vital step of the proof requires that $\Theta$ is \emph{very sparse}. This further motivates our investigation of sparse linear regression when covariates are drawn from Gaussian Graphical Models: the sparse dependency structure is crucial.

Second, we do not want $\Theta$ to be arbitrarily close to degenerate; we want it to have $\poly(n)$ condition number. This introduces a new wrinkle, but the same insight still mostly holds: we want to find a PSD matrix $\Thetat$ (supported on some graph), such that the kernel is high-dimensional and is ``robustly" dense, i.e., not too close to any sparse vector. This matrix should also have a polynomial condition number on $\text{span}(\Thetat)$. The following theorem formalizes the above framework, i.e., conditions on $\Theta$ under which $S$-preconditioned Lasso cannot succeed:


\begin{theorem}[Restatement of Theorem~\ref{theorem:lower-bound}]\label{theorem:lower-bound-overview}
Let $\Thetat \in \RR^{n \times n}$ be a PSD matrix. Let $k,m,s > 0$. Let $\tau > 0$ and $V \subseteq [n]$ and let $\eta$ be the infimum of $\norm{x_V-y}_2/\norm{x}_2$ over all nonzero $x \in \ker(\Thetat)$ and $\tau$-sparse $y \in \RR^V$. 
Also, let $\lambda$ be the smallest non-zero eigenvalue of $\Thetat$. Suppose that the following hold:
\begin{itemize}\itemsep 0em
    \item The rows (and columns) of $\Thetat$ are $k$-sparse
    \item $r := \dim \ker(\Thetat) > 2m$
    \item $k > 3(|V|/\tau)\log(n)$
\end{itemize}

Pick any positive $\epsilon < \eta^2 \lambda^3/(16200n^3 \lVert\Thetat\rVert_F^2).$ Define $\Theta = \Thetat + \epsilon I$. For any preconditioner $S \in \RR^{n \times s}$, there is some $k$-sparse signal such that $S$-preconditioned Lasso fails at exact recovery with probability at least $1 - \frac{4m}{3r} - \exp(-\Omega(m))$, from independent covariates $X_1,\dots,X_m \sim N(0,\Theta^{-1})$ and noiseless responses $Y_i = \langle w^*, X_i \rangle$.
\end{theorem}

\paragraph{The expander graph.} Ultimately, we will need to construct a positive semi-definite matrix $\Thetat$ supported on a variant of the grid graph, whose kernel has the above properties. However, we first discuss a simple construction on an expander graph, which shares several ideas with the more intricate grid graph construction. Our approach is to define $\Thetat = M^TM$ for an appropriate matrix $M \in \RR^{n-r \times n}$. Then $\Thetat$ is necessarily PSD, and its kernel must have dimension at least $r$. As $\ker(\Thetat) = \ker(M)$, we can view each row of $M$ as an equation that constrains the kernel. Specifically, we let each row of $M$ be a sparse Bernoulli random vector. With high probability, $M$ is the adjacency matrix of a bipartite expander graph, and classical results show that $\ker(M)$ is robustly dense.

This construction is noteworthy in several ways. First, unlike the lower bounds achieved in Theorem~\ref{theorem:lower-bound-intro}, this construction yields a linear sample complexity lower bound (in the number of variables $n$, or equivalently in the treewidth) for recovering $\text{polylog}(n)$-sparse vectors, for the expander graph. Second, the proof utilizes well-known results from compressed sensing and random matrix theory, providing an interesting example of how techniques originally developed to prove the \emph{success} of sparse recovery methods can be applied to establish \emph{failure}. Third, this construction may be a good candidate for stronger lower bounds (e.g., computational or statistical query).

\paragraph{The grid graph.} We now turn to constructing a positive-semidefinite matrix $\Thetat$ supported on the grid graph, such that the kernel has the above properties. As with the expander graph, we define $\Thetat = M^TM$ for an appropriate matrix $M \in \RR^{n-r \times n}$.

The requirement that $\Thetat$ be supported on the grid graph means that the equations must be in some sense ``local". An entry $\Thetat_{ij}$ is nonzero if there is some equation containing both variables $i$ and $j$. This is problematic if we want $\Thetat$ to truly be supported on the grid graph (since it is triangle-free, no equation could have more than $2$ variables). Instead, we relax that condition slightly, to require that $\Thetat$ be supported on the \emph{simplicized grid graph}, depicted in Figure~\ref{fig:simplicized-grid} for $n=4$.

\begin{figure}
    \centering
    \includegraphics[width=0.2\textwidth]{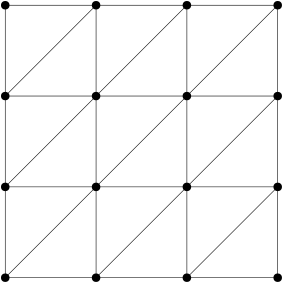}
    \caption{$4\times 4$ simplicized grid graph}
    \label{fig:simplicized-grid}
\end{figure}

Now, for every triangle of the simplicized grid graph, there can be an equation constraining the triangle's vertices. We define a subset $\mathcal{X}$ of the top row and a subset $\mathcal{Y}$ of the bottom row. We construct equations so that $\mathcal{X}$ is a set of free variables, and every other vertex has precisely one constraint, and any solution is robustly dense on either $\mathcal{X}$ or $\mathcal{Y}$, which suffices for our needs.

That such a construction exists is a priori unclear; for example, if the equations have random weights, then it turns out that in most solutions, the norm on row $r$ decays exponentially as $r$ increases. Hence, if the variables in $\mathcal{X}$ are set to some sparse vector, then $\mathcal{Y}$ will be very close to $0$ and thus not robustly dense. Similar issues arise if the equations' weights are periodic; e.g., if the first row avoids high-frequency Fourier vectors, then subsequent rows may decay exponentially.

Instead of these approaches, we take inspiration from a simple construction that does not observe the ``locality" conditions but instead is essentially a complete bipartite graph between $\mathcal{X}$ and $\mathcal{Y}$. 
Specifically, if there are no conditions on the locality of the equations, then we may introduce constraints so that each variable in $\mathcal{Y}$ is a Gaussian random linear combination of the variables in $\mathcal{X}$, i.e., $v_\mathcal{Y} = Av_\mathcal{X}$ where $A$ is a Gaussian random matrix, and $v$ is any solution. Standard matrix concentration results imply that if $v$ is nonzero, then either $v_\mathcal{Y}$ or $v_\mathcal{X}$ must be robustly dense (this can be thought of as an uncertainty principle, as in \cite{donoho1989uncertainty}).

Obviously, the complete bipartite graph cannot be directly embedded in the simplicized grid graph. However, if the grid is sufficiently large (specifically, having side length $\Omega(|\mathcal{X}|^2)$), then the complete bipartite circuit defining $\mathcal{Y}$ in terms of $\mathcal{X}$ can in fact be simulated on the simplicized grid graph. Paths between all pairs of $\mathcal{X}$ and $\mathcal{Y}$ are constructed to avoid overlap. Vertex crossings are inevitable, but they can be replaced by constant-size ``swap gadgets" which simulate crossing paths via the XOR/addition swapping trick:
$$x := x+y; \qquad y := x-y; \qquad x := x-y.$$
See Figure~\ref{fig:circuit} for a schematic of the implementation on the grid graph (not showing swap gadgets).

\begin{figure}[t]
    \centering
    \includegraphics[width=0.4\textwidth, trim = {0 0 0 5cm}]{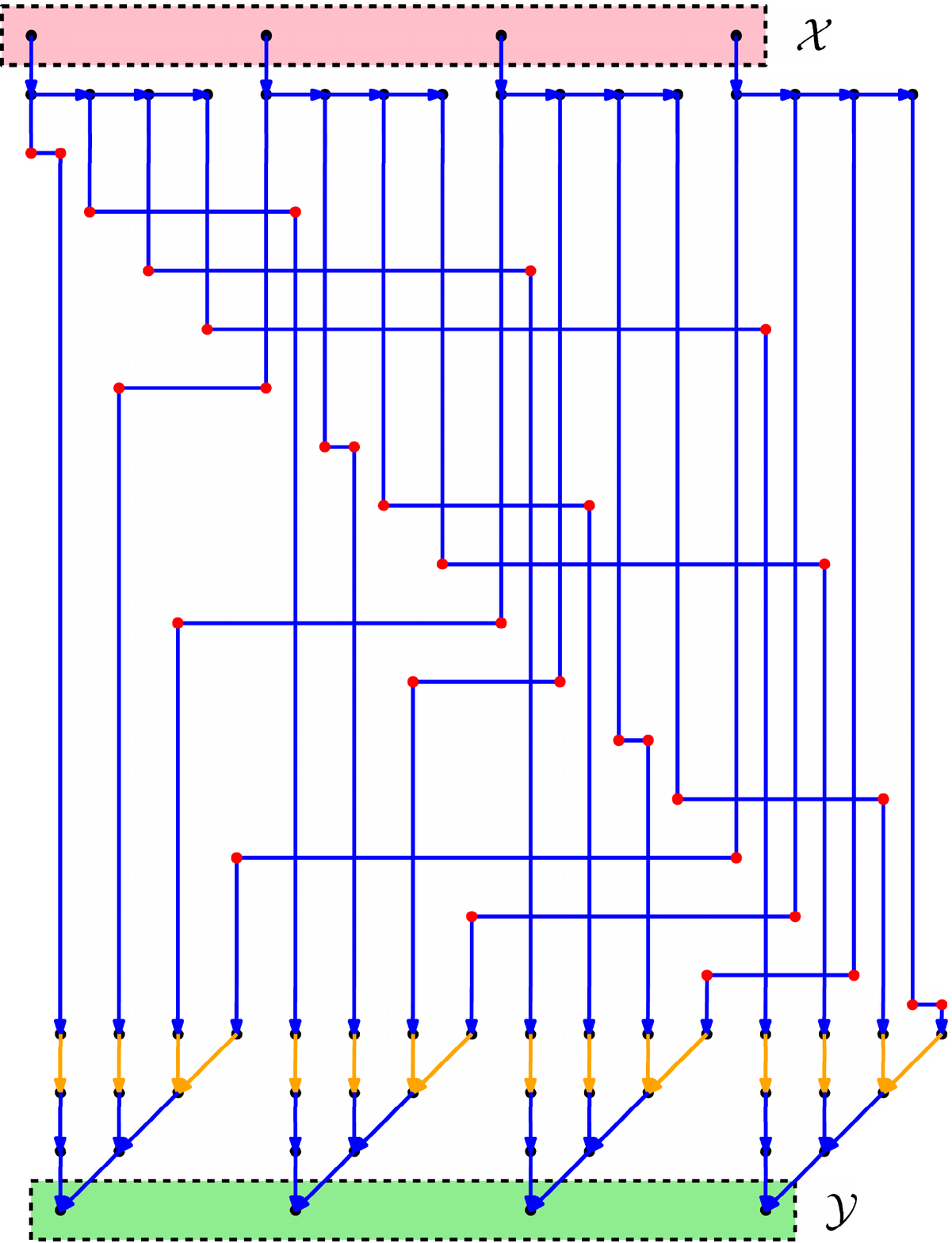}
    \caption{Schematic of equations defining the kernel of a positive-semidefinite matrix $\tilde \Theta$ supported on the simplicized grid graph. The kernel has dimension $|\mathcal{X}|$, and every vector in the kernel (i.e. solution to the equations) is robustly dense on either $\mathcal{X}$ or $\mathcal{Y}$. Note that there is a path from every vertex of $\mathcal{X}$ to every vertex of $\mathcal{Y}$. This allows us to enforce the constraint that for any solution, each variable of $\mathcal{Y}$ is equal to a random linear combination of the variables of $\mathcal{X}$. Every blue directed path denotes that all vertices on the path are constrained to be equal. A vertex with multiple incoming arrows indicates that the vertex is constrained to be the sum of its predecessors. The orange arrows are assigned Gaussian random weights. Every crossing between two paths is replaced by a ``swap gadget" ensuring that the paths do not interfere. Vertices not on any path are constrained to be $0$.}
    \label{fig:circuit}
\end{figure}

\paragraph{Unminoring.} With the above techniques, we can prove that there is a precision matrix $\Theta$ supported on the simplicized grid graph, such that for any preconditioner, the preconditioned Lasso needs a polynomial number of samples to succeed when covariates are drawn from $N(0,\Theta^{-1})$. To extend this result, we make the following simple observation: if a covariance matrix $\Sigma$ can be $S$-preconditioned so that Lasso succeeds at sparse recovery with covariates from $N(0,\Sigma)$, then certainly the same holds for any submatrix $\Sigma_{VV}$; the preconditioner is just $S_V$. In the language of precision matrices, this means that a precision matrix $\Theta$ is a hard instance (against all preconditioners) if it has a Schur complement $\Theta/\Theta_{\bar{V}\bar{V}}$ which is a hard instance.

Our goal is therefore to prove that for any high-treewidth graph $G$, there is a precision matrix $\Theta$ supported on $G$ and a vertex subset $V$ such that the Schur complement $\Theta/\Theta_{\bar{V}\bar{V}}$ approximates the hard simplicized grid instance. We appeal to the celebrated Grid Minor Theorem, which states that any graph with treewidth $t$ contains a grid minor of size $t^{\Omega(1)} \times t^{\Omega(1)}$ \cite{chekuri2016polynomial, chuzhoy2021towards}. Finally, we prove that if $G, H$ are graphs and $H$ is a minor of $G$, then any positive-definite matrix supported on $H$ can be approximated to arbitrary accuracy as a Schur complement of some positive-definite matrix supported on $G$. This last step is technically involved, but the construction is fairly simple: since $H$ is a minor of $G$, each vertex of $H$ corresponds to a connected component of $G$, and any edge in $H$ corresponds to an edge between the respective components in $G$. Given a matrix $\Gamma$ supported on $H$, we construct a nearly block-diagonal matrix $\Theta$ on $G$, where each block is a large multiple of the Laplacian of the induced subgraph of a component of $G$. This means each block is approximately a \emph{Gaussian free field} \cite{sheffield2007gaussian}, which induces a strong positive correlation between the variables inside the block. The entries of $\Gamma$ are then assigned to appropriate edges of $G$ and added into $\Theta$.


Using this result, we extend our lower bounds against preconditioned Lasso to all high-treewidth graphs, completing our tight graphical characterization of the power of preconditioned Lasso.

\subsection{Organization}
In Section~\ref{section:related} we include some final discussion of related work. In
Section~\ref{section:preliminaries} we recall some technical preliminaries, e.g. notation and concentration results. In Section~\ref{section:discussion} we remark on potential directions for future research.

\paragraph{Upper bound.} Section~\ref{sec:upper-bounds} discusses the proof of the algorithmic results for low treewidth graphs. As mentioned above, the bulk of the technical content handles the case where the population covariance $\Sigma$ is unknown; given this result, we show how to combine our preconditioner with the analysis of Lasso and model-based Iterative Hard Thresholding from the literature.

\paragraph{Lower bound.}
Sections~\ref{section:ill-conditioned}, \ref{section:lower-bounds}, \ref{section:grid-construction}, and~\ref{section:unminoring} are devoted to the full lower bound proof. Specifically, in Section~\ref{section:ill-conditioned}, the Weak $S$-Preconditioned Compatibility Condition is formally introduced, and we show that it is a necessary condition for $S$-Preconditioned Lasso to succeed. In Section~\ref{section:lower-bounds}, we show that if the precision matrix satisfies certain properties (most crucially, the dense kernel property), and if the Weak Compatibility Condition holds, then the preconditioned Lasso still fails. Culminating in Theorem~\ref{theorem:lower-bound}, this completes the generic framework for constructing hard examples for preconditioned Lasso.

Section~\ref{section:expander} gives a simple instantiation of this framework, on the expander graph, which is tangential to the proof of Theorem~\ref{theorem:lower-bound-intro}. Sections~\ref{section:grid-construction} and~\ref{section:unminoring} return to the main proof. In Section~\ref{section:grid-construction}, we construct a precision matrix supported on the grid graph with the desired properties. In Section~\ref{section:unminoring}, we extend this construction in a black-box manner to all high-treewidth graphs, proving Theorem~\ref{theorem:lower-bound-intro}.

\section{Further Related Work}\label{section:related}

\paragraph{Generalizations of the Lasso.} There is an immense literature on generalizations of the Lasso. However, to our knowledge, our work is the first to study the preconditioned Lasso as defined above, for the purpose of solving sparse linear regression. It is worth contrasting with two related branches of prior work:
\begin{enumerate}
    \item The \emph{generalized Lasso} \cite{she2010sparse, tibshirani2011solution} is defined as $$\argmin_{w \in \RR^n} \norm{Y - Xw}_2^2 + \lambda \norm{Dw}_1,$$ for a penalty matrix $D$. Definition~\ref{definition:preconditioned-lasso} is a certainly a program of this form.
    However, 
    the motivation is quite different: while 
    we consider the preconditioned Lasso as a class of approaches for linear regression with \emph{sparsity in the original basis}, the generalized Lasso
    was introduced to encapsulate ``problems that use the $\ell_1$ norm to enforce certain structural constraints---instead of pure sparsity" \cite{tibshirani2011solution}.
    That line of work has largely focused on algorithms for the generalized Lasso and applications for specific choices of penalty matrix $D$.
    
    \item A different notion of preconditioned Lasso introduced in prior work \cite{wauthier2013comparative, jia2015preconditioning} is the notion of solving the Lasso on ``preconditioned" samples $(AX, AY)$, for some invertible matrix $A$, 
    instead of the original samples $(X, Y)$. This approach has the same motivation as ours: modifying the problem so that Lasso will succeed at signal recovery (or e.g. sign recovery \cite{jia2015preconditioning}) for sparse linear regression. However,
    the two kinds of preconditioning (left vs right) are very different: theirs occurs in the space of samples, whereas ours occurs in the space of parameters. This is most clear in the noiseless setting, where Lasso reduces (if we send $\lambda \to 0$) to the basis pursuit program $$\argmin_{w \in \RR^n: Xw = Y} \norm{w}_1.$$
    As defined in \cite{jia2015preconditioning}, preconditioning has no effect on the basis pursuit program. In contrast, with our definition, preconditioned basis pursuit can often provably succeed where basis pursuit fails, as exhibited in Section~\ref{section:introduction-example}. We do note that \cite{wauthier2013comparative} also suggested a more general definition of preconditioning which includes ours, though they focused on the effect of left preconditioning as discussed above.
\end{enumerate}

\paragraph{Lower bound related work.} Impossibility results for the sparse linear regression problem fall into several categories, depending on whether they address the fixed-design setting or the random design setting, and on what classes of algorithms they rule out. In the fixed design setting, there are computational lower bounds against finding a sparse solution to a system of linear equations \cite{natarajan1995sparse, zhang2014lower, foster2015variable, har2016approximate}. No comparable results are known for random designs; there is a lower bound for \emph{robust} sparse linear regression under an assumption related to hardness of planted clique \cite{brennan2020reducibility}, but this appears to be an unrelated phenomenon, in that it holds even when $\Sigma$ is the identity matrix.

There is a richer literature on lower bounds specifically against the Lasso, for both fixed and random designs. In the random design setting, most focus on well-conditioned or identity covariances, and seek to pinpoint the constant factor in sample $m = ck \log n$ \cite{wainwright2009sharp, chen2016bayes, reeves2019all}. In contrast, we seek asymptotically stronger sample complexity lower bounds, which of course requires passing to ill-conditioned covariance matrices. There is also prior work bounding what rates the Lasso can achieve, in terms of the compatibility constant of the design matrix and the regularization parameter $\lambda$ \cite{van2018tight, bellec2018noise, bellec2018slope}. However, these results in general provide no lower bounds against Lasso in the noiseless setting, where the tuning parameter $\lambda$ is sent to $0$. Moreover, these works do not touch upon the issue of preconditioning.

In that vein, our work is most closely related to \cite{zhang2017optimal}, which constructs a (fixed-design) lower bound against the generalization of Lasso to arbitrary coordinate-separable penalties instead of the $\ell_1$ norm. Analogous to our motivation, they considered this class because there are problems for which the Lasso fails, but succeeds after an appropriate diagonal preconditioning.
Note that when the penalty is a weighted linear combination of the magnitudes of the regression coefficients, this corresponds to a diagonal preconditioner. However, coordinate-separable penalties do not encompass the full power of preconditioning by an arbitrary matrix (and vice versa). Indeed, it is a limitation of the prior work that the constructed lower bound design matrices are block-diagonal with block size $2$, and therefore amenable to being solved by the preconditioned Lasso. This is a limitation we address in our work.






\section{Preliminaries}\label{section:preliminaries}

\paragraph{Linear Algebra Notation.}
The inner product of two vectors is denoted $\langle x, y \rangle = x^T y$.
For a matrix $M$, we let $\rspan M$ denote the span of the rows of $M$ and let $\ker M = \{x : M x = 0 \}$ denote the kernel of $M$. Note that these spaces are orthogonal complements, because if $M^T x$ is an arbitrary vector in $\rspan M$ and $y \in \ker M$, then $\langle M^T x, y \rangle = \langle x, M y \rangle = 0$. For a matrix $M \in \RR^{n \times p}$, a vector $v \in \RR^n$, and a subset $U \subseteq [n]$, we use the notation that $M_U$ is the $|U| \times p$ submatrix consisting of the rows of $M$ indexed by $U$; and $v_U$ consists of the entries of $v$ indexed by $U$; and $U^c = \bar{U} = [n] \setminus U$ is the complement of $U$ in $[n]$.

We use standard notation for vector and matrix norms: $\|x\|_p = \left(\sum_i |x_i|^p\right)^{1/p}$ and 
\[ \|M\|_{p \to q} = \sup_{x : \|x\|_p = 1} \|M x\|_q. \]
If no subscript is given for a vector norm, it denotes the Euclidean norm $\|\cdot\|_2$; if no subscript is given for a matrix, it denotes the operator norm. 
We also denote the set of $k$-sparse vectors as
\begin{equation}\label{eqn:k-sparse}
B_0(k) = \{ x : \#\{i : x_i \ne 0\} \le k \}. 
\end{equation}
Given a positive semidefinite matrix $\Sigma$, the corresponding \emph{Mahalanobis norm} is defined as
\[ \|x\|_{\Sigma}^2 = \langle x, \Sigma x \rangle. \]
For symmetric matrices $A,B$, the Loewner order $A \preceq B$ means that $B - A$ is positive semidefinite. For a graph $G$ on $[n]$, we say that an $n \times n$ matrix $\Theta$ is \emph{$G$-sparse} or \emph{supported on $G$} if $\Theta_{ij} \neq 0$ implies that $(i,j) \in E(G)$ (implicitly adding all self-loops to $G$).
We use the standard notation $U \sqcup V$ to denote the union of sets $U$ and $V$ when $U$ and $V$ are assumed to be disjoint.

\paragraph{Model and objective.} We recall from \eqref{eqn:linear-model} that the generative model of data is 
\[ Y_i = \langle w^*, X_i \rangle + \xi \]
where the covariates $X_i$ are sampled i.i.d. from a (sub-)Gaussian distribution and $\xi_i \sim N(0,\sigma^2)$ is independent. In linear algebra notation, this is written
\[ Y = X w^* + \xi \]
where $Y = (Y_1,\ldots,Y_m)$ is the response vector, $X : m \times n$ is the \emph{design matrix} with rows $(X_1,\ldots,X_m)$ and $\xi = (\xi_1,\ldots,\xi_m)$ is the noise vector.
As stated in the introduction, the objective we focus on is the \emph{out-of-sample prediction error} or \emph{population loss}
\[ L(w) = \EE[(Y_0 - \langle w, X_0 \rangle)^2] \]
where $(X_0,Y_0)$ is a fresh sample from the model. By expanding the definition of $Y_0$ and using independence of the noise, we see this can be written
\[ \EE[(Y_0 - \langle w, X_0 \rangle)^2] = \EE[(Y_0 - \langle w^*, X_0 \rangle)^2] + \EE[(\langle w^* - w, X_0 \rangle)^2] = \sigma^2 + \|w - w^*\|_{\Sigma}^2 \]
where the first term is the optimal \emph{Bayes risk} and the second term, the Mahalanobis norm, is the \emph{excess risk}. Thus, minimizing \eqref{eqn:oos-prediction-error} is equivalent to estimating $w^*$ in the intrinsic Mahalanobis norm $\|\cdot\|_{\Sigma}$. Closely related is the in-sample MSE (Mean Squared Error) objective
\begin{equation}\label{eqn:mse}
\frac{1}{m} \|X w^* - X w\|_2^2 = \|w^* - w\|_{\hat \Sigma}^2. 
\end{equation}
where $\hat \Sigma = \frac{1}{m} X^T X$ is the empirical covariance matrix.
This objective does not account for the cost of generalizing to fresh samples, but it has the benefit of being defined for an arbitrary design matrix $X$, i.e. outside of the random design setting. Our positive results naturally give upper bounds on both objectives.

As we recall below (Lemma~\ref{lem:covest}), for subgaussian covariates the norms $\|\cdot\|_{\Sigma}$ and $\|\cdot\|_{\hat \Sigma}$ are equivalent up to constants when evaluated on $k$-sparse vectors and given $\Omega(k \log(n))$ samples, so as long as we consider sparse $w$ we can ignore the difference between minimizing \eqref{eqn:mse} and \eqref{eqn:oos-prediction-error}. The difference between the metrics is important in a few cases: if $\sigma = 0$, $\Sigma$ is invertible, and $w$ is potentially allowed to be dense, then minimizing \eqref{eqn:oos-prediction-error} and achieving $\|w - w^*\|_{\Sigma} = 0$ implies exact recovery of $w^*$, but any solution $Y = Xw$ minimizes \eqref{eqn:mse}, and there can be many solutions if $\hat \Sigma$ is rank degenerate. For example, the ``slow rate'' guarantee for the Lasso (see e.g. Theorem 2.15 of \cite{rigollet2015high}) with suitably normalized $X$ gives that
\[ \|w - w^*\|_{\hat \Sigma}^2 = \frac{1}{m} \|X(w - w^*)\|^2 \lesssim \sigma \|w^*\|_1 \sqrt{\frac{\log(n)}{m}}  \]
and the bound goes to zero when $\sigma = 0$, but it does not imply exact recovery in this setting, because there is an additional ``generalization'' term of $\|w^*\|_1^2/n$ in the analogous bound on the excess out-of-sample prediction error (see e.g. \cite{srebro2010optimistic,mendelson2014learning}) which does not vanish.

\paragraph{Information-theoretic upper bound (Best Subset Selection).} Instead of solving a convex program, the Best Subset Selection algorithm (see e.g. \cite{bertsimas2016best} and references within) chooses the size-$k$ support which minimizes the square loss, and has very strong statistical guarantees. For example, the following guarantee for Best Subset Selection is well known:
\begin{theorem}
Suppose $\hat{w} = \arg\min_{|\supp(w)| \le k} \|Y - X w\|^2$ in the model described above. Then as long as $m = \Omega(k \log(n/\delta))$, we have
\[ \|w^* - \hat{w}\|_{\Sigma}^2 \lesssim \|w^* - \hat{w}\|_{\hat \Sigma}^2 \lesssim \frac{\sigma^2 (k \log(n) + \log(2/\delta))}{m}. \]
\end{theorem}
\begin{proof}
The second inequality is a standard guarantee for Best Subset Selection and holds for an arbitrary fixed design; it follows, for example, by adapting the proof of Theorem 2.14 from \cite{rigollet2015high} for the (more complex) BIC estimator, or see e.g. \cite{foster1994risk}. The first inequality follows directly from Wishart concentration and the union bound; it's given formally as Lemma~\ref{lem:covest} below.
\end{proof}
Note that, just like Ordinary Least Squares (see e.g. \cite{hsu2012random}), this guarantee holds regardless of whether $\Sigma$ is well-conditioned or not.

\paragraph{Subgaussian random vectors.} For $\Sigma$ a positive semidefinite matrix, we say that a mean-zero random vector $X$ is $(L,\Sigma)$-\emph{subgaussian} if it satisfies 
\begin{equation}\label{eqn:subgaussian}
\log \EE \exp(\langle w, X \rangle) \le L^2 \langle w, \Sigma w \rangle/2.
\end{equation}
for all vectors $w \in \mathbb{R}^n$. When $\Sigma$ is omitted, we implicitly take $\Sigma = \EE X X^T$; note that definition is independent of choice of basis, e.g. any Gaussian $N(0,\Sigma)$ is a $(1,\Sigma)$-subgaussian random vector. This should not be confused with the notion of an $(L,I)$-subgaussian random vector which is defined with respect to a fixed basis/choice of Euclidean norm.
We say a mean-zero random vector $X$ is \emph{isotropic} if $\EE XX^T = I$; note that for an isotropic $L$-subgaussian random vector the rhs of \eqref{eqn:subgaussian} is just $L^2 \|w\|^2/2$.

We will use standard bounds on extremal singular values of Gaussian and sub-Gaussian random matrices (see, e.g., Corollary~5.35 in \cite{vershynin2010rmt}). 

\begin{theorem}\label{thm:rmt}
    Let $n,N \in \NN$. Let $A \in \RR^{N \times n}$ be a random matrix with entries i.i.d. $N(0,1)$. Then for any $t>0$, it holds with probability at least $1 - 2\exp(-t^2/2)$ that $$\sqrt{N} - \sqrt{n} - t \leq \sigma_\text{min}(A) \leq \sigma_\text{max}(A) \leq \sqrt{N} + \sqrt{n} + t.$$
\end{theorem}
This Theorem implies a concentration inequality for Wishart matrices, via a change of basis argument. In fact, we cite the following more general result:
\begin{theorem}[Theorem 4.6.1 of \cite{vershynin2018high}]\label{thm:rmt-general}
Suppose $X_1,\ldots,X_m$ are i.i.d. isotropic $L$-sub-Gaussian random vectors in $\mathbb{R}^n$ and let $X : m \times n$ be the matrix with row $i$ equaling $X_i$. Then with probability at least $1 - 2e^{-t}$,
\[ \left\|\frac{1}{m} X^T X - I \right\|_{OP} \le L^2 \max(\delta,\delta^2) \]
where $\delta = C\left(\sqrt{n/m} + t/\sqrt{m}\right)$.
\end{theorem}
\begin{corollary}[See e.g. Exercise 4.7.3 of \cite{vershynin2018high}]\label{corr:wishart}
Suppose $X_1,\ldots,X_m \sim N(0,\Sigma)$ with $\Sigma : n \times n$ a positive definite matrix, $t > 0$ and $m = \Omega(n + t^2)$. Let $\hat{\Sigma} = \frac{1}{m} \sum_i X_i X_i^T$. Then with probability at least $1 - 2\exp(-t^2/2)$,
\[ (1 - \epsilon) \Sigma \preceq \hat{\Sigma} \preceq (1 + \epsilon) \Sigma \]
with $\epsilon = O(\sqrt{n/m} + \sqrt{t^2/m})$.
\end{corollary}
\begin{proof}
If we let $Z_i = \Sigma^{-1/2} X_i$ so $Z_i \sim N(0,I)$, and define $Z : m \times n$ to be the matrix with rows the $Z_i$, which satisfies $X = Z \Sigma^{1/2}$ the previous Theorem~\ref{thm:rmt-general} implies
\[ (1 - \epsilon) I \preceq \frac{1}{m} Z^T Z \preceq (1 + \epsilon)I.\]
Multiplying on the left and right by $\Sigma^{1/2}$ proves the claim. 
\end{proof}

\paragraph{Interpreting Factorizations of the Precision Matrix.} 
In this paper we often (e.g. in the lower bound constructions) consider factorizations of the precision matrix of the form $\Theta = M^T M$ for some invertible matrix $M$. In the case that $M^T$ is a lower triangular matrix, so the decomposition is a Cholesky decomposition, such factorizations correspond to the choice of a \emph{Structural Equation Model} which generates the distribution $N(0,\Sigma)$. In this case, for $M$ with generic entries, determining the sparsity pattern of $\Theta$ from that of $M$ is a simple process called \emph{moralization}, which converts a directed graphical model to an undirected one. See e.g. \cite{peters2017elements,pearl2009causality,lauritzen1996graphical} for further discussion.

For general $M$ the factorization has a similar interpretation: note that if $X \sim N(0,\Sigma)$ with $\Sigma = \Theta^{-1}$ then
\[ \EE[(MX)(MX)^T] = M (M^T M)^{-1} M^T = I \]
so sampling can be interpreted as solving the equation $M X = Z$ with $Z \sim N(0,I)$. Furthermore, the sparsity pattern of $\Theta$ is still determined in a simple way, from the observation that
\[ \Theta_{ij} = \langle M_i, M_j \rangle \]
where $M_i$ is the $i$th column of $M$. Therefore if we interpret the rows of $M$ as corresponding to equations and the columns as variables, two variables $i$ and $j$ can only be connected in the undirected graphical model (corresponding to the support of $\Theta$) if they both appear in the same equation.
\subsection{$\ell_1$-regularized regression methods and guarantees}
In this section we review the definition of the Lasso and related estimators for high-dimensional linear regression (see e.g. textbooks \cite{vershynin2018high,rigollet2015high,wainwright2019high} for general references). As input, this procedure takes in a \emph{design matrix} $X$ and a \emph{response vector} $Y$. 

The Lasso 
refers to the $\ell_1$-regularized least squares estimator
\begin{equation}\label{eqn:lasso-def}
\min_w \|X w - Y\|_2^2 + \lambda \|w\|_1 
\end{equation}
or its constrained version
\[ \min_{w : \|w\|_1 \le R} \|X w - Y\|_2^2 \]
where $R,\lambda$ are related as $\lambda$ is the corresponding Lagrange multiplier of the $\ell_1$-norm constraint. In the noiseless case, the Lasso with $\lambda \to 0$ reduces to the \emph{Basis Pursuit} (BP) linear program
\[ \min_{Y = X w} \|w\|_1. \]
Note that when the goal is \emph{exact recovery}, i.e. we are interested in solutions to the equation $X w = Y$, then success of the Lasso directly implies success of Basis Pursuit (equivalently, any lower bound against Basis Pursuit implies a lower bound against Lasso). 
Beyond these methods there are further variants such as the Dantzig selector \cite{candes2007dantzig} or SLOPE (see e.g. \cite{bellec2018slope}) with different adaptivity properties, etc. but for the purposes of this paper those differences are not so relevant. All of these methods generally require control of some notion of restricted eigenvalues, as discussed below, in order to achieve nearly optimal statistical performance.

\paragraph{Restricted eigenvalues and compatibility condition.} There is a vast literature of analyses of the Lasso/basis pursuit in compressed sensing and sparse linear regression which we do not attempt to survey; instead we refer the reader to references such as \cite{wainwright2019high,vershynin2018high,rigollet2015high,van2009conditions,candes2007dantzig,bickel2009simultaneous}. 
In particular, there are many related definitions in the literature for properties of the \emph{design matrix} which guarantee the success of the Lasso in exact recovery and other versions of the sparse linear regression problem: e.g. the restricted isometry property (RIP), restricted eigenvalue (RE), and compatibility condition. The paper \cite{van2009conditions} gives a fairly exhaustive overview of the relation between these definitions; one of the weakest conditions identified there is the \emph{compatibility condition} which we define now:
\begin{definition}[Compatibility Condition (e.g. \cite{van2009conditions})]\label{def:compatability-condition}
For a positive semidefinite matrix $\Sigma : n \times n$, $L \ge 1$, and set $U \subset [n]$, we say $\Sigma$ has \emph{$(L,U)$-restricted $\ell_1$-eigenvalue}
\[ \phi_{compatible}^2(\Sigma,L,U) = \min_{w \in \mathcal{C}(L,U)} \frac{|U| \cdot \langle w, \Sigma w \rangle}{\|w_U\|^2_1}  \]
where the cone $\mathcal{C}(L,U)$ is defined as
\[ \mathcal{C}(L,U) = \{ w \ne 0 : \|w_{U^C}\|_1 \le L \|w_{S}\|_1 \}. \]
When $U$ is omitted, the $(L,k)$-restricted $\ell_1$-eigenvalue is the minimum over all $U$ of size at most $k$. We say the \emph{(L,U)-compatability condition} holds if the $(L,U)$-restricted $\ell_1$-eigenvalue is nonzero.
\end{definition}
For the purposes of this discussion, it will be enough to consider the case $L = 2$, in which case the denominator $\|w_S\|_1$ is within a constant factor of $\|w\|_1$. This assumption suffices to prove the following performance guarantee for the Lasso:
\begin{theorem}[Combined Lemma 11.1 and 11.2 of \cite{van2009conditions}]\label{thm:lasso-compatability}
Suppose that $w^*$ is supported on $U$,
$Y = X w^* + \xi$ with $X : m \times n$ and $\xi \sim N(0, \sigma^2 I_{m \times m})$,
$\delta > 0$ and $\lambda = 4\sigma \sqrt{\frac{2\log(2/\delta) + 2 \log n}{m}}$. Let $\hat{\Sigma} = \frac{1}{m} X^T X$ and suppose the diagonal of $\hat{\Sigma}$ is all-ones (i.e. the columns of $X$ are scaled to have $\ell_2$ norm $\sqrt{m}$).
If $w$ is the output of the Lasso with regularization parameter $\lambda$ then
\begin{align*} \max\left\{\| w - w^* \|_{\hat{\Sigma}}^2, \lambda \|(w^*)_{U^C}\|_1\right\} &\lesssim \frac{\lambda^2 |S|}{\phi_{compatible}^2(\hat{\Sigma},2,U)} \\
&\lesssim \frac{1}{\phi_{compatible}^2(\hat{\Sigma},2,U) } \cdot \frac{ \sigma^2|U| (\log(2/\delta) + \log n)}{m} 
\end{align*}
with probability at least $1 - \delta$, assuming the denominator is nonzero (i.e. the compatibility condition holds).
\end{theorem}
For simplicity in the Theorem statement we assumed the noise is Gaussian and $L = 2$ but the result generalizes straightforwardly to sub-Gaussian noise and other values of $L$, see \cite{van2009conditions} and other references above. 

\paragraph{Preconditioned Basis Pursuit.}
In the context of exact recovery problems (where there is no noise), the \emph{$S$-preconditioned Lasso} will generally refer to the limit of the $S$-preconditioned Lasso as $\lambda \to 0$, i.e.
\begin{equation}\label{eqn:preconditioned-lasso-exact}
\min_{w : Y = X w} \|S^T w\|_1.
\end{equation}
We also refer to this as $S$-preconditioned BP (Basis Pursuit).
As mentioned earlier, if Lasso for any $\lambda > 0$ succeeds at exact recovery its output will be a minimizer of \eqref{eqn:preconditioned-lasso-exact} as well.
\subsection{Treewidth}
We use the standard notion of treewidth of a graph; see e.g. \cite{bodlaender2005discovering} for a survey.
\begin{definition}[Treewidth]
For a graph $G = (V,E)$, a \emph{tree decomposition} is a tree $T$ where each vertex $u$ of $T$ is labeled by a \emph{bag} $B_u \subset V$ such that:
\begin{enumerate}
    \item For every edge $(i,j) \in E$, there exists a bag $B_u$ containing both $i$ and $j$.
    \item The collection of bags containing node $i$ form a connected subtree of $T$.
\end{enumerate}
The \emph{width} of $T$ is the one less than the size of the largest bag, i.e. $\max_u |B_u| - 1$, and the \emph{treewidth} of $G$, denoted $\tw(G)$, is the minimum $r$ such that a tree decomposition of $G$ with width $r$ exists.
\end{definition}
If we are given as input a graph $G$, we can compute a tree decomposition of width $O(\tw(G)\sqrt{\log \tw(G)})$ in polynomial time \cite{feige2008improved}, and a decomposition of width $O(\tw(G))$ in time $O(2^{\tw(G)}n)$ \cite{bodlaender2016c}. If $\Theta$ is a $G$-sparse matrix, i.e. the off-diagonal support of $\Theta$ corresponds to the adjacency matrix of $G$, then we let $\tw(\Theta)$ denote the treewidth of $G$. 
\section{Algorithms for Low-treewidth Precision Matrices}\label{sec:upper-bounds}
In this section we describe a computationally efficient algorithm for sparse linear regression over GGMs on low treewidth graphs. The main technical step is the construction of a sparse preconditioner for the covariance matrix $\Sigma$, which is based on a version of tree wavelets (see e.g. \cite{sharpnack2013detecting}).
Given the preconditioner, we can efficiently perform sparse linear regression by running the preconditioned Lasso on the data. In the main technical section below (Section~\ref{section:preconditioner-construction}), we show how to construct the preconditioner, even in the practically relevant situation where $\Sigma$ is unknown but the algorithm knows the graphical structure of $\Theta$ (i.e. the location of its nonzero entries). That this is possible is not at all obvious, since we are given very few samples with which to estimate $\Sigma$. In Section~\ref{section:recovery-guarantees}, we apply the guarantees of our preconditioning algorithm in conjunction with standard $\ell_1$ recovery guarantees, to show that this preconditioned Lasso solves our problem, proving Theorem~\ref{thm:upper-bound-intro}. Finally, in Section~\ref{section:model-based} we show how to obtain a further improvement in the sample complexity, matching the minimax rate for sparse linear regression on bounded treewidth graphs, using combinatorial techniques from \emph{model-based compressive sensing} \cite{baraniuk2010model}.

In what follows, we always assume the tree decomposition of the graph is known. Otherwise, we can simply apply the approximation algorithm of \cite{feige2008improved} to get an approximately optimal tree decomposition, within a factor of $O(\sqrt{\log \tw(\Theta)})$.
\subsection{Constructing a Sparse Preconditioner}\label{section:preconditioner-construction}
In this section, we show how to construct a sparse preconditioner for low-treewidth models. We consider both the case where $\Sigma$ is known exactly, in which case the preconditioner $S$ is an exact factorization $\Sigma = SS^T$, and the case where only the dependency graph is known before hand, in which case we show $SS^T$ is a good approximation of $\Sigma$ on sparse vectors (i.e. in the sense of the Restricted Isometry Property).

Let $\Theta = \Sigma^{-1}$ be the (unknown) precision matrix for an $n$-variable Gaussian Graphical Model with variables $V_1,\dots,V_n$. Let $G$ be the graph on $[n]$ encoding the sparsity structure of $\Theta$, and suppose that we are given a tree decomposition $(T,\{B_u\}_u)$ of $G$ of width $r$. 

Suppose that $X \in \RR^{m \times n}$ is a matrix whose rows are $m$ i.i.d. samples from $N(0,\Sigma)$. 
We construct a preconditioner $S \in \RR^{n \times n}$ for $X$ with sparse rows. More precisely, the goal is to ensure the matrix $X (S^T)^{-1}$, which has row $i$ equal to $S^{-1} X_i$, satisfies the Restricted Isometry Property; the row sparsity condition is needed so that the linear system $Xw^* = X(S^T)^{-1}v$ has a sparse solution $v = S^T w^*$. 

The preconditioner we describe is constructed recursively according to the structure of a \emph{centroid decomposition tree} \cite{chazelle1982theorem,guibas1987linear} built from the tree decomposition, which we define precisely next:
\begin{definition}[Centroid]
For any tree $T$ with $n$ vertices, we say vertex $v$ is a \emph{centroid} of $T$ if removing $v$ and its adjacent edges creates a forest where every tree has size at most $n/2$; this implies the resulting set of trees can be partitioned into two groups such that each group has at least $n/3$ nodes \cite{guibas1987linear}.
We recall that centroids always exist and can be found in linear time \cite{chazelle1982theorem}.

We define the \emph{centroid} of a forest in the same way. Note that the centroid of the largest tree in a forest is always a valid centroid of the entire forest. 
\end{definition}
The centroid is used to recursively construct a \emph{centroid tree} or \emph{balanced decomposition} by recursively splitting the tree at its centroid \cite{guibas1987linear}.  We call the centroid decomposition tree, formally defined below, to be the output of this recursive splitting procedure on the tree decomposition. As mentioned above, use of the centroid decomposition tree to generalize wavelets to trees has been done before, for example in \cite{sharpnack2013detecting}; we apply the same ideas to low-treewidth graphs:
\begin{definition}[Centroid Decomposition]
Given a graph over vertex set $V$ and a tree decomposition of width at most $r$, we define a (binary) \emph{centroid decomposition tree} to be any tree formed by the following recursive procedure:
\begin{enumerate}
    \item Choose a centroid of the tree decomposition and let $A$ be the corresponding bag of nodes. Removes the nodes in $A$ from all other bags of the tree decomposition. This yields a partition $A \sqcup P \sqcup Q = V$ such that $|A| \leq r$ and $|P|,|Q| \leq 2|R|/3$ and there are no edges between $P$ and $Q$.
    \item Return a binary tree with $A$ as its root and subtrees given by centroid decompositions of $P$ and $Q$.
\end{enumerate}
\end{definition}

With this definition, we can introduce our algorithm for constructing the preconditioner:

\noindent
\paragraph{Algorithm \textsc{GraphicalCholesky}($\tilde \Sigma$).}
\begin{enumerate}
\item Order the rows of $\tilde \Sigma$ according to a preorder traversal of the centroid decomposition tree (i.e. the nodes in the root first, then in the left subtree, then in the right). Initialize $S$ as a $n \times 0$ matrix.
\item Perform the following recursive procedure on pairs $(R, \tilde{\Sigma})$ where $R \subseteq [n]$, starting with $[n]$ and $\Sigmahat$:
\begin{enumerate}
\item Let $A \sqcup P \sqcup Q = R$ such that $|A| \leq r$ and $|P|,|Q| \leq 2|R|/3$ be the partition given by the root node, left subtree, and right subtree of the centroid decomposition tree.
\item 
Define $S_P,S_Q$ by recursing on $(P, \tilde{\Sigma}_{PP} - \tilde{\Sigma}_{PA} \tilde{\Sigma}_{AA}^{-1}\tilde{\Sigma}_{AP})$ and $(Q, \tilde{\Sigma}_{QQ} - \tilde{\Sigma}_{QA} \tilde{\Sigma}_{AA}^{-1} \tilde{\Sigma}_{AQ})$.
\item Return the block matrix
\begin{equation}\label{eqn:block-decomposition}
S := \begin{bmatrix}
\tilde{\Sigma}_{AA}^{1/2} & 0 & 0 \\
\tilde{\Sigma}_{PA} \tilde{\Sigma}_{AA}^{-1/2} & S_P & 0 \\
\tilde{\Sigma}_{QA} \tilde{\Sigma}_{AA}^{-1/2} & 0 & S_Q.
\end{bmatrix}
\end{equation}
\end{enumerate}
\end{enumerate}
\begin{remark}\label{rmk:graphical-cholesky}
 This algorithm can be interpreted as a natural variant of a (block) Cholesky factorization of $\tilde{\Sigma}$, with the goal of outputing $S$ such that $\tilde{\Sigma} \approx S S^T \approx \Sigma$ on sparse vectors. If desired, it is possible to replace $\tilde{\Sigma}_{AA}^{-1/2}$ with the appropriate Cholesky factors so that $S$ is genuinely lower triangular, without changing the algorithm's guarantee. In this case, the algorithm differs from Cholesky factorization of $\tilde \Sigma$ in only one crucial way: after eliminating $A$, the algorithm proceeds on a denoised version of $\tilde{\Sigma}_{P,Q | A}$ with block factorization
\begin{equation}
\begin{bmatrix} \tilde\Sigma_{P | A} & 0 \\ 0 & \tilde \Sigma_{Q | A}\end{bmatrix}
\end{equation}
since we know the off-diagonal blocks of the true conditional covariance $\Sigma_{P,Q | A}$ are indeed zero, by the Markov property. From the perspective of approximating the true covariance matrix, the denoising step avoids the accumulation of errors which would otherwise occur in the Cholesky factorization: e.g. since $\tilde{\Sigma}$ is low rank, there is no hope that the standard Cholesky factorization would even output an invertible $S$.
\end{remark}
We note that in the recursive step, the recursive call is applied to a Schur complement matrix (see e.g. \cite{ouellette1981schur}), which is just a formal definition of the result of performing a conditioning step.
\begin{definition}[Schur Complement]
For a matrix $M$ with block decomposition
\[ M = \begin{bmatrix} A & B \\ C & D \end{bmatrix} \]
the \emph{Schur Complement} of block $A$ is defined to be $M/A := D - CA^{-1}B$ provided that $A$ is invertible. Likewise, the Schur complement of block $D$ is $M/D := A - BD^{-1}C$.
\end{definition}
When $\Sigma$ is known, we have the following guarantee which arises by applying $\textsc{GraphicalCholesky}$ to the true covariance matrix $\Sigma$:
\begin{lemma}
Suppose $\Sigma$ is a positive definite matrix and $S = \textsc{GraphicalCholesky}(\Sigma)$. Then $\Sigma = SS^T$.
\end{lemma}
\begin{proof}
As discussed in Remark~\ref{rmk:graphical-cholesky}, the algorithm is exactly performing a block Cholesky factorization in this case; the entries which are zerod out in the algorithm must equal exactly zero by the Markov property.
\end{proof}
We next state the precise combinatorial sense in which the preconditioner $S$ is sparsity-preserving.
\begin{definition}[cf. \cite{baraniuk2010model}]
Given a rooted tree $\mathcal{T}$ with vertices $G_1,\ldots,G_{\ell}$ corresponding to a partition of $[n]$, we say a vector $v \in \mathbb{R}^n$ is \emph{$k$-group-tree-sparse} if its support is contained in a union of $k$ parts which form a rooted subtree of $\mathcal{T}$. We say a vector is \emph{exactly $k$-group-tree-sparse} if it is $k$-group-tree-sparse and not $(k - 1)$-group-tree-sparse.
\end{definition}
We will use the fact that projection onto the set of $k$-group-tree-sparse vectors can be performed efficiently using dynamic programming. It may be possible to obtain a faster runtime for this projection using a variant of CSSA (Condensing Sort and Select Algorithm) from \cite{baraniuk1994signal}.
\begin{proposition}[cf. \cite{baraniuk1994signal,baraniuk2010model}]
Suppose that $\mathcal{T}$ is a tree with maximum degree $O(1)$. There exists a algorithm which runs in time $poly(k) \cdot n$ to compute the projection onto $k$-group-tree-sparse vectors with respect to $\mathcal{T}$, i.e. to solve the minimization problem $\min_{x'} \|x - x'\|$ where $x'$ ranges over $k$-group-tree-sparse vectors, for arbitrary input $x$.
\end{proposition}
\begin{lemma}\label{lem:tree-sparsity-preserving}
Let $S$ be the output of Algorithm~\textsc{GraphicalCholesky} and suppose $S$ is invertible. Then $w$ is exactly $k$-group-tree-sparse with respect to the centroid decomposition tree iff $S^T w$ is exactly $k$-group-tree-sparse.
\end{lemma}
\begin{proof}
By induction, observe that each row of $S$ has nonzero entries only for nodes which are in the same part of the centroid decomposition or in the part of an ancestor. This proves the tree sparsity is not increased; because the block diagonal of $S$ is positive-definite (based on the fact $S$ is a block Cholesky decomposition and invertible) 
we see that the tree sparsity is not decreased either. Hence, $S^T$ exactly preserves group-tree-sparsity.
\end{proof}

\subsubsection{Analysis for unknown $\Sigma$}

For known $\Sigma$, it was easy to see that $SS^T = \Sigma$. In this section, we consider the more difficult (and realistic) case where the graph structure is known but the exact entries of $\Theta$ are unknown. In this case, we run \textsc{GraphicalCholesky} with the empirical covariance matrix $\hat{\Sigma} = \frac{1}{m} \sum_i X_i X_i^T$, i.e. we set $\tilde{\Sigma} = \hat{\Sigma}$. 

The main result of this section is the following Theorem, which shows that after preconditioning with $S$, the design matrix satisfies the \emph{Restricted Isometry Property} \cite{candes2005decoding}.

\begin{theorem}\label{thm:sst}
Let $\Sigma$ be an arbitrary positive definite matrix and let $\Theta = \Sigma^{-1}$. 
Let $\delta > 0$, $k \ge \tw(\Theta)$ and suppose the number of samples $m = \Omega(\frac{k\log^2(n)\log(2n/\delta)}{\epsilon^2})$. 
Then with probability at least $1 - \delta$, 
$S$ is invertible and
$$1-\epsilon \leq \frac{v^T SS^T v}{v^T \Sigmahat v} \leq 1 + \epsilon$$ 
uniformly over all $k$-sparse $v \in \RR^n$. 
\end{theorem}
The only property of the empirical covariance matrix (i.e. the only probabilistic argument) needed is the fact that its small submatrices spectrally approximate the true covariance. This fact is well-known (see e.g. \cite{raskutti2010restricted}) and we include its proof for completeness.
\begin{lemma}\label{lem:covest}
Let $k\ge 1$ and $\epsilon,\delta \in (0,1/2)$. Suppose that $m = \Omega(k \log(n/\delta)/\epsilon^2)$ . With probability at least 
$1 - \delta$,
we have that for all $A \subseteq [n]$ with $|A| \leq k$, $$(1 - \epsilon)\Sigma_{AA} \preceq \Sigmahat_{AA} \preceq (1+\epsilon)\Sigma_{AA}$$ 
and $$(1 - \epsilon)\Sigma_{AA}^{-1} \preceq \Sigmahat_{AA}^{-1} \preceq (1+\epsilon)\Sigma_{AA}^{-1}.$$
\end{lemma}
\begin{proof}
Let $A \subseteq [n]$ with $|A| \leq k$. Then by Corollary~\ref{corr:wishart}
we have $$(1 - \epsilon)\Sigma_{AA} \preceq \Sigmahat_{AA} \preceq (1+\epsilon)\Sigma_{AA}$$
with probability at least $1 - e^{-\Omega(\epsilon^2 m)}.$ Since $M \preceq N$ implies $N^{-1} \preceq M^{-1}$ for positive-definite matrices $M,N$, it follows that $$(1-\epsilon)\Sigma_{AA}^{-1} \preceq \frac{1}{1+\epsilon} \Sigma_{AA}^{-1} \preceq \Sigmahat_{AA}^{-1} \preceq \frac{1}{1-\epsilon}\Sigma_{AA}^{-1} \preceq (1+2\epsilon)\Sigma_{AA}^{-1}.$$
Union bounding over the $O(n^k)$ possible subsets $A$ completes the result.
\end{proof}


The following lemma enscapsulates the application of the Cauchy-Schwarz inequality with respect to the $\|\cdot\|_A$ norm;  it lets us bound the error between terms from the empirical and population covariance matrices.
\begin{lemma}\label{lemma:covbound}
Suppose that $A$ and $\hat{A}$ are symmetric matrices and let $\epsilon \in (0,1)$. Suppose that $A$ is positive definite and $$(1-\epsilon)A \preceq \hat{A} \preceq (1+\epsilon)A.$$
Then for any vectors $v,w$ we have that $|\langle v, (A - \hat{A}) w \rangle| \leq \epsilon \sqrt{\langle v, A v \rangle \langle w, A w \rangle}$.
\end{lemma}
\begin{proof}
Since $A$ is positive definite, it has an invertible square root $A^{1/2}$. Moreover, by rewriting the assumption we have
$$-\epsilon I \preceq I - A^{-1/2} \hat{A} A^{-1/2} \preceq \epsilon I.$$ 
So by Cauchy-Schwarz and the above operator norm bound,
\begin{align*}
|v^T (A - \hat{A}) w|
&= |(A^{1/2} v)^T (I - A^{-1/2} \hat{A} A^{-1/2}) (A^{1/2} w)| \\
&\leq \norm{A^{1/2} v}_2 \norm{(I - A^{-1/2} \hat{A} A^{-1/2}) A^{1/2} w}_2 \\
&\leq \epsilon \norm{A^{1/2}v}_2 \norm{A^{1/2} w}_2
\end{align*}
as desired.
\end{proof}

The following lemma is essentially used to bound the errors incurred in a step of block Cholesky elimination, i.e. step 2 (c) of Algorithm~\textsc{GraphicalCholesky}; we will use sparsity to always apply this Lemma in situations where the dimension $n$ in this lemma is small (i.e. not equal to the ambient dimension in our regression problem, but proportional to the sparsity).
\begin{lemma}\label{lem:covbound2}
Suppose that $\Sigma : n \times n$ is positive definite and $\tilde \Sigma$  satisfies
\[ (1 - \epsilon)\Sigma \preceq \tilde{\Sigma} \preceq (1 + \epsilon)\Sigma. \]
Then for any $v \in \mathbb{R}^n$ and $A \subset [n]$,
\[ \|\Sigma_{AA}^{-1/2} (\tilde{\Sigma}_A - \Sigma_A) v\|_2 \le \epsilon \sqrt{\langle v, \Sigma v \rangle} \]
where $\Sigma_{AA} : |A| \times |A|$ is the submatrix of $\Sigma$ given by selecting rows and columns of $A$, $\Sigma_A : |A| \times n$ is the submatrix of $\Sigma$ given by selecting the rows in $A$, and likewise for $\tilde{\Sigma}_A$.
\end{lemma}
\begin{proof}
The conclusion is equivalent to showing for all $v \in \mathbb{R}^n,w \in \mathbb{R}^{|A|}$ that
\[ |\langle \Sigma_{AA}^{-1/2} (\tilde{\Sigma}_A - \Sigma_A) v, w \rangle| \le \epsilon \sqrt{\langle v, \Sigma v \rangle} \|w\|_2. \]
We can rewrite the left-hand side by observing the term inside the absolute value is
\[ \langle (\tilde{\Sigma}_A - \Sigma_A) v,  \Sigma_{AA}^{-1/2} w \rangle = \langle (\tilde{\Sigma} - \Sigma) v, E_A \Sigma_{AA}^{-1/2} w \rangle \]
where $E_A : n \times |A|$ is the embedding operator which takes a vector $w \in \mathbb{R}^{|A|}$ to an $A$-sparse vector by padding with zeros. It follows by Lemma~\ref{lemma:covbound} that
\[ |\langle (\tilde{\Sigma} - \Sigma) v, E_A \Sigma_{AA}^{-1/2} w \rangle| \le \epsilon \sqrt{\langle v, \Sigma v \rangle \langle E_A \Sigma_{AA}^{-1/2} w, \Sigma E_A \Sigma_{AA}^{-1/2} w \rangle} = \epsilon \sqrt{\langle v, \Sigma v \rangle} \|w\|_2 \]
using in the last equality that $\Sigma_{AA} = E_A^T \Sigma E_A$.
\end{proof}
By the Markov property, we have $\langle v, (\Sigma - \Sigma_A^T \Sigma_{AA}^{-1} \Sigma_A) w \rangle = 0$ whenever $\supp(v),\supp(w)$ are separated by $A$ in the dependency graph corresponding to $\supp(\Theta)$, where $\Theta = \Sigma^{-1}$. This was the crucial property used in showing that \textsc{GraphicalCholesky} gives an exact block Cholesky factorization when applied to the true covariance matrix $\Sigma$. In the following lemma, we show that Schur complements of $\tilde{\Sigma}$ still approximately satisfy this ``conditional independence'' property when we consider sparse vectors.
\begin{lemma}\label{lemma:nodeerror}
Let $\Theta$ be a positive definite matrix and let $\Sigma = \Theta^{-1}$.
Let $A \sqcup P \sqcup Q = [n]$ be a vertex-separator partition, i.e. all paths from $P$ to $Q$ go through $A$ in the graph corresponding to the support of $\Theta$. Let $\delta \in (0,1)$ and $1 \le k \le n$ and suppose $|A| \le k$. Suppose that $\hat{\Sigma}$ is a matrix satisfying
\[ (1 - \delta) \Sigma_{BB} \preceq \hat{\Sigma}_{BB} \preceq (1 + \delta) \Sigma_{BB} \]
and
\[ (1 - \delta) \Sigma_{BB}^{-1} \preceq \hat{\Sigma}^{-1}_{BB} \preceq (1 + \delta) \Sigma^{-1}_{BB} \]
for all $B \subset [n]$ of size at most $3k$.
Then for all $k$-sparse $v,w \in \RR^n$ with $\supp(v) \subseteq P$ and $\supp(w) \subseteq Q$, we have 
\begin{equation}\label{eqn:original-goal}
|v^T(\Sigmahat - \Sigmahat_A^T \Sigmahat_{AA}^{-1} \Sigmahat_A)w| \le C \delta \sqrt{(v^T(\Sigma - \Sigma_A^T\Sigma_{AA}^{-1}\Sigma_A)v)(w^T(\Sigma - \Sigma_A^T \Sigma_{AA}^{-1} \Sigma_A)w)}.
\end{equation}
for some absolute constant $C > 0$.
\end{lemma}

\begin{proof}
First, we argue it is equivalent to show the same result with the conclusion \eqref{eqn:original-goal} replaced by
\begin{equation}\label{eqn:goal-modified}
|v^T(\Sigmahat - \Sigmahat_A^T \Sigmahat_{AA}^{-1} \Sigmahat_A)w| \le (C\delta/2) (v^T(\Sigma - \Sigma_A^T\Sigma_{AA}^{-1}\Sigma_A)v + w^T(\Sigma - \Sigma_A^T \Sigma_{AA}^{-1} \Sigma_A)w).
\end{equation}
Clearly \eqref{eqn:goal-modified} is implied by the original conclusion and the AM-GM inequality $ab \le a^2/2 + b^2/2$; in the reverse direction, we show the original conclusion by applying \eqref{eqn:goal-modified} with $v' = rv$ and $w' = v/r$ with $r > 0$; if $a = v^T(\Sigma - \Sigma_A^T\Sigma_{AA}^{-1}\Sigma_A)v$ and $b = w^T(\Sigma - \Sigma_A^T \Sigma_{AA}^{-1} \Sigma_A)w$ then using that the minimum of $f(r) = ra + b/r$ is attained at $r^*=\sqrt{b/a}$ where $f(r^*) = 2\sqrt{ab}$ proves the original conclusion \eqref{eqn:original-goal}.

We now proceed to show \eqref{eqn:goal-modified}.
Let $E_A : n \times |A|$ be the linear map which embeds a vector $v \in \mathbb{R}^{|A|}$ as an $A$-sparse vector in $\mathbb{R}^n$ by zero-padding. 
Let $f = (I - E_A \Sigma_{AA}^{-1} \Sigma_A)v$ and $g = (I - E_A \Sigma_{AA}^{-1} \Sigma_A)w$. Since $$E_A^T(\Sigmahat - \Sigmahat_A^T \Sigmahat_{AA}^{-1} \Sigmahat_A) = \Sigmahat_A - \Sigmahat_{AA} \Sigmahat_{AA}^{-1} \Sigmahat_A = 0,$$ 
and likewise $(\hat{\Sigma} - \hat{\Sigma}_A^T \hat{\Sigma}_{AA}^{-1} \hat{\Sigma}_A) E_A = 0$,
it follows that 
$$\langle v, (\Sigmahat - \Sigmahat_A^T \Sigmahat_{AA}^{-1} \Sigmahat_A)w \rangle = \langle f, (\Sigmahat - \Sigmahat_A^T \Sigmahat_{AA}^{-1} \Sigmahat_A)g \rangle$$
and it remains to upper bound the magnitude of the right-hand side. We control the terms $\langle f, \hat{\Sigma} g \rangle$ and $\langle f, \Sigmahat_A^T \Sigmahat_{AA}^{-1} \Sigmahat_A g \rangle$ separately, starting with the former.
Observe that
\begin{align*}
f^T \Sigma g
&= v^T(I - \Sigma_A^T \Sigma_{AA}^{-1} E_A^T)\Sigma(I - E_A\Sigma_{AA}^{-1} \Sigma_A)w \\
&= v^T(\Sigma - \Sigma_A^T \Sigma_{AA}^{-1} \Sigma_A)w = 0
\end{align*}
where the last equality is by the Markov property,
since $\supp(v) \subseteq P$, $\supp(w) \subseteq Q$, and $\Theta_{PQ}$ is identically $0$; equivalently, it follows from Schur complement identities. Now since $f$, $g$, and $f+g$ are $2k+|A| \le 3k$ sparse, we have that
\begin{align*}
2|f^T(\Sigma - \Sigmahat)g|
&\leq |f^T(\Sigma - \Sigmahat)f| + |g^T(\Sigma - \Sigmahat)g| + |(f+g)^T(\Sigma-\Sigmahat)(f+g)| \\
&\leq \delta|f^T\Sigma f| + \delta|g^T \Sigma g| + \delta |(f+g)^T \Sigma (f+g)| \\
& \leq 2\delta|f^T \Sigma f| + 2\delta |g^T \Sigma g| + 2\delta |f^T \Sigma g| \\
&\leq 3\delta |f^T \Sigma f| + 3\delta |g^T \Sigma g|.
\end{align*}
where in the first line we used the Parallelogram identity $2\langle f, g \rangle = \langle f + g, f + g \rangle - \langle f, f \rangle - \langle g, g \rangle$ for an arbitrary inner product space, in the second line we used Lemma~\ref{lemma:covbound}, and in the last line we used Cauchy-Schwarz and the AM-GM inequality.
To bound $f^T \Sigmahat_A^T \Sigmahat_{AA}^{-1} \Sigmahat_A g$, we decompose $$f^T \Sigmahat_A^T \Sigmahat_{AA}^{-1} \Sigmahat_A g = f^T \Sigmahat_A^T \Sigma_{AA}^{-1} \Sigmahat_A g + f^T \Sigmahat_A^T (\Sigmahat_{AA}^{-1} - \Sigma_{AA}^{-1}) \Sigmahat_A g.$$
Since $$\Sigma_{AA}^{-1} \Sigma_A g = \Sigma_{AA}^{-1}\Sigma_A(I - E_A \Sigma_{AA}^{-1} \Sigma_A)w = 0$$ and similarly $\Sigma_{AA}^{-1} \Sigma_A f = 0$, Cauchy-Schwartz and Lemma~\ref{lemma:covbound} (applied on the submatrix of $\Sigma$ indexed by $|\supp(f) \cup \supp(g) \cup A| \le 3k$) imply that $$|f^T \Sigmahat_A^T \Sigma_{AA}^{-1} \Sigmahat_A g| \leq \norm{\Sigma_{AA}^{-1/2} \Sigmahat_A f}_2 \norm{\Sigma_{AA}^{-1/2}\Sigmahat_A g}_2 \leq \delta^2 \sqrt{(f^T \Sigma f)(g^T \Sigma g)}.$$

The second term can be bounded by applying Lemma~\ref{lemma:covbound} and Lemma~\ref{lem:covbound2} as
\begin{align*}
|f^T \Sigmahat_A^T (\Sigmahat_{AA}^{-1} - \Sigma_{AA}^{-1}) \Sigmahat_A g| 
&\leq \delta \sqrt{|f^T \Sigmahat_A^T \Sigma_{AA}^{-1} \Sigmahat_A f| \cdot |g^T \Sigmahat_A^T \Sigma_{AA}^{-1} \Sigmahat_A g|} \\
&\leq \delta^3 \sqrt{(f^T \Sigma f)(g^T \Sigma g)}.
\end{align*}
Since $f^T \Sigma f = v^T (\Sigma - \Sigma_A^T \Sigma_{AA}^{-1} \Sigma_A)v$ and $g^T \Sigma g = w^T (\Sigma - \Sigma_A^T \Sigma_{AA}^{-1} \Sigma_A)w$, the claim \eqref{eqn:goal-modified} is proved.
\end{proof}

For each node $(R, \tilde{\Sigma})$ of the recursion tree with input $\tilde{\Sigma} = \Sigmahat_{RR} - \Sigmahat_{RU} \Sigmahat_{UU}^{-1} \Sigmahat_{UR}$, let $\Sigma^R = \Sigma_{RR} - \Sigma_{RU} \Sigma_{UU}^{-1} \Sigma_{UR}$. Let $S^R$ denote the matrix returned by this node in the recursion and let $\Gamma^R = (S^R)(S^R)^T$ be the corresponding estimate of $\tilde{\Sigma}$.

\begin{lemma}\label{lem:cholesky-bound}
Suppose $\Sigma$ is a positive definite matrix and $k \ge 1$.
Suppose that $\hat{\Sigma}$ is a matrix satisfying
\[ (1 - \epsilon) \Sigma_{BB} \preceq \hat{\Sigma}_{BB} \preceq (1 + \epsilon) \Sigma_{BB} \]
and
\[ (1 - \epsilon) \Sigma_{BB}^{-1} \preceq \hat{\Sigma}^{-1}_{BB} \preceq (1 + \epsilon) \Sigma^{-1}_{BB} \]
for all $B \subset [n]$ of size at most $3k$. 
Consider an arbitrary node $(R, \tilde{\Sigma})$ of the recursion tree of \textsc{GraphicalCholesky} run with input $\Sigmahat$.
Let $t$ be the depth of the recursion subtree rooted at $R$. Then
for every $k$-sparse $v$, 
$$|\langle v, (\Gamma^R - \tilde{\Sigma})v \rangle| \leq Ct\epsilon \langle v, \Sigma^R v \rangle$$
for some absolute constant $C > 0$.
\end{lemma}

\begin{proof}
We induct on $t$. We can write $\tilde{\Sigma} = \Sigmahat_{RR} - \Sigmahat_{RU} \Sigmahat_{UU}^{-1} \Sigmahat_{UR}$ for some $U \subseteq [n]$ with $|U| = O(r \log n)$, where $r$ is the width of the input tree decomposition. Let the partition be $A \sqcup P \sqcup Q = R$. If $t = 0$, then $A = R$, so $\Gamma^R = \tilde{\Sigma} \tilde{\Sigma}^{-1} \tilde{\Sigma} = \tilde{\Sigma}$, and the bound holds. Suppose $t>0$. Then
$$v^T\Gamma^R v = v^T \tilde{\Sigma}_A^T \tilde{\Sigma}_{AA}^{-1} \tilde{\Sigma}_A v + v_P^T\Gamma^P v_P + v_Q^T \Gamma^Q v_Q.$$
Let $\Xi = \tilde{\Sigma} -  \tilde{\Sigma}_A^T \tilde{\Sigma}_{AA}^{-1} \tilde{\Sigma}_A$. Then
\begin{align*}
v^T (\tilde{\Sigma} - \Gamma^R) v
&= v^T \Xi v - v_P^T\Gamma^P v_P - v_Q^T \Gamma^Q v_Q \\
&= v_P^T (\Xi_{PP} - \Gamma^P) v_P + v_Q^T (\Xi_{QQ} - \Gamma^Q) v_Q + 2v_P^T \Xi_{PQ} v_Q.
\end{align*}
By the inductive hypothesis, $$|v_P^T (\Xi_{PP} - \Gamma^P)v_P| \leq C(t-1)\epsilon \cdot v_P^T \Sigma^P v_P$$
$$|v_Q^T (\Xi_{QQ} - \Gamma^Q)v_Q| \leq C(t-1)\epsilon \cdot v_Q^T \Sigma^P v_Q.$$
By Lemma~\ref{lemma:nodeerror} and the AM-GM inequality, 
$$2|\langle v_P, \Xi_{PQ} v_Q \rangle| \leq C\epsilon(v_P^T \Sigma^P v_P + v_Q^T \Sigma^Q v_Q).$$
Letting $D = U \cup A$,
\begin{align*}
v_P^T \Sigma^P v_P + v_Q^T \Sigma^Q v_Q 
&= v_P^T (\Sigma_{PP} - \Sigma_{PD} \Sigma_{DD}^{-1} \Sigma_{DP}) v_P + v_Q^T (\Sigma_{QQ} - \Sigma_{QD} \Sigma_{DD}^{-1} \Sigma_{DQ}) v_Q \\
&= v^T(\Sigma_{RR} - \Sigma_{RD} \Sigma_{DD}^{-1} \Sigma_{DR}) v \\
&= v^T (\Sigma^R - \Sigma^R_A (\Sigma^R)_{AA}^{-1} \Sigma^R_A) v \\
&\leq v^T \Sigma^R v
\end{align*}
where the second equality uses that $D$ separates $P$ from $Q$ so that the cross-term $\Sigma_{PQ} - \Sigma_{PD}\Sigma_{DD}^{-1}\Sigma_{DQ} = 0$; the third equality uses Schur complements; and the last uses positive semi-definiteness. It follows from the above bound that $$|v^T(\tilde{\Sigma} - \Gamma^R)v| \leq Ct\epsilon v^T \Sigma^R v$$ as desired.
\end{proof}
\begin{proof}[Proof of Theorem~\ref{thm:sst}] 
Applying the above Lemma~\ref{lem:cholesky-bound} with parameter $\epsilon' > 0$ to the top node of the recursion tree, where $\Gamma^R = SS^T$ and $\tilde{\Sigma} = \hat{\Sigma}$, we get that
\[ |v^T(SS^T - \Sigmahat)v| \leq O(\epsilon' \log n) \cdot v^T \Sigma v \] 
for $k$-sparse $v$. Taking $\epsilon' = O(\epsilon/\log n)$ and combining this bound with Lemma~\ref{lem:covest}, we have proved the main claim from Theorem~\ref{thm:sst}. The fact that $S$ is invertible follows by taking the determinant of its block diagonal decomposition \eqref{eqn:block-decomposition} and using that this can be expressed as the product of determinants of small submatrices of $SS^T$, which are invertible as a consequence of the main claim. 
\end{proof}

\subsection{Recovery Guarantees}\label{section:recovery-guarantees}
In this section, we now suppose there is a ground truth sparse vector $w^*$ and give guarantees for algorithms which attempt to recover this vector.
More precisely, we consider the following well-specified setup, which slightly generalizes the model described in the Introduction to its obvious subgaussian-noise analogue. 
Suppose that $X_0 \sim N(0,\Sigma)$ (where $\Sigma$ is invertible with inverse $\Theta$) and 
\[ Y_0 = \langle w^*, X_0 \rangle + \xi_0 \]
where, conditional on $X_0$, the noise $\xi_0$ is mean zero and $\sigma^2$-subgaussian. We suppose $w^*$ is $k$-sparse. In this case, $w^*$ is a minimizer of the population squared loss
$\EE[(Y_0 - \langle w^*, X_0 \rangle)^2$.
 We will consider algorithms which attempt to recover $w^*$ given $m$ iid copies of $(X_0,Y_0)$ labeled  $(X_1,Y_1),\ldots,(X_m,Y_m)$. As before, we use the notation $X$ for the matrix with rows $X_1,\ldots,X_m$, and $Y,\xi$ for the corresponding vectors so $Y = X w^* + \xi$.
Our goal will be to prove bounds for recovering $w^*$ in the Mahalanobis norm $\|\cdot\|_{\Sigma}$, i.e. prove an upper bound on the quantity
\[ \| w - w^*\|_{\Sigma}^2 = \EE[\langle w - w^*, X_0 \rangle^2] \]
where $w$ is the output of some algorithm. More specifically, the goal will to prove an recovery guarantee of the form
\[ \|w - w^*\|_{\Sigma} \le \epsilon \]
for some absolute constant $C > 1$ and $\epsilon$ small. In fact, we will prove guarantees where $\epsilon = 0$ when $\sigma = 0$ and the number of samples $m = polylog(n)$, which means (since $\Sigma$ is positive definite) that we achieve exact recovery of the vector $w^*$.

First, we need to recall some standard terminology and results from the literature.
\begin{definition}[Restricted Isometry Constant \cite{candes2005decoding}]
For $\Sigma : n \times n$ a positive semidefinite matrix, the $k$-\emph{restricted isometry constant} is the smallest $\delta \ge 0$ such that
\[ (1 - \delta) I \preceq \Sigma_{SS} \preceq (1 + \delta) I \]
for all $S \subset [n]$ of size at most $k$. We abbreviate this condition as $(k,\delta)$-RIP for future use.
\end{definition}
It's well-known that a sufficiently small Restricted Isometry Constant implies the (weaker) Restricted Eigenvalue (RE) condition. This condition is very similar to the compatibility condition introduced before, but uses the conventional $\ell_2$-notion of eigenvalue which makes it a slightly stronger assumption \cite{van2009conditions}. To be consistent with our previous definition, we say that RE is a property of an $n \times n$ (possibly empirical) covariance matrix, whereas in \cite{bickel2009simultaneous}, the condition is equivalently stated in terms of a matrix $X : m \times n$, we state it in terms of $\hat \Sigma = \frac{1}{m} X^T X$.
\begin{definition}[Restricted Eigenvalue \cite{bickel2009simultaneous}]\label{def:restricted-eigenvalue}
We say that a matrix $\Sigma : n \times n$ satisfies the \emph{restricted eigenvalue condition}\footnote{In the notation of \cite{bickel2009simultaneous} this is $RE(k,n - k,c_0)$. Our convention agrees with Definition 7.12 of \cite{wainwright2019high}.} $RE(k,c_0)$ if 
\[ \kappa^2(k,c_0) = \min_{|U| \le s} \min_{\delta \ne 0, \|\delta_{\sim U}\|_1 \le c_0 \|\delta_U\|_1} \frac{\langle \delta, \Sigma \delta \rangle}{\|\delta\|^2} > 0.\]
\end{definition}
\begin{lemma}[\cite{bickel2009simultaneous,van2009conditions}]\label{lem:rip-to-re}
For any $\alpha > 0$, there constants $c = c(\alpha),c' = c'(\alpha) \in (0,1)$ such that the following is true.
If $\Sigma$ has a $2s$-restricted isometry constant at most $c$, then $\Sigma$ satisfies $RE(k,3 + 4/\alpha)$ with $\kappa^2(s,3 + 4/\alpha) \le c'$.
\end{lemma}
\begin{theorem}[Theorem 7.2 of \cite{bickel2009simultaneous}]\label{thm:lasso-recovery}
In the setting described above, suppose that $X : m \times n$ and $\hat \Sigma = \frac{1}{m} X^T X$ satisfies the $RE(k,3)$ assumption and the columns of $X$ are norm $\sqrt{m}$. Suppose that $m \ge k$ and $A \ge 2\sqrt{2}$.
The Lasso estimator with regularization parameter $\lambda = A \sigma\sqrt{\log(n)/m}$ outputs $\hat{w}$ satisfying
\[ \|\hat{w} - w^*\|_2^2 \lesssim \frac{A^2}{\kappa^2(k,3)} \frac{\sigma^2 k \log(n)}{m}. \]
with probability at least $1 - n^{1 - A^2/8}$.
\end{theorem}

Our Lasso-based recovery algorithm follows by combining the preconditioner, the above result, and a projection step. The projection step can be removed if we cite a more precise result about the behavior of the Lasso (see e.g. \cite{van2009conditions,wainwright2019high}), but we include it as it keeps the analysis simple.
\begin{enumerate}
    \item Use Theorem~\ref{thm:sst} to construct a preconditioner $S$ and tree $T$.
    \item Let $\hat{u}$ be the minimizer of the Lasso program
    \[ \min \|Y - X(S^T)^{-1} u\|_2^2 + \lambda \|u\|_1. \]
    \item Let $\hat{u}'$ be the projection of $\hat{u}$ onto the set of $O(k\log(n))$-group-tree sparse vectors and let $\hat{w} = (S^T)^{-1} u$. 
    \item Return $\hat{w}$. 
\end{enumerate}
This procedure has the following guarantee, completing the proof of Theorem~\ref{thm:upper-bound-intro}:
\begin{theorem}\label{thm:lasso-sst}
Suppose that $\Sigma$ is positive-definite with $\Theta = \Sigma^{-1}$. Provided that $$m = \Omega(k \tw(\Theta) \log^2(n)\log(n/\delta)),$$ the output $\hat{w}$ of the above preconditioned Lasso procedure satisfies
\[ \|w - w^*\|_{\Sigma}^2 \lesssim \frac{\sigma^2 k \tw(\Theta) \log(n) \log(n/\delta)}{m} \]
with probability at least $1 - \delta$. When $\sigma = 0$, the same guarantee holds using preconditioned BP.
\end{theorem}
\begin{proof}
First, we formally check that Theorem~\ref{thm:sst} ensures the RIP property after change of basis, which follows by explicitly writing out the definition.
Observe that if $X' = X (S^T)^{-1}$ then the empirical covariance matrix of the rows of $X'$ is
\[ \hat \Sigma_{X'} := [X']^T X' = [X (S^T)^{-1}]^T X (S^T)^{-1} = S^{-1} X^T X (S^T)^{-1}. \]
Now for an arbitrary vector $v$ and $u := (S^T)^{-1} v$, 
\[ \langle v, S^{-1} X^T X (S^T)^{-1} v \rangle = \langle (S^{-1})^T v, X^T X (S^T)^{-1} v \rangle = \langle u, X^T X u \rangle. \]
The guarantee of Theorem~\ref{thm:sst} ensures that
\[ (1 - \epsilon) \langle u, SS^T u \rangle \le \langle u, X^T X u \rangle \le (1 + \epsilon) \langle u, SS^T u \rangle \]
and $\langle u, SS^T u \rangle = \langle v, v \rangle$. This means that $\hat \Sigma_{X'}$ satisfies the restricted isometry condition with $\alpha, \beta = O(\epsilon)$. Second, we use that the ground truth vector $w^*$ is $k$-sparse, which means it is $O(k \log(n))$-group-tree-sparse (because the tree corresponding to the centroid decomposition has depth $O(\log n)$), and so by Lemma~\ref{lem:tree-sparsity-preserving}, the corresponding vector $u^* = S^T w^*$ under the change of basis is also $O(k \log(n))$-group-tree-sparse, hence $k' := O(k \tw(\Theta) \log(n))$ sparse in the ordinary sense.
As a consequence, we can apply Lemma~\ref{lem:rip-to-re} and Theorem~\ref{thm:lasso-recovery} to ensure 
\[ \|\hat{u} - u^*\|_2^2 \lesssim \frac{A^2}{\kappa^2(k,3)} \frac{\sigma^2 k' \log(n)}{m} \]
where $u^* = S^T w^*$, 
and because $u^*$ is $O(k \log(n))$-group-tree-sparse the same guarantee holds for the  $\hat{u}'$. Note (by Lemma~\ref{lem:tree-sparsity-preserving}) that $\hat{w}$ is $O(k')$ sparse and satisfies
\[ \|(S^T) \hat{w} - (S^T) w^*\|_2^2 \lesssim \frac{A^2}{\kappa^2(k,3)} \frac{\sigma^2 k' \log(n)}{m} \]
so be appealing to Theorem~\ref{thm:sst} again, the same guarantee holds for $\|\hat{w} - w^*\|_{\Sigma}^2$ provided $m = \Omega(k' \log^2(n)\log(n/\delta))$. The conclusion follows by taking $A^2 = \Theta(1 + \log(2/\delta)/\log(n))$.

The guarantee for preconditioned BP follows by considering the limit $\sigma \to 0$ (which means that $\lambda \to 0$), or by using directly the result of e.g. \cite{candes2005decoding}.
\end{proof}
\subsection{Model-Based Iterative Hard Thresholding}\label{section:model-based}
We can slightly improve the guarantee obtained with the preconditioned Lasso, using the model-based Iterative Hard Thresholding (IHT) approach \cite{baraniuk2010model}, which takes advantage of the fact that our preconditioner preserves $k$-group-tree-sparsity. In this case, the algorithm is:
\begin{enumerate}
    \item Use Theorem~\ref{thm:sst} to construct a preconditioner $S$.
    \item Apply group-tree-sparse IHT (Lemma~\ref{lem:iht}) to the sparse linear model
    \[ Y = X (S^T)^{-1} u^* + \xi \]
    where $u^* = S^T w^*$, let $\hat u$ be the output of IHT, and let $\hat{w} = (S^T)^{-1} u^*$.
    \item Return $\hat{w}$.
\end{enumerate}
and its guarantee is given in Theorem~\ref{thm:iht-sst} below.\\\\
\noindent
Algorithm~\textbf{IHT}($\mathcal{S},T,X,Y$):
\begin{enumerate}
    \item Set $w_0 = 0$.
    \item For $t = 1$ to $T$:
    \begin{enumerate}
        \item Set $u_{t} = w_{t - 1} + \frac{1}{m} X^T(Y - X w_{t - 1})$.
        \item Set $w_{t} = \Proj_{w : \supp(w) \in \mathcal{S}}[u_t]$.
    \end{enumerate}
    \item Return $w_T$.
\end{enumerate}
The key observation motivating tree sparsity is the following Lemma, which can be proved using properties of Catalan numbers; it shows that the number of $k$-(group)-tree sparse supports grows like $e^k$ instead of $n^k$ for the number of total $k$-(group)-sparse vectors.
\begin{lemma}[cf. Proposition 1 of \cite{baraniuk2010model}]\label{lem:num-tree-sparse}
The number of distinct $k$-group-tree-sparse supports is at most $\frac{(2e)^k}{k + 1}$.
\end{lemma}

We proceed to state a variant of the guarantee for Iterative Hard Thresholding in our setting. We defer the proof to the Appendix, since it is a fairly straightforward variant of known results (see e.g. \cite{blumensath2009iterative,jain2014iterative,jain2016structured,baraniuk2010model}). The key difference vs. considering unstructured sparse vectors is that we can eliminate the presence of $\log(n)$ from the upper bound; in our application the $k$ we use will itself have a factor of $\log(n)$, which is why we end up with a single instead of a squared log factor.
\begin{lemma}[See Appendix~\ref{apdx:iht}]\label{lem:iht}
Suppose $X : m \times n$ and $\hat \Sigma = \frac{1}{m} X^T X$ is $(3kr,\alpha)$-RIP 
with $\alpha < 1/3$ and $\mathbb{Y} = X w^* + \xi$ where $w^*$ is $k$-group-tree-sparse with respect to a tree with node groups of size at most $r$, 
and $\xi$ is a random vector in $\mathbb{R}^m$ with independent $\sigma^2$-subGaussian entries. Then Iterative Hard Thresholding with projection onto the set of $k$-group-tree-sparse vectors and $T = O(\log(\|w^*\|nm/\sigma))$ succeeds to recover $w$ such that
\[ \|w^* - w_t\| \lesssim \sigma \sqrt{\frac{kr + \log(2/\delta)}{m}} \]
with probability at least $1 - \delta$ over the randomness of $\xi$.
\end{lemma}


\begin{remark}
Note that we can apply this Lemma with $\sigma' > \sigma$ an error parameter if the original $\sigma$ is very small or zero.
We expect it is possible to eliminate the logarithmic dependence in the runtime on the bit-complexity related term $\|w^*\|/\sigma$ if exact real arithmetic is used and the algorithm is slightly modified; cf. the Appendix to \cite{needell2009cosamp}.
\end{remark}
\begin{theorem}\label{thm:iht-sst}
Provided $m = \Omega(k \tw(\Theta) \log^2(n)\log(n/\delta))$, the output $\hat{w}$ of the above preconditioned IHT procedure satisfies
\[ \|w - w^*\|_{\Sigma}^2 \lesssim \frac{\sigma^2 (k \tw(\Theta) \log(n) + \log(2/\delta))}{m} \]
with probability at least $1 - \delta$.
\end{theorem}
\begin{proof}
This follows by combining Theorem~\ref{thm:sst} with Lemma~\ref{lem:iht},
just as in Theorem~\ref{thm:lasso-sst}.
\end{proof}
\section{Failure of poorly preconditioned Lasso}\label{section:ill-conditioned}

In this section, we begin studying the conditions under which preconditioned Lasso fails. Through this section (and subsequent sections), we will be studying \emph{noiseless} sparse linear regression, and proving lower bounds against the $S$-preconditioned Lasso in the limit $\lambda \to 0$, which may also be called the $S$-preconditioned basis pursuit:
$$\min_{w: Xw = Y} \norm{S^T w}_1.$$
A ``lower bound" for a positive definite precision matrix $\Theta$ and a preconditioner $S$ is a statement of the following form: for some $k$-sparse signal, $S$-preconditioned basis pursuit requires at least $m$ samples to succeed at exact recovery with high probability, when the covariates are drawn independently from $N(0,\Theta^{-1})$. Ultimately, we will construct precision matrices $\Theta$ such that for every preconditioner, we can prove a lower bound. As previously mentioned, exact recovery lower bounds against $S$-preconditioned basis pursuit automatically apply to $S$-preconditioned Lasso for any $\lambda>0$, so in our lower bound statements we simply refer to the $S$-preconditioned Lasso.

We start by introducing the \emph{Weak ($S$-Preconditioned) Compatibility Condition}, together with a quantitative \emph{Weak Compatibility Ratio}, and compare it to the well-studied compatibility condition \cite{van2009conditions}. The main result of this section is that if the Weak $S$-Preconditioned Compatibility Ratio exceeds an absolute constant, then $S$-preconditioned Lasso necessarily fails at exact recovery.


For a covariance matrix $\Sigma$ and a preconditioner $S$, we define the following measure of how well $S$ preconditions $\Sigma$; it depends on the signal sparsity $k$ and the number of samples $m$. This definition requires that $S$ is not identically zero:\footnote{We will ignore this corner case henceforth, since the $0$-preconditioned basis pursuit is an underdetermined linear program and therefore cannot provably succeed.}

\begin{definition}[Weak $S$-Preconditioned Compatibility Condition]\label{def:weak-compatibility} 
We say that $$\alpha^{(1)}_{\Sigma,S,k} = \inf_{w \in B_0(k) \setminus \{0\}} \frac{\langle w, \Sigma w \rangle}{\norm{S^T w}_1^2}$$
where we recall from \eqref{eqn:k-sparse} that $B_0(k)$ denotes the set of $k$-sparse vectors
and
$$\beta^{(1)}_{\Sigma,S,m,k} = \sup\{\beta \in \RR: \dim W_{\Sigma,S,\beta} \geq 2m\}$$
where $$W_{\Sigma,S,\beta} = \left\{w: \langle w, \Sigma w \rangle \geq \beta \norm{S^T w}_1^2\right\},$$
and $\dim W_{\Sigma,S,\beta}$ is defined as the largest dimension of any subspace contained in $W_{\Sigma,S,\beta}$.

We say the $(\Sigma,S,k)$-\emph{weak $S$-preconditioned compatibility condition} is satisfied if $\alpha^{(1)}_{\Sigma,S,k} > 0$, in which case we define the \emph{weak $S$-preconditioned compatibility ratio} to be $$\gamma^{(1)}_{\Sigma,S,m,k} = \frac{\beta^{(1)}_{\Sigma,S,m,k}}{\alpha^{(1)}_{\Sigma,S,k}} > 0.$$
\end{definition}

We will only be concerned with invertible $\Sigma$, in which case the weak $S$-preconditioned compatibility condition always holds, so the weak compatibility ratio is always defined.

\begin{remark}[Comparison to Compatibility Condition]\label{rmk:compatibility-vs-weak}
In the non-preconditioned case $S = I$, the above definition is similar to the compatibility condition (Definition~\ref{def:compatability-condition}) with some important differences that make Definition~\ref{def:weak-compatibility} \emph{weaker}, i.e. less difficult to satisfy:
\begin{enumerate}
    \item The $\ell_1$-eigenvalue in the compatibility condition is defined with respect to the larger \emph{cone} $\mathcal{C}(L,S)$, for all sets $S$ of size at most $k$, instead of with respect to the set of $k$-sparse vectors as in $\alpha^{(1)}_{\Theta,S,k}$. The stronger requirement in the compatibility condition is needed for the Lasso analysis to succeed.\footnote{Note that \emph{sufficiently strong control} over $k$-sparse vectors as in a sufficiently small RIP (Restricted Isometry Property) constant suffices to also control conditioning over the set $\mathcal{C}(L,S)$, see e.g. \cite{candes2007dantzig,wainwright2019high,van2009conditions}, but for large RIP constants this argument breaks down which is why more general Lasso guarantees are stated in terms of the behavior over the cone.} 
    \item When the compatibility condition is used (e.g. as in its use in Theorem~\ref{thm:lasso-compatability}) it is assumed that the diagonal of $\Sigma$ is at most $1$. Since (by convexity) the maximum of a convex function on a convex polytope is obtained at the extreme points,
    \[ \sup_{\|w\|_1 = 1} \langle w, \Sigma w \rangle = \max_i \Sigma_{ii}, \]
    this is the same as requiring $\dim W_{\Sigma,S,\beta} = 0$ for all $\beta > 1$, i.e. there are no ``large'' directions of $\Sigma$. In contrast, Definition~\ref{def:weak-compatibility} requires only that $\dim W_{\Sigma,S,\beta} < 2m$, i.e. there are not too many ``large'' directions.
\end{enumerate}
Combined, these differences mean that $\gamma^{(1)}_{\Sigma,S,m,k}$ can be significantly smaller than the $\ell_1$-eigenvalue $\phi_{compatibility}^2(\Sigma,k)$ or similar quantities. We illustrate this in Example~\ref{ex:weak-compatibility} below.
\end{remark}
\begin{example}\label{ex:weak-compatibility} 
We consider the example $\Theta = \Sigma = I$ with no preconditioning, i.e. $S = I$. In this case $\alpha^{(1)}_{I,I,k} = \inf_{w \in B_0(k) \setminus \{0\}} \frac{ \|w\|_2^2}{\|w\|_1^2} = 1/k$ as $\|w\|_1^2 \le k \|w\|_2^2$ by Cauchy-Schwarz and sparsity. Furthermore \[ \dim W_{I,I,\beta} = \dim \left\{w \in \mathbb{R}^n : \frac{ \|w\|_2^2}{\|w\|_1^2} \ge \beta \right\} = \lfloor 1/\beta \rfloor \]
by Theorem~\ref{lem:botelho} below, so $\beta^{(1)}_{I,I,m,k} = 1/2m$. Combining, we see that $\gamma^{(1)}_{I,I,m,k} = \frac{k}{2m}$. 
\end{example}
\begin{theorem}[Theorem~2.5 of \cite{botelho2017exact}]\label{lem:botelho}
A $k$-dimensional subspace $V$ of $\mathbb{R}^n$ satisfies the inequality
\[ \|x\|_1^2 \le k \|x\|_2^2 \]
iff $V$ is the span of a set of $k$ standard basis vectors.
\end{theorem}
As explained in Remark~\ref{rmk:compatibility-vs-weak},
to prove that preconditioned Lasso succeeds, stronger conditions present in the literature must be assumed. 
Thus, assuming that $\gamma^{(1)}_{\Sigma,S,m,k}$ be near $1$ is \emph{not a sufficient condition} for exact recovery. 
However, we show that it is \emph{necessary}, which will aid us in proving lower bounds against preconditioned Lasso.

We could also have defined similar quantities $\alpha$ and $\beta$ purely in terms of $\ell_2$ norms as follows; this definition is closer to the classical Restricted Isometry Property (RIP) or Restricted Eigenvalue (RE) condition, but the corresponding lower bound loses a factor of $s$ (where $s$ is the number of columns in the preconditioner, i.e. $S \in \RR^{n \times s}$).

\begin{definition}[Weak RE Condition]
We say that $$\alpha_{\Sigma, S, k} = \inf_{w \in B_0(k) \setminus \{0\}} \frac{w^T \Sigma w}{w^T SS^T w}$$
and
$$\beta_{\Sigma,S,m,k} = \inf \{\beta \in \RR: \lambda_{2m}(\Sigma - \beta SS^T) \leq 0\}.$$
Define $\gamma_{\Sigma,S,m,k} = \beta_{\Sigma,S,m,k}/\alpha_{\Sigma,S,k}$.
\end{definition}
The following Lemma relates the two definitions at the cost of the aforementioned dimension factor.
\begin{lemma}\label{lem:l1-to-l2-condition}
For any positive-definite $\Sigma \in \RR^{n \times n}$, preconditioner $S \in \RR^{n \times s}$, and integer $m,k>0$, the following relations hold: $\alpha_{\Sigma,S,k} \ge \alpha^{(1)}_{\Sigma,S,k}$, $\beta_{\Sigma,S,m,k} \le s \cdot \beta^{(1)}_{\Sigma,S,m,k}$, and $\gamma^{(1)}_{\Sigma,S,m,k} \ge \gamma_{\Sigma,S,m,k}/s$.
\end{lemma}
\begin{proof}
Notice that for any $w \in B_0(k) \setminus \{0\}$, $$\frac{w^T \Sigma^{-1} w}{w^T SS^T w} \geq \frac{w^T \Sigma^{-1} w}{\norm{S^T w}_1^2} \geq \alpha^{(1)}_{\Sigma,S,k}$$
and thus $\alpha_{\Sigma,S,k} \geq \alpha^{(1)}_{\Sigma,S,k}$. Similarly, pick any $\beta < \beta_{\Sigma,S,m,k}/s$. We know that $\lambda_{2m}(\Sigma - s\beta SS^T) > 0$. So 
$$w^T \Sigma w \geq s\beta\norm{S^T w}_2^2 \geq \beta\norm{S^T w}_1^2$$ 
for any $w$ in the span of the top $2m$ eigenvectors of $\Sigma - s\beta SS^T$. Therefore $\dim W_{\Sigma,S,\beta} \geq 2m$ and hence $\beta^{(1)}_{\Sigma,S,m,k} \geq \beta$. We conclude that $\beta^{(1)}_{\Sigma,S,m,k} \geq \beta_{\Sigma,S,m,k}/s$, so $\gamma^{(1)}_{\Sigma,S,m,k} \geq \gamma_{\Sigma,S,m,k}/s$.
\end{proof}

Using the Weak Compatibility Condition, we prove an upper bound on the probability that preconditioned Lasso succeeds at exact recovery with $m$ samples, if $\gamma^{(1)}_{\Sigma,S,m,k}$ is large.

\begin{theorem}\label{theorem:ill-conditioned-lasso-failure-l1}
Let $\Sigma \in \RR^{n \times n}$ be positive-definite and let $S \in \RR^{n \times s}$. Let $m,k \in \NN$. If $\gamma^{(1)}_{\Sigma,S,m,k} > 18$, then there is a $k$-sparse signal $w^* \in \RR^n$ such that the $S$-preconditioned Lasso exactly recovers $w^*$ with probability at most $\exp(-\Omega(m))$, from $m$ samples with independent covariates $X_1,\dots,X_m \sim N(0,\Sigma)$ and noiseless responses $Y_i = \langle w^*, X_i \rangle$.
\end{theorem}

\begin{proof}
For convenience of notation let $\alpha = \alpha^{(1)}_{\Sigma,S,k}$ and $\beta = \beta^{(1)}_{\Sigma,S,m,k}$ and $\Theta = \Sigma^{-1}$. We want to show that there is $k$-sparse $w^* \in \RR^n$ such that with high probability, the $S$-preconditioned Lasso \eqref{eqn:preconditioned-lasso-exact} fails to recover $w^*$, i.e. $$w^* \not \in \argmin_{w: Xw = Xw^*} \norm{S^T w}_1$$
where $X$ has rows $X_1,\dots,X_m \sim N(0,\Sigma)$. 
The set $\{w : \|S^T w\|_1 = 1\}$ is a compact set, so by homogeneity and the definition of $\alpha$ we can find $k$-sparse $w^* \in \RR^n$ such that $$(w^*)^T \Sigma w^* = \alpha \norm{S^T w^*}_1^2.$$

By definition of $\beta$, there is a subspace $U \subseteq \RR^n$ of dimension $2m$ such that $w^T \Sigma w \geq \beta \norm{S^T w}_1^2$ for all $w \in U$. Let $v_1,\dots,v_{2m} \in U$ form an orthonormal basis for $U$, and let $V \in \RR^{n \times 2m}$ be the matrix with columns $v_1,\dots,v_{2m}$.

We construct $v \in \RR^n$ to satisfy $Xw^* = Xv$ and $\norm{S^T v}_1 < \norm{S^T w^*}_1$ as follows. Let $\Gamma = V^T \Sigma V \in \RR^{2m \times 2m}$. The columns of $V$ have no linear dependencies, and $\Sigma$ is symmetric positive-definite, so $\Gamma$ is symmetric positive-definite. Thus, there is an invertible matrix $N \in \RR^{2m \times 2m}$ such $\Gamma = N^T N$. Define $$c = N^{-1} (XVN^{-1})^\pinv Xw^* \in \RR^{2m}$$ and define $v = Vc \in \RR^n$. By construction we have $v \in U$, so \begin{equation} v^T \Sigma v \geq \beta \norm{S^T v}_1^2.\end{equation}
Second, note that $$\EE[(XVN^{-1})^T(XVN^{-1})] = m(N^{-1})^T V^T \Sigma V N^{-1} = m(N^{-1})^T \Gamma N^{-1} = mI_{2m}.$$
Moreover, the rows of $XVN^{-1}$ are independent and Gaussian. So in fact $XVN^{-1}$ has i.i.d. $N(0,1)$ entries. Thus, with probability $1 - \exp(-\Omega(m))$, we have $\sigma_\text{min}((XVN^{-1})^T) \geq \sqrt{m}/3$ since the dimensions of $(XVN^{-1})^T$ are $2m \times m$ (by Theorem~\ref{thm:rmt}). Hence, $\sigma_\text{max}((XVN^{-1})^\pinv) \leq 3/\sqrt{m}$. We can conclude that \begin{equation} v^T \Sigma v = c^T N^T N c = (w^*)^T X^T (XVN^{-1})^{\pinv T} (XVN^{-1})^\pinv Xw^* \leq (9/m)(w^*)^T X^T X w^*. \end{equation}
We can now check that $\norm{S^Tv}_1 < \norm{S^T w}_1$. Indeed, $(w^*)^T X^T X w^* \leq 2m (w^*)^T \Sigma w^*$ with probability at least $1 - \exp(-\Omega(m))$, so
\begin{align*}
\norm{S^T v}_1
&\leq \sqrt{\frac{1}{\beta} v^T \Sigma v} \\
&\leq \sqrt{\frac{9}{m\beta} (w^*)^T X^T X w^*} \\
&\leq \sqrt{\frac{18}{\beta} (w^*)^T \Sigma w^*} \\
&\leq \sqrt{\frac{18 \alpha}{\beta}} \norm{S^T w^*}_1
\end{align*}
which produces the desired inequality as long as $\beta/\alpha > 18$.

Finally, since $XVN^{-1}$ is rank-$m$ with probability $1$, we have $(XVN^{-1})(XVN^{-1})^\pinv = I_m$, and thus \begin{equation} Xv = XVN^{-1}(XVN^{-1})^\pinv Xw^* = Xw^* \end{equation}
as desired.
\end{proof}

The above result can be directly extended to use the Weak RE Condition instead:

\begin{corollary}
Let $\Sigma \in \RR^{n \times n}$ be positive-definite and let $S \in \RR^{n \times s}$. If $\gamma_{\Sigma,S,m,k} > 18s$, then there is a $k$-sparse signal $w^* \in \RR^n$ such that the $S$-preconditioned Lasso exactly recovers $w^*$ from $m$ samples with probability at most $\exp(-\Omega(m))$, over the randomness of independent covariates $X_1,\dots,X_m \sim N(0,\Sigma)$.
\end{corollary}

\begin{proof}
By the assumption that $\gamma_{\Sigma,S,m,k} > 18s$ we get from Lemma~\ref{lem:l1-to-l2-condition} that $\gamma^{(1)}_{\Sigma,S,m,k} > 18$, so Theorem~\ref{theorem:ill-conditioned-lasso-failure-l1} yields the result.
\end{proof}

\subsection{The random walk example}

In this section, we illustrate the utility of the Weak Compatibility Condition by applying it to the random walk example introduced in Section~\ref{section:introduction}. Let $Z_1,\dots,Z_n \sim N(0,1)$ be independent standard Gaussian random variables, and let $R_1 = Z_1$ and $R_i = R_{i-1} + Z_i$ for $2 \leq i \leq n$. Then $(R_1,\dots,R_n) \sim N(0,\Sigma)$ where $\Sigma_{ij} = \min(i,j)$ for $i,j \in [n]$. We derive lower bounds for (a) directly applying Lasso, and (b) normalizing the variances to $1$ before applying Lasso. These correspond to two different diagonal preconditioners $S$.

\paragraph{No preconditioning ($S = I$).} We take $S = I_n$. Then we can see that
$$\alpha_{\Sigma,S,k}^{(1)} \leq \frac{e_1^T \Sigma e_1}{\norm{S^Te_1}_1^2} = 1.$$
On the other hand, let $t = n/(2m)$ and define $$W = \text{span}\{e_t,e_{2t},\dots,e_n\}.$$
Let $v \in W$. Then
\begin{align*}
\Var(vR)
&= \Var(v_tR_t + v_{2t}R_{2t} + \dots + v_nR_n) \\
&= \Var\left(\sum_{i=1}^{n/t} (v_{it} + v_{(i+1)t} + \dots + v_n)(Z_{(i-1)t+1}+\dots+Z_{it})\right) \\
&= t\sum_{i=1}^{n/t} (v_{it} + v_{(i+1)t} + \dots + v_n)^2.
\end{align*}
By the inequality $x^2+(a+x)^2 \geq a^2/2$, note that for any $1 \leq i < n/t$ we have
$$(v_{it}+v_{(i+1)t}+\dots+v_n)^2 + (v_{(i+1)t} + \dots + v_n)^2 \geq \frac{v_{it}^2}{2}.$$
Thus, $$\Var(vR) \geq \frac{tv_n^2}{2} + \frac{t}{2}\sum_{i=1}^{n/t-1} \frac{v_{it}^2}{2} \geq \frac{t}{4}\norm{v}_2^2.$$
Since $v$ is $2m$-sparse, this means that $$\frac{v^T\Sigma v}{\norm{S^T v}_1^2} \geq \frac{t}{4}\frac{\norm{v}_2^2}{\norm{v}_1^2} \geq \frac{t}{8m} = \frac{n}{16m^2}$$ for all $v \in W$. Since $\dim(W) = 2m$, it follows that $\beta_{\Sigma,S,m,k}^{(1)} \geq n/(16m^2)$. By Theorem~\ref{theorem:ill-conditioned-lasso-failure-l1}, we get that for any $k \geq 1$ and $m < \sqrt{n}/(12\sqrt{2})$, there is a $k$-sparse combination of $X_1,\dots,X_n$ which Lasso with high probability fails to learn with $m$ samples.

\paragraph{Normalize variance.} We take $S$ to be the diagonal matrix with $S_{ii} = \sqrt{i}$ for $i \in [n]$. Then
$$\alpha_{\Sigma,S,k}^{(1)} \leq \frac{(e_n-e_{n-1})^T\Sigma(e_n-e_{n-1})}{\norm{S^T(e_n-e_{n-1})}_1^2} = \frac{1}{\left(\sqrt{n} + \sqrt{n-1}\right)^2} \leq \frac{1}{4(n-1)}.$$
On the other hand, define $t = n/(2m)$ and $W = \Span\{e_{t},e_{2t},\dots,e_n\}$. Let $v \in W$. As in the previous setting, it holds that $$\Var(vR) \geq \frac{t}{4}\norm{v}_2^2.$$
To bound $\norm{S^Tv}_1^2$ we simply use that $\norm{S^Tv}_1 \leq \norm{v}_1\sqrt{n}$. Thus, using that $v$ is $2m$-sparse, $$\frac{v^T \Sigma v}{\norm{S^T v}_1^2} \geq \frac{t}{8n} \frac{\norm{v}_2^2}{\norm{v}_1^2} \geq \frac{t}{16mn} = \frac{1}{32m^2}.$$
As $\dim(W) = 2m$ it follows that $\beta_{\Sigma,S,m,k}^{(1)} \geq 1/(32m^2)$, so $\gamma_{\Sigma,S,m,k}^{(1)} \geq (n-1)/(8m^2).$
So by Theorem~\ref{theorem:ill-conditioned-lasso-failure-l1}, we get that for any $k \geq 1$ and $m < \sqrt{n-1}/12$, there is a $k$-sparse combination of $X_1,\dots,X_n$ which Lasso with high probability fails to learn with $m$ samples.

To summarize these two impossibility results, we have the following theorem:

\begin{theorem}\label{theorem:random-walk-lb}
Let $R_1,\dots,R_n \sim N(0,\Sigma)$ be the standard Gaussian random walk defined above. For any $k\geq 2$ and $m < \sqrt{n}/(12\sqrt{2})$, we have that for some $k$-sparse signal, Lasso (with no preconditioning) with $m$ independent noiseless samples from $N(0,\Sigma)$ succeeds at exact recovery with probability at most $\exp(-\Omega(m))$. Moreover, if $S$ is the diagonal preconditioner with $S_{ii} = \sqrt{i}$ for $i \in [n]$, and if $m < \sqrt{n-1}/12$, then once again the $S$-preconditioned Lasso succeeds with probability at most $\exp(-\Omega(m))$ for some $k$-sparse signal.
\end{theorem}


\section{Lower bounds from dense least-eigenspace}\label{section:lower-bounds}

In the previous section, we showed that for preconditioned Lasso to succeed, the preconditioner must be ``compatible" with $\Sigma$, in the sense that $\gamma_{\Sigma,S,m,k}$ cannot be too large. In this section, we show that if the precision matrix has certain structure, then any compatible preconditioner must be ``dense" in a strong sense, that we show suffices to prove that preconditioned Lasso fails.

More concretely, let $\Theta$ be positive semi-definite, and define the precision matrix to be $\Thetat = \Theta + \epsilon I$, where $\epsilon>0$ is very small. For a preconditioner $S$ to be compatible with $\Sigma = \Thetat^{-1}$, it roughly holds that $SS^T$ spectrally approximates $\Sigma$. However, $\Sigma$ is very large in directions correlated with $\ker(\Theta)$, and small on $\rspan(\Theta)$. This means that the columns of $S$ must roughly lie in $\ker(\Theta)$.

Guided by the intuition that dense preconditioners should cause Lasso to fail, we make the assumption that $\ker(\Theta)$ only contains vectors that are dense in a quantitative and robust sense. In this case, the columns of $S$ must be either dense or have very small norm (so as to not contribute bad directions to $SS^T$). This is formalized and proven in Section~\ref{section:dense-structure}.

But does this structural property of $S$ cause $S$-preconditioned Lasso to fail? Intuitively, the number of dense columns of $S$ should be at least $\dim \ker(\Theta)$, so if the number of samples $m$ is smaller, then the preconditioned Lasso should fail. This argument can be formalized but is technically involved; see Section~\ref{section:dense-failure} for details.



\subsection{Structural properties of compatible preconditioners}\label{section:dense-structure}

First, we will want the following notation to express robust density:

\begin{definition}\label{def:dist}
For $v \in \RR^n$ and $k \in \NN$ define $\dist_k(v) = \inf_{w \in B_0(k)} \norm{v-w}_2$. Moreover, for $V \subseteq [n]$ define $\dist_{k,V}(v) = \dist_k(v_V)$.
\end{definition}

Now we can formally state the theorem:


\begin{theorem}\label{thm:columns-dense-or-small-l1}
Let $\Theta \in \RR^{n \times n}$ be a PSD matrix, and let $\epsilon > 0$. Let $k,m > 0$. Suppose that $\Theta$ is $k$-sparse (i.e. every row of $\Theta$ has at most $k$ nonzero entries) and that $r := \dim \ker(\Theta) > 2m$. Let $\tau > 0$ and $V \subseteq [n]$ and suppose that $$\eta := \min_{x \in \ker(\Theta) \setminus \{0\}} \frac{\dist_{\tau,V}(x)}{\norm{x}_2} > 0.$$
Let $\lambda$ be the smallest non-zero eigenvalue of $\Theta$. Let $\tilde{\Theta} = \Theta + \epsilon I$.
Let $S : n \times s$ and $\gamma = \beta^{(1)}_{\Thetat^{-1},S,k,m}/\alpha^{(1)}_{\Thetat^{-1},S,k}$. Let $$\delta = \frac{3}{\eta} \sqrt{\frac{\epsilon \gamma n}{\lambda^3}} \norm{\Theta}_F.$$ Then there is a subset of column indices $D \subseteq [s]$ (of the matrix $S$) with the following properties:
\begin{itemize}
\item For any $i \in D$, the column $v = (S^T)_i$ satisfies $\dist_{\tau,V}(v) \geq (\eta/3)\norm{v}_2$ (quantitatively dense)
\item The submatrix $(S^T)_{D^c}$ satisfies $\norm{(S^T)_{D^c}}_{2\to 1} \leq \delta \norm{S^T}_{2\to 1}.$ (small norm)
\end{itemize}
\end{theorem}

Before proving this theorem, we parse the notation. Essentially, the theorem shows that if $\ker(\Theta)$ contains only robustly dense vectors with respect to some subset of coordinates $V \subseteq [n]$ (i.e. $\eta$ is bounded away from $0$), and if the preconditioner well-conditions $\Theta$ (i.e. $\gamma$ is not too large) then for sufficiently small $\epsilon$, every column of the preconditioner is also quantitatively dense over $V$, except for a set of potentially sparse columns $D^c$ with small total norm.

As $\epsilon$ approaches $0$, the condition number of $\Thetat$ degrades (which is why we do not wish to take the limit $\epsilon \to 0$), but the bound $\delta$ on the norm of the sparse columns improves. Indeed, in the definition of $\delta$, the parameters $\eta$, $\lambda$, and $\norm{\Theta}_F$ all depend on $\Theta$ but not on $\epsilon$. So if we assume that the weak compatibility ratio $\gamma$ is bounded by a constant (so that Theorem~\ref{theorem:ill-conditioned-lasso-failure-l1} does not apply), then if $\eta, \lambda \geq 1/\poly(n)$ and $\norm{\Theta}_F \leq \poly(n)$, we can take $\epsilon = \Omega(1/\poly(n))$ and get $\delta$ to be inverse polynomially small.

\begin{remark}
An interesting caveat of the above theorem is the requirement that $\Theta$ be $k$-sparse. This is not a limitation for our final results, since we ultimately are able to construct hard instances with sparse precision matrices, but it illustrates why Gaussian Graphical Models are a useful tool in the broader context of constructing hard random-design instances for Lasso. A priori, it is not clear whether assuming sparsity of the precision matrix should help or hinder lower bounds, but in this framework it plays a crucial role.
\end{remark}

\begin{proof}[\textbf{Proof of Theorem~\ref{thm:columns-dense-or-small-l1}}]
Let $\Theta = \sum_{i=1}^n \lambda_i v_i v_i^T$ be an orthonormal diagonalization with $\lambda_1 \geq \dots \geq \lambda_n \geq 0$. 
Then $(\Theta + \epsilon I)^{-1} = \sum_{i=1}^n (\lambda_i + \epsilon)^{-1} v_i v_i^T$. Recalling that $\lambda_{n-r+1} = \dots = \lambda_n = 0$, for any $w \in \RR^n$, $$w^T \Thetat^{-1} w = \sum_{i=1}^{n-r} \frac{1}{\lambda_i + \epsilon} \langle v_i, w \rangle^2 + \epsilon^{-1}\sum_{i=n-r+1}^n \langle v_i, w \rangle^2.$$
Hence,\begin{equation} \epsilon^{-1}\norm{\Proj_{\ker\Theta} w}_2^2 \leq w^T \Thetat^{-1} w \leq \epsilon^{-1}\norm{\Proj_{\ker \Theta} w}_2^2 + \frac{1}{\lambda} \norm{w}_2^2.\label{eq:thetat-form-ineq}\end{equation}
Since $\Theta_1,\dots,\Theta_n$, the rows of $\Theta$, are $k$-sparse, we have from the definition of $\alpha = \alpha^{(1)}_{\Thetat^{-1},S,k}$ that $\Theta_i^T \Thetat^{-1}\Theta_i \geq \alpha \norm{S^T\Theta_i}_1^2$ for all $i \in [n]$, so by Equation~\ref{eq:thetat-form-ineq}, $$\norm{S^T \Theta_i}_1^2 \leq \frac{1}{\lambda \alpha} \norm{\Theta_i}_2^2.$$
Therefore summing over $i = 1,\dots,n$,
$$\frac{1}{\sqrt{\lambda \alpha}} \sum_{i=1}^n \norm{\Theta_i}_2 \geq \sum_{i=1}^n \norm{S^T \Theta_i}_1 = \sum_{j=1}^s \norm{\Theta (S^T)_j}_1.$$
The first term in the inequality is at most $\sqrt{n/(\lambda\alpha)}\norm{\Theta}_F$. On the other hand $\norm{\Theta (S^T)_j}_1 \geq \norm{\Theta (S^T)_j}_2 \geq \lambda \norm{(S^T)_j - w_j}_2$ where $w_j = \Proj_{\ker \Theta} (S^T)_j$. Therefore
\begin{equation}\sum_{j=1}^s \norm{(S^T)_j - w_j}_2 \leq \sqrt{\frac{n}{\lambda^3\alpha}} \norm{\Theta}_F.\label{eq:column-dist-to-kernel-0} \end{equation}
We can rewrite this bound in terms of $\delta$ as follows. Pick any $\beta > \beta^{(1)}_{\Thetat^{-1},S,k,m}$. By definition, $\dim W_{\Thetat^{-1},S,\beta} < 2m$, so there is no dimension-$2m$ subspace contained in $W_{\Thetat^{-1},S,\beta}$. In particular, $\dim \ker\Theta > 2m$, so $\ker \Theta \not \subseteq W_{\Thetat^{-1},S,\beta}$. Certainly $0 \in W_{\Thetat^{-1},S,\beta}$, so there is some nonzero $u \in \ker(\Theta)$ for which $u^T \Thetat^{-1} u < \beta \norm{S^T u}_1^2$. Together with Equation~\ref{eq:thetat-form-ineq}, this gives $$\epsilon^{-1} \norm{u}_2^2 \leq u^T \Theta^{-1} u < \beta \norm{S^T u}_1^2 \leq \beta \norm{u}_2^2 \norm{S^T}_{2 \to 1}^2.$$
Hence $1 \leq \sqrt{\epsilon \beta} \norm{S^T}_{2\to 1}$. Since $\beta > \beta^{(1)}_{\Thetat^{-1},S,k,m}$ was arbitrary it follows that $1 \leq \sqrt{\epsilon \beta^{(1)}_{\Thetat^{-1},S,k,m}} \norm{S^T}_{2\to 1}$. Multiplying into Equation~\ref{eq:column-dist-to-kernel-0} gives that 
\begin{equation}\sum_{j=1}^s \norm{(S^T)_j - w_j}_2 \leq \delta \frac{\eta}{3} \norm{S^T}_{2 \to 1}\label{eq:column-dist-to-kernel}\end{equation}
so the columns of $S$ are close to $\ker \Theta$ in an absolute sense. Now define $D \subseteq [s]$ to be the set of columns close to $\ker \Theta$ in a relative sense: $$D = \left\{j \in [s]: \norm{(S^T)_j - w_j}_2 \leq \frac{\eta}{3} \norm{(S^T)_j}_2\right\}.$$
For any $j \in [s]$, its projection onto the kernel is quantitatively dense, i.e. $\dist_{\tau,V}(w_j) \geq \eta \norm{w_j}_2$ by definition of $\eta$. So for any $j \in D$, by the triangle inequality, \begin{align*}
\dist_{\tau,V}((S^T)_j) 
&\geq \dist_{\tau,V}(w_j) - \norm{(S^T)_j - w_j}_2 \\
&\geq \eta \norm{w_j}_2 - \norm{(S^T)_j - w_j}_2 \\
&\geq \eta(\norm{(S^T)_j}_2 - \norm{(S^T)_j - w_j}_2) - \norm{(S^T)_j - w_j}_2 \\
&\geq \frac{\eta}{3} \norm{(S^T)_j}_2
\end{align*}
as desired, where the last inequality uses the assumption that $j \in D$. It remains to bound $\norm{(S^T)_{D^c}}_{2\to 1}$. But by Cauchy-Schwarz, the assumption that $j \in D^c$, and Equation~\ref{eq:column-dist-to-kernel}, $$\norm{(S^T)_{D^c}}_{2 \to 1} \leq \sum_{j \in D^c} \norm{(S^T)_j}_2 \leq \frac{3}{\eta} \sum_{j \in D^c} \norm{(S^T)_j - w_j}_2 \leq \delta \norm{S^T}_{2 \to 1}$$ as claimed.
\end{proof}

\subsection{Failure of compatible preconditioners}\label{section:dense-failure}

The next step is to convert the above statement about density of columns of $S$ into a sample complexity lower bound for the success of $S$-preconditioned Lasso:

\begin{theorem}\label{theorem:dense-lasso-failure-l1}
Under the conditions of Theorem~\ref{thm:columns-dense-or-small-l1}, let $X_1,\dots, X_m \sim N(0, \Thetat^{-1})$ be independent samples. Suppose $r = \dim \ker \Theta > 2m$, and $k > 3(|V|/\tau)\log(n)$. Suppose that $$\delta := \sqrt{\epsilon \gamma n} \cdot 3 \norm{\Theta}_F \eta^{-1} \lambda^{-3/2} < \frac{1}{10n}.$$ Then there is some $k$-sparse signal such that $S$-preconditioned Lasso fails at exact recovery with probability at least $1 - \frac{4m}{3r} - \exp(-\Omega(m))$, over the randomness of the samples $X_1,\ldots,X_m$.
\end{theorem}

Before rigorously proving this theorem, we give a proof sketch and prove several technical lemmas. In broad strokes, the idea is as follows. It should be possible to choose a sparse vector $w^*$ such that the support of $S^T w^*$ contains all coordinates corresponding to dense columns of $S$. Due to Theorem~\ref{thm:columns-dense-or-small-l1}, every column of $S$ is either dense or very low norm, so $(S^T w^*)_i = 0$ only if $(S^T)_i$ has very low norm. In this case, if we consider perturbing $w^*$ by some sufficiently small vector $d$, then the penalty of $w^*+d$ can be approximated (with error proportional to $\norm{d}_2$) as 
\begin{equation}\label{eq:kkt-approximation}
\norm{S^T (w^* - d)}_1 \approx \norm{S^T w^*}_1 - \langle S^T d, \text{sign}(S^T w^*)\rangle.
\end{equation}

Ultimately, to show that $S$-preconditioned Lasso does not exactly recover $w^*$, it suffices to exhibit $d \in \RR^n$ such that $\norm{S^T(w^*-d)}_1 < \norm{S^T w^*}_1$ and $Xd = 0$. Motivated by \eqref{eq:kkt-approximation}, we define $d = \Proj_{\ker(X)}[S\sign(S^T w^*)]$; this is the optimal choice of $d$, in the sense that it maximizes $\langle S^T d, \sign(S^T w^*)\rangle$ for fixed $\norm{d}_2$, over $\ker(X)$. 

However, showing that $\norm{S^T(w^*-d)}_1 < \norm{S^T w}_1$ requires lower bounding $$\norm{\Proj_{\ker(X)}[S\sign(S^T w^*)]}.$$ That is, we need to prove that $S\sign(S^T w^*)$ does not lie near $\rspan(X)$.

This is where we use the assumption that $\ker(\Theta)$ is high-dimensional (which is a necessary assumption, because constructing PSD matrices $\Theta$ with dense, low-dimensional kernels is trivial even for low-treewidth dependency graphs). Indeed, $\rspan(X)$ is concentrated near $\ker(\Theta)$, so if $m > \dim \ker(\Theta)$ then $\rspan(X)$ may approximately contain $\ker(\Theta)$. As the columns of $S$ are near $\ker(\Theta)$, it may then follow that $S\sign(S^T w^*)$ is near $\ker(\Theta)$ and thus near $\rspan(X)$. Fortunately, if $m \ll \dim \ker(\Theta)$ then this is unlikely to happen: intuitively, $\rspan(X)$ is nearly a random subspace of $\ker(\Theta)$, so it is unlikely to align with \emph{any} fixed direction, and the probability of alignment is $O(m/\dim\ker(\Theta))$. The following lemma formalizes this intuition.


\begin{lemma}\label{lem:no-alignment}
Let $\Theta \in \RR^{n \times n}$ be a PSD matrix with minimum nonzero eigenvalue $\lambda$. Let $\epsilon, m > 0$ and let $\Thetat = \Theta + \epsilon I$. Let $X_1,\dots,X_m \sim N(0, \Thetat^{-1})$. If $\epsilon \leq c\lambda/n$ for a sufficiently small absolute constant $c>0$, and $r := \dim \ker \Theta > 2m$, then for any fixed $v \in \RR^n$, we have $$\Pr_{X_1,\dots,X_m}[v^T (I - P) v \geq (v^T v)/8] \geq 1 - \frac{4m}{3r} - \exp(-\Omega(m)),$$ where $P = X^T (XX^T)^{-1} X$ is the projection map onto $\vspan\{X_1,\dots,X_m\}$, and where $X: m \times n$ is the matrix with rows $X_1,\ldots,X_m$.
\end{lemma}

\begin{proof}
The statement of the lemma is basis-independent (e.g. does not depend on sparsity of $\Theta$ or $v$), so we can assume without loss of generality that $\Theta$ is diagonal. Then $\Thetat^{-1}$ is diagonal, and we can choose a basis ordering such that the first $r = \dim \ker \Theta$ diagonal entries are each $\epsilon^{-1}$. Let $w =v_{[r]}$ be the first $r$ coordinates of $v$. For $i \in [m]$ let $Y_i = (X_i)_{[r]}$ be the first $r$ coordinates of $X_i$. Then $Y_1,\dots,Y_m$ are i.i.d. $N(0, \epsilon^{-1} I_r)$. So if $P_Y = Y^T (YY^T)^{-1} Y$, then $P_Y$ is projection onto an isotropically random dimension-$m$ subspace of $\RR^r$. Hence, $$\EE \norm{P_Y w}_2^2 = \frac{m}{r} \norm{w}_2^2.$$
With probability at least $1 - 4m/(3r)$ we have $\norm{P_Y w}_2^2 \leq \frac{3}{4}\norm{w}_2^2$. So $\norm{w - P_Y w}_2^2 \geq \norm{w}_2^2/4$. Now $\norm{w-P_Yw}_2$ is the distance from $w$ to the subspace $\vspan\{Y_1,\dots,Y_m\}$. For any vector in $\vspan\{X_1,\dots,X_m\}$, its first $r$ coordinates lie in $\vspan\{Y_1,\dots,Y_m\}$, so the distance to $v$ must be at least $\norm{w-P_Yw}_2$. Thus,
\begin{equation} \norm{v - Pv}_2^2 \geq \norm{w - P_Yw}_2^2 \geq \frac{1}{4} \norm{w}_2^2. \label{eq:dist-to-projection-1}\end{equation}
Next, note that $(X^T)_{[r]}$ is a $r \times m$ matrix with i.i.d. $N(0, \epsilon^{-1})$ entries, so $\sigma_\text{min}((X^T)_{[r]}) \geq c\epsilon^{-1/2}\sqrt{r}$ with probability at least $1 - \exp(-\Omega(m))$, for some constant $c>0$ (by Theorem~\ref{thm:rmt}). On the other hand, since the entries of $\Thetat^{-1}_{[r]^c,[r]^c}$ are bounded by $1/\lambda$, we also have $\sigma_\text{max}((X^T)_{[r]^c}) \leq C\sqrt{n/\lambda}$ with probability at least $1 - \exp(-\Omega(n))$, for some constant $C$. This means that for any $u \in \RR^m$, $$\norm{(X^T u)_{[r]^c}}_2 \leq C\sqrt{\frac{n}{\lambda}} \norm{u}_2 \leq \frac{C\sqrt{n\epsilon}}{c\sqrt{\lambda r}} \norm{(X^T u)_{[r]}}_2.$$ By assumption, $\epsilon \leq (c/4C)^2 \lambda r/n$, so that $\norm{(X^T u)_{[r]^c}}_2 \leq \norm{(X^T u)_{[r]}}_2/4$. Now $Pv$ lies in the span of $X_1,\dots,X_m$, so there is some $u \in \RR^m$ with $Pv = X^T u$. This means that $$\norm{(Pv)_{[r]^c}}_2 \leq \frac{1}{4} \norm{Pv}_2 \leq \frac{1}{4} \norm{v}_2.$$
So $$\norm{v-Pv}_2 \geq \norm{(v - Pv)_{[r]^c}}_2 \geq \norm{v_{[r]^c}}_2 - \frac{1}{4} \norm{v}_2.$$
Together with Equation~\ref{eq:dist-to-projection-1}, which states that $\norm{v-Pv}_2 \geq \frac{1}{2} \norm{v_{[r]}}_2$, we get that $\norm{v-Pv}_2^2 \geq \frac{1}{8} \norm{v}_2^2.$
\end{proof}

There is one detail missing from the above proof sketch. Namely, if the preconditioner has $s$ dense columns, and we want a $k$-sparse vector $w^*$ such that $\supp(S^T w^*)$ contains the coordinates of all of these columns, then it may be necessary to take $k \geq \Omega(\log s)$ (and $s$ may be much larger than $n$). To avoid paying this, we instead show how to construct $w^*$ such that $\supp(S^T w^*)$ contains \emph{enough} coordinates that the total norm of the remaining columns of $S$ is a $\poly(n)$ factor smaller than the norm of $S$. This suffices for our purposes, and allows taking $k = O(\log n)$, as the following two lemmas show.

\begin{lemma}\label{lemma:2to1-approximation}
Let $S \in \RR^{n \times s}$ be a matrix. Then $$\frac{1}{n}\sum_{i=1}^n \norm{S_i}_1 \leq \norm{S^T}_{2\to 1} \leq \sum_{i=1}^n \norm{S_i}_1.$$
\end{lemma}

\begin{proof}
On the one hand, for any $u \in \RR^n$, we have by Cauchy-Schwartz that $$\norm{S^T u}_1 = \sum_{j=1}^s |(S^T)_j u| \leq \norm{u}_2 \sum_{j=1}^s \norm{(S^T)_j}_2 \leq \norm{u}_2 \sum_{j=1}^s \norm{(S^T)_j}_1 = \norm{u}_2 \sum_{i=1}^n \norm{S_i}_1.$$
On the other hand, for any $i \in [n]$, we have $\norm{S^T e_i}_1 = \norm{S_i}_1$, so $$\norm{S^T}_{2\to 1} \geq \max_{i \in [n]} \norm{S_i}_1 \geq \frac{1}{n}\sum_{i=1}^n \norm{S_i}_1$$
as claimed.
\end{proof}

\begin{lemma}\label{lemma:2to1-norm-reduction}
Let $A \in \RR^{n\times p}$ be a matrix. Let $V \subseteq [n]$ and $\tau \in \NN$, and suppose the ``density condition" holds that $A_{Vj}$ has at least $\tau$ nonzero entries, for all $j \in [p]$. Let $\mu \in (0,1)$. Then there is a subset $K \subseteq V$ with the following properties:
\begin{itemize}
    \item $|K| \leq (|V|/\tau)\log(n/\mu)$
    \item Define $N(K) := \{j \in [p]: A_{Kj} \neq 0\}$. Then $$\norm{(A^T)_{N(K)^c}}_{2\to 1} \leq \mu \norm{A^T}_{2\to 1}.$$
\end{itemize}
\end{lemma}

\begin{proof}
For any matrix $B$, define $F(B) = \sum_i \norm{B_i}_1$. Extend the definition of $N(K)$ to any subset of $[n]$, i.e. $N(L) = \{j \in [p]: A_{Lj} \neq 0\}$. It suffices to show that there exists a row index $i \in [n]$ such that $F(A_{[n], N(\{i\})^c}) \leq (1 - \tau/|V|)F(A)$. Indeed, suppose that this holds. The matrix $A_{[n],N(\{i\})^c}$ also satisfies the density condition, so we can induct. After $k$ steps, we have a set $K \subseteq [n]$ of size $k$, such that $F(A_{[n], N(K)^c}) \leq (1 - \tau/|V|)^k F(A)$. Then by Lemma~\ref{lemma:2to1-approximation}, $$\norm{(A^T)_{N(K)^c}}_{2\to 1} \leq F(A_{[n], N(K)^c}) \leq e^{-\tau k/|V|} F(A) \leq ne^{-\tau k/|V|} \norm{A^T}_{2\to 1}.$$
Taking $k \geq (|V|/\tau)\log(n/\mu)$ yields the desired inequality. So it remains to prove that there exists a row index $i \in [n]$ such that $F(A_{[n], N(\{i\})^c}) \leq (1 - \tau/|V|)F(A)$. Let $\xi \sim \text{Unif}(V)$ be a uniformly random row index from $V$. Then $\Pr[j \in N(\{\xi\})^c] \leq 1 - \tau/|V|$, by the density condition. Hence, 
\begin{align*}
\EE F(A_{[n],N(\{\xi\})^c})
&= \EE \sum_{i \in [n]} \norm{A_{\{i\},N(\{\xi\})^c}}_1 \\
&= \EE \sum_{j \in N(\{\xi\})^c} \norm{A_{[n],j}}_1 \\
&= \sum_{j \in [p]} \norm{A_{[n],j}}_1 \Pr[j \in N(\{\xi\})^c] \\
&\leq (1 - \tau/|V|)\sum_{j \in [p]} \norm{A_{[n],j}}_1 \\
&= (1 - \tau/|V|) F(A).
\end{align*}
Thus, there must exist some $i \in V$ satisfying $F(A_{[n], N(\{i\})^c}) \leq (1 - \tau/|V|)F(A)$, which completes the proof.
\end{proof}

We also need the following simple lemma about the $2\to 1$ norm.

\begin{lemma}\label{lemma:2to1-subsampling}
Let $A \in \RR^{s \times n}$ be a matrix. Then there is some $i \in [n]$ such that the matrix $A_{[s],\{i\}} : s \times 1$ satisfies $\norm{A_{[s],\{i\}}}_{2 \to 1} \geq \frac{1}{n} \norm{A}_{2 \to 1}$.
\end{lemma}

\begin{proof}
Let $v \in \RR^n$ satisfy $\norm{v}_2 = 1$ and $\norm{Av}_1 = \norm{A}_{2\to 1}$. Then if $A_j$ denotes row $j$ of $A$,
$$\norm{Av}_1 = \sum_{j=1}^s |\langle A_j, v\rangle| \leq \sum_{j=1}^s \norm{A_j}_2 \leq \sum_{j=1}^s \norm{A_j}_1 = \sum_{i=1}^n \norm{(A^T)_i}_1.$$
On the other hand, the $2\to 1$ norm of the $s \times 1$ matrix $A_{[s],\{i\}}$ is precisely $\norm{(A^T)_i}_1$. So the index $i \in [n]$ maximizing $\norm{(A^T)_i}_1$ has the desired property.
\end{proof}

We can now prove Theorem~\ref{theorem:dense-lasso-failure-l1}.

\begin{proof}[\textbf{Proof of Theorem~\ref{theorem:dense-lasso-failure-l1}}]
Let $D \subseteq [s]$ be the set of indices of columns of $S$ guaranteed by Theorem~\ref{thm:columns-dense-or-small-l1}. Every $j \in D$ satisfies $\dist_{\tau,V}((S^T)_j) \geq (\eta/3)\norm{(S^T)_j}_2 > 0$, so $S_{Vj}$ has at least $\tau$ non-zero entries. By Lemma~\ref{lemma:2to1-norm-reduction} applied to $S_{[n],D}$, there is a subset of row indices $K \subseteq V$ of size $K \leq 3(|V|/\tau)\log(n)$ such that $\norm{(S^T)_{D \setminus N(K)}}_{2\to 1} \leq n^{-2}\norm{(S^T)_D}_{2\to 1}$, where $N(K) = \{j \in D: S_{Kj} \neq 0\}$.
Append to $K$ the index of the largest row of $S$, i.e. $\argmax_{i \in [n]} \norm{S_i}_1$; by Lemma~\ref{lemma:2to1-subsampling}, this ensures that $\norm{(S_K)^T}_{2\to 1} \geq \norm{S^T}_{2\to 1}/n$.

By genericity, there is a ($\ell_2$) unit vector $w^* \in \RR^n$ supported on $K$ such that the support $U = \supp(S^T w^*)$ satisfies $U \supseteq N(K)$, and such that $\norm{S^T w^*}_1 \geq \norm{(S_K)^T}_{2\to 1}/3$.

By construction, $w^*$ is $k$-sparse. To show that $S$-preconditioned Lasso fails to recover the signal $w^*$, it suffices
to exhibit $d \in \RR^n$ such that $Xd = 0$ and 
\begin{equation}\label{eqn:l1-shrinks-l1}
\norm{(S^Td)_{U^c}}_1 < \langle S^T d, \sign(S^T w^*) \rangle,
\end{equation}
because this shows that $w^*$ does not satisfy first-order optimality conditions, i.e. for sufficiently small $\epsilon > 0$ the vector $w^* + \epsilon d$ is a solution to the linear system with smaller $\ell_1$ norm than $w^*$.

Define $z = \sign(S^T w^*)$. We'll take $d = (I - P)Sz$ where $P$ is the orthogonal projection matrix onto the row span of $X$, $$P = \Proj_{\rspan(X)} = X^T (XX^T)^{-1} X.$$
Then $I-P$ projects onto $\ker(X)$, so certainly $Xd = 0$. It remains to show \eqref{eqn:l1-shrinks-l1}. Since $U \supseteq N(K)$, we have $U^c \subseteq (D \setminus N(K)) \cup D^c$, so by subadditivity of $\norm{\cdot}_{2\to 1}$, we know that
\begin{align*}
\norm{(S^T d)_{U^c}}_1 
&\leq \norm{(S^T)_{U^c}}_{2\to 1} \norm{d}_2 \\
&\leq (\norm{(S^T)_{D\setminus N(K)}}_{2\to 1} + \norm{(S^T)_{D^c}}_{2\to 1})\norm{d}_2 \\
&\leq (n^{-2}+\delta) \norm{S^T}_{2\to 1} \norm{d}_2
\end{align*}
by the guarantees of Lemma~\ref{lemma:2to1-norm-reduction} and Theorem~\ref{thm:columns-dense-or-small-l1}. On the other hand,
$$\langle S^T d, z \rangle = z^T S^T (I-P)Sz = \norm{d}_2^2.$$
So to prove \eqref{eqn:l1-shrinks-l1}, it suffices to show that $\norm{d}_2 > (n^{-2}+\delta) \norm{S^T}_{2\to 1}$. By Lemma~\ref{lem:no-alignment} applied to the vector $Sz$ (which doesn't depend on $X$), we have $d^T d = z^T S^T (I-P)Sz \geq (1/8)\norm{Sz}_2^2$ with probability at least $1 - 4m/(3r) - \exp(-\Omega(m))$. But by Cauchy-Schwarz, $$\norm{w^*}_2 \norm{Sz}_2 \geq (w^*)^T Sz = \langle S^T w^*, \sign(S^T w^*) \rangle = \norm{S^T w^*}_1.$$
Hence, $$\norm{d}_2 \geq \frac{1}{2\sqrt{2}} \norm{Sz}_2 \geq \frac{1}{2\sqrt{2}} \frac{\norm{S^T w^*}_1}{\norm{w^*}_2} \geq \frac{1}{6\sqrt{2}} \norm{(S_K)^T}_{2\to 1} \geq \frac{1}{6n\sqrt{2}} \norm{S^T}_{2\to 1}.$$
Since we assumed that $\delta < 1/(10n)$, it therefore certainly holds that $\norm{d}_2 > (n^{-2}+\delta)\norm{S^T}_{2\to 1}$ for sufficiently large $n$. Therefore Equation~\eqref{eqn:l1-shrinks-l1} holds, and so $S$-preconditioned Lasso fails to perform exact recovery of $w^*$.
\end{proof}

\subsection{The lower bound framework}

We can now put together the results of the last two sections into the following theorem, which states conditions on a precision matrix under which we can prove a sample complexity lower bound against preconditioned Lasso, regardless of the preconditioner. Notably, the sample complexity is determined by the dimension of $\ker(\Theta)$, and the signal sparsity is determined by the density of the vectors in the kernel.

\begin{theorem}\label{theorem:lower-bound}
Let $\Theta \in \RR^{n \times n}$ be a PSD matrix. Let $k,m,s > 0$. Let $\tau > 0$ and $V \subseteq [n]$ and define $$\eta := \inf_{x \in \ker(\Theta) \setminus \{0\}} \frac{\dist_{\tau,V}(x)}{\norm{x}_2}.$$
Also, let $\lambda$ be the smallest non-zero eigenvalue of $\Theta$. Suppose that the following hold:
\begin{itemize}
    \item The rows (and columns) of $\Theta$ are $k$-sparse
    \item $r := \dim \ker(\Theta) > 2m$
    \item $k > 3(|V|/\tau)\log(n)$
\end{itemize}

Pick any positive $$\epsilon < \frac{\eta^2 \lambda^3}{16200n^3 \norm{\Theta}_F^2}.$$ Define $\tilde{\Theta} = \Theta + \epsilon I$. For any preconditioner $S \in \RR^{n \times s}$, there is some $k$-sparse signal such that $S$-preconditioned Lasso fails at exact recovery with probability at least $1 - \frac{4m}{3r} - \exp(-\Omega(m))$, over the randomness of independent covariates $X_1,\dots,X_m \sim N(0,\Thetat^{-1})$ and with noiseless responses $Y_i = \langle w^*, X_i \rangle$.
\end{theorem}

\begin{proof}
Fix $S \in \RR^{n \times s}$ and let $\gamma = \gamma_{\Thetat^{-1},s,m,k}^{(1)}$. If $\gamma > 18$, then by Theorem~\ref{theorem:ill-conditioned-lasso-failure-l1}, there is a $k$-sparse signal $w^* \in \RR^n$ such that the $S$-preconditioned Lasso exactly recovers $w^*$ from $m$ samples with probability at most $\exp(-\Omega(m))$.

On the other hand, if $\gamma \leq 18$, then by choice of $\epsilon$, the quantity $\delta := \sqrt{\epsilon \gamma n} \cdot 3\norm{\Theta}_F \eta^{-1}\lambda^{-3/2}$ satisfies $\delta < 1/(10n)$. Thus, Theorem~\ref{theorem:dense-lasso-failure-l1} implies that there is some $k$-sparse signal $w^* \in \RR^n$ such that the $S$-preconditioned Lasso fails to exactly recover $w^*$ with probability at least $1 - \frac{4m}{3r} - \exp(-\Omega(m))$.
\end{proof}

\section{The expander graph construction}\label{section:expander}

In this section, we instantiate Theorem~\ref{theorem:lower-bound}, the lower bound proved in the last two sections, with a precision matrix supported on an \emph{expander graph}. While this will not help us prove the main lower bound result for high-treewidth graphs, it is fairly simple compared to the full proof, and moreover achieves a stronger sample complexity lower bound. We'll prove the following existence theorem, and at the end of the section we'll apply it to Theorem~\ref{theorem:lower-bound}.

\begin{theorem}\label{theorem:expander}
Let $n \in \NN$. There is a graph $G$ with maximum degree $O(\log^2 n)$, a density parameter $k = \Omega(n/\log n)$, and a positive semi-definite matrix $\Theta$ supported on $G$, with the following properties:
\begin{itemize}
    \item $\dim \ker(\Theta) \geq n/2$
    \item For any $x \in \ker(\Theta)$ and any $k$-sparse $y \in \RR^n$, it holds that $$\norm{x-y}_2 \geq \frac{4}{5\sqrt{n}}\norm{x}_2.$$
    \item $\norm{\Theta}_F \leq O(n\log n)$
    \item The smallest nonzero eigenvalue $\lambda$ of $\Theta$ satisfies $\lambda \geq \Omega(n^{-2-\delta})$ for any constant $\delta>0$.
\end{itemize}
\end{theorem}

The proof idea is as follows. We define $\Theta = M^T M$, where $M \in \RR^{n/2 \times n}$ has independent entries $M_{ij} \sim \text{Ber}(p)$ and $p = \Theta((\log n)/n)$, with the exact constant to be determined. The kernel of $\Theta$ is then the solution set of $n/2$ sparse random equations, which intuitively should not contain sparse or nearly-sparse vectors. Indeed, this can be formalized via the theory of \emph{expander graphs}.

By classical arguments, with high probability, $M$ is the adjacency matrix of a sparse, nearly regular unbalanced bipartite expander graph. Sparsity of $\Theta$ follows from row and column sparsity of $M$. Since $M$ has only $n/2$ rows, it's immediate that $\dim \ker(\Theta) \geq n/2$. The fact that no vector in $\ker(\Theta)$ is nearly sparse is due to an uncertainty principle about expander graphs (originally used to prove \emph{success} of compressed sensing techniques \cite{berinde2008sparse}), and the least nonzero eigenvalue bound holds with high probability due to recent work on least singular values of sparse random matrices \cite{basak2021sharp}.

To be more formal, let $n,m,d \in \NN$ with $m \leq n$, and $p = d/m \in (0,1)$ be chosen later. Define a random matrix $M \in \RR^{m \times n}$ with independent entries $M_{ij} \sim \text{Ber}(p)$. This defines a random bipartite graph with left vertex set $[n]$ and right vertex set $[m]$. Based on this interpretation, we make the following standard graph-theoretic definitions.

\begin{definition}
For $S \subseteq [n]$, define $N(S) \subseteq [m]$, the \emph{neighborhood} of $S$, to be $\{y \in [m]: \exists x \in S: M_{yx} = 1\}.$ Conversely, for $T \subseteq [m]$ define $N'(T) \subseteq [n]$ to be the neighborhood of $T$, i.e. $\{x \in [n]: \exists y \in T: M_{yx} = 1\}$. For sets $S \subseteq [n]$ and $T \subseteq [m]$, define $E(S:T) = \{(i,j): M_{ji} = 1\}$.
\end{definition}

The following result is folklore (see e.g. \cite{alon2004probabilistic}):

\begin{lemma}[Expansion of a random bipartite graph]\label{lemma:expansion}
Let $\epsilon \in (0,1)$. Let $k \leq \epsilon/(2p)$. Suppose that $p \geq 32\epsilon^{-2}(\log n)/m$. It holds with probability at least $1-2/n$ that for all $S \subseteq [n]$ with $|S| \leq k$, $$|N(S)| \geq d(1-\epsilon)|S|.$$
\end{lemma}

\begin{proof}
For $1 \leq l \leq k$ let $q_l$ be the probability that there exists some $S \subseteq [n]$ with $|S| = l$ and $|N(S)| < d(1-\epsilon)|S|$. To bound this probability, fix $S \subseteq [n]$ with $|S| = l$. For any $y \in [m]$, we have $$\Pr[y \in N(S)] = 1 - (1 - p)^l \geq 1 - e^{-pl} \geq pl - (pl)^2 \geq (1-\epsilon/2)pl$$ so long as $0 \leq pl \leq \epsilon/2$. Thus, by the Chernoff bound, $$\Pr[|N(S)| < (1-\epsilon)plm] \leq \exp(-(\epsilon/2)^2(1-\epsilon/2)plm/2) \leq \exp(-\epsilon^2plm/16).$$
By the union bound, if $\epsilon^2pm/16 \geq 2\log n$, we have that $$q_l \leq \binom{n}{l} \exp(-\epsilon^2 plm/16) \leq \exp(l\log n - \epsilon^2 plm/16) \leq \exp(-l\log n).$$
Finally, by a union bound over $1 \leq l \leq k$, the lemma holds with probability at least $$1 - \sum_{l=1}^k q_l \geq 1 - \sum_{l=1}^k n^{-l} \geq 1 - \frac{2}{n}$$ as claimed.
\end{proof}

We'll also need the following simple result:

\begin{lemma}[Degree bounds]\label{lemma:left-degree}
Let $\epsilon \in (0,1)$. Suppose that $p \geq 6\epsilon^{-2}(\log n)/m$. It holds with probability at least $1 - 1/n$ that $$|N(x)| \leq d(1+\epsilon)$$ for all $x \in [n]$. Similarly, it holds with probability at least $1 - 1/n$ that $|N'(y)| \leq (n/m)d(1+\epsilon)$ for all $y \in [m]$.
\end{lemma}

\begin{proof}
Fix $x \in [n]$. By the Chernoff bound, $$\Pr[|N(x)| > (1+\epsilon)pm] \leq \exp(-\epsilon^2 pm/3).$$
Since $\epsilon^2 pm/3 \geq 2\log n$, this bound is at most $1/n^2$. Union bounding over $x \in [n]$ completes the proof of the first claim.

Similarly, fix $y \in [m]$. By the Chernoff bound, $$\Pr[|\{x: y \in N(x)\}| > (1+\epsilon)pn] \leq \exp(-\epsilon^2 pn/3).$$
Since $\epsilon^2 pn/3 \geq 2\log n$, this bound is at most $1/n^2$, and union bounding over $y \in [m]$ completes the proof.
\end{proof}

And we need the following result from random matrix theory:

\begin{lemma}[Theorem~1.1 of \cite{basak2021sharp}]\label{lemma:basak}
There are constants $c,C>0$ with the following property. Suppose that $1/2 \geq p \geq (\log m)/m$. Then $$\Pr[\sigma_\text{min}(M^T) \leq c m^{-1 - \frac{C}{\log \log m}}] \leq \frac{1}{2}$$ for sufficiently large $m$. 
\end{lemma}

\begin{proof}
We apply Theorem~1.1 from \cite{basak2021sharp} to the square submatrix $(M^T)_{[m]}$ and simply note that $\sigma_\text{min}((M^T)_{[m]}) \leq \sigma_\text{min}(M^T)$.
\end{proof}

Suppose that the conclusions of the above lemmas hold. Then $M$ is the adjacency matrix of a nearly regular bipartite expander graph. The next two lemmas are due to \cite{berinde2008sparse} and \cite{berinde2008combining}, slightly generalized to accommodate that our expander graph has left-degrees that are not exactly $d$ but rather $(1\pm \epsilon)d$. The proofs are essentially unchanged but we include them for completeness.

\begin{lemma}[Theorem~10 of \cite{berinde2008combining}]\label{lemma:rip-1}
Let $x \in \RR^n$ be $k$-sparse. Then $$\norm{Mx}_1 \geq d(1-5\epsilon)\norm{x}_1.$$
\end{lemma}

\begin{proof}
Without loss of generality, assume that $|x_1| \geq \dots \geq |x_n|$. Let $E = \{(i,j) \in [n] \times [m]: M_{ji} = 1 \land \exists i'<i: M_{ji'} = 1\}$. For any $j \in [m]$, if $M_j$ is not identically zero, and if $i^*$ is the minimal index $i \in [n]$ such that $M_{ji} = 1$, we have that $$|M_j x| \geq |x_{i^*}| - \sum_{i>i^*} |M_{ji}x_i| \geq \sum_{i \in [n]} |M_{ji}x_i| - 2\sum_{i>i^*} |M_{ji}x_i|.$$
As a result, summing over $j$, we have that $$\norm{Mx}_1 \geq \sum_{j,i} |M_{ji}x_i| - 2\sum_{(i,j) \in E} |x_i|.$$
By the expansion property, every column of $M$ contains at least $d(1-\epsilon)$ ones. So in fact $$\norm{Mx}_1 \geq d(1-\epsilon)\norm{x}_1 - 2\sum_{(i,j) \in E} |x_i|.$$
So it remains to bound the last term. For any $i > k$, we have $|x_i| = 0$. For any $i \leq k$, the number of nonzero rows in $M_{[m],\{1,\dots,i\}}$ is at least $(1-\epsilon)di$ by the expansion property, whereas the number of ones in $M_{[m],\{1,\dots,i\}}$ is at most $(1+\epsilon)di$. So $|E \cap \{1,\dots,i\}\times [m]| \leq 2\epsilon di$. Subject to this constraint, the way to choose a set $E \subseteq [n]\times [m]$ to maximize $\sum_{(i,j) \in E} |x_i|$ is if there are $2\epsilon d$ elements of $E$ in each column $\{i\}\times [m]$, since $|x_i|$ is non-increasing in $i$. Thus, $$\sum_{(i,j) \in E} |x_i| \leq 2\epsilon d \norm{x}_1.$$
The lemma follows.
\end{proof}

\begin{lemma}[Lemma~1 of \cite{berinde2008sparse}]\label{lemma:l1-uncertainty}
Let $x \in \RR^n$ be such that $Mx = 0$. Let $S \subseteq [n]$ with $|S| \leq k$. Then $$\norm{x_S}_1 \leq \frac{4\epsilon}{1-5\epsilon}\norm{x}_1.$$
\end{lemma}

\begin{proof}
Without loss of generality, assume that $S$ contains the largest $k$ coordinates of $x$ in magnitude. Define $S_0 = S$ and sets $S_1,\dots,S_t$ such that the coordinates of $x$ in $S_i$ are no larger in magnitude than the coordinates of $x$ in $S_{i-1}$, for all $i>0$, and $|S_i| = k$ for all $i<t$. Observe that $\norm{M_{N(S)}x_S}_1 = \norm{Mx_S}_1 \geq d(1-5\epsilon)\norm{x_S}_1$ by Lemma~\ref{lemma:rip-1}. So
\begin{align*}
\norm{M_{N(S)}x}_1
&\geq \norm{M_{N(S)}x_S}_1 - \sum_{l \geq 1} \norm{Mx_{S_l}}_1 \\
&\geq d(1-5\epsilon)\norm{x_S}_1 - \sum_{l \geq 1} \sum_{i \in S_l \land j \in N(S) \land M_{ij}=1} |x_i| \\
&\geq d(1-5\epsilon)\norm{x_S}_1 - \sum_{l\geq 1} |E(S_l:N(S))| \cdot \min_{i \in S_l} |x_i| \\
&\geq d(1-5\epsilon)\norm{x_S}_1 - \frac{1}{k}\sum_{l\geq 1} |E(S_l:N(S))| \cdot \norm{x_{S_{l-1}}}_1
\end{align*}
where the last inequality is because every coordinate of $x$ in $S_l$ is no larger than every coordinate of $x$ in $S_{l-1}$. But now for any $l\geq 1$, we have $|N(S\cup S_l)| \geq d(1-\epsilon)|S\cup S_l|$ but there are at most $d(1+\epsilon)|S\cup S_l|$ edges out of $S\cup S_l$. Thus $|E(S_l:N(S))| \leq 4\epsilon d k$. As a result, $$\norm{M_{N(S)}x}_1 \geq d(1-5\epsilon)\norm{x_S}_1 - \sum_{l \geq 1} 4\epsilon d \norm{x_{S_{l-1}}}_1.$$
Since $M_{N(S)}x = 0$, we conclude that $$\norm{x_S}_1 \leq \frac{4\epsilon}{1-5\epsilon} \sum_{l\geq 1} \norm{x_{S_{l-1}}}_1 \leq \frac{4\epsilon}{1-5\epsilon}\norm{x}_1$$ as claimed.
\end{proof}

We can now prove Theorem~\ref{theorem:expander}.

\begin{proof}[\textbf{Proof of Theorem~\ref{theorem:expander}}]
Take $m = n/2$, $\epsilon = 1/10$, and $p = 3200(\log n)/m$. Let $k = \epsilon/(2p) = m/(64000\log n)$. Let $M$ be such that Lemmas~\ref{lemma:expansion}, \ref{lemma:left-degree}, and \ref{lemma:basak} are satisfied. Define $\Theta = M^T M$.

Sparsity of the graph $G$ on which $\Theta$ is supported follows from Lemma~\ref{lemma:left-degree}. Next, we must prove the four properties of $\Theta$. The first property follows since $M$ has only $n/2$ rows. To see the second property, let $x \in \ker(\Theta)$. Then $Mx = 0$, so by Lemma~\ref{lemma:l1-uncertainty} we get that for any size-$k$ subset $S \subseteq [n]$, $$\norm{x-x_S}_2 = \norm{x_{S^c}}_2 \geq \frac{1}{\sqrt{n}}\norm{x_{S^c}}_1 \geq \frac{4}{5\sqrt{n}}\norm{x}_1 \geq \frac{4}{5\sqrt{n}}\norm{x}_2.$$
The third property is because $\norm{\Theta}_F \leq \norm{M}_F^2$ and by Lemma~\ref{lemma:left-degree}, every column of $M$ has at most $O(\log n)$ nonzero entries.

For the last property, note that $\lambda = \sigma^2$, where $\sigma$ is the smallest nonzero singular value of $M$. But this is precisely the smallest singular value of $M^T$, which is lower bounded in Lemma~\ref{lemma:basak}.
\end{proof}

The following lower bound is a direct consequence of Theorem~\ref{theorem:lower-bound} and Theorem~\ref{theorem:expander}. Note that this result proves that $S$-preconditioned Lasso requires a \emph{linear} number of samples to succeed with high probability.

\begin{corollary}\label{corollary:expander-lower-bound}
Let $n \in \NN$. Then there is some $k = O(\log^2 n)$ and some positive-definite matrix $\Thetat$ with $O(\log^2 n)$-sparse rows and columns, and with $\cond(\Thetat) \leq \poly(n)$, with the following property: for any preconditioner $S \in \RR^{n \times s}$, there is a $k$-sparse signal $w^*$ such that $S$-preconditioned Lasso with $m$ samples fails at exact recovery with probability at least $1 - O(m/n) - \exp(-\Omega(m))$, over the randomness of covariates $X_1,\dots,X_m \sim N(0,\Thetat^{-1})$ and with noiseless responses $Y_i = \langle w^*, X_i \rangle$.
\end{corollary}

\begin{proof}
Let $\Theta \in \RR^{n \times n}$ be the positive semi-definite matrix guaranteed by Theorem~\ref{theorem:expander}. Let $d = O(\log^2 n)$ be the maximum degree of the graph on which $\Theta$ is supported, and let $f$ be such that $\norm{x-y}_2 \geq (4/5\sqrt{n})\norm{x}_2$ for every $x \in \ker(\Theta)$ and $f$-sparse $y \in \RR^n$. Theorem~\ref{theorem:expander} guarantees that $f = \Omega(n/\log n)$. Now define $k = \max(d, 1 + 3(n/f)\log(n))$; it's clear that $k = O(\log^2 n)$.

Let $m < n/4$, so that $\dim \ker(\Theta) > 2m$. We know that the smallest non-zero eigenvalue $\lambda$ of $\Theta$ satisfies $\lambda = \Omega(n^{-2.1})$, and the Frobenius norm satisfies $\norm{\Theta}_F \leq O(n\log n)$. Moreover, $$\eta := \inf_{x \in \ker(\Theta)\setminus \{0\}} \frac{\dist_{f, [n]}(x)}{\norm{x}_2} \geq \frac{4}{5\sqrt{n}}.$$
Hence, there is some $\epsilon = \Omega(n^{-12.3}\log^{-2}(n))$ such that $$\epsilon < \frac{\eta^2 \lambda^3}{16200n^3\norm{\Theta}_F^2}.$$
Applying Theorem~\ref{theorem:lower-bound} to $\Thetat = \Theta + \epsilon I$ yields the desired result.
\end{proof}

\section{The grid graph construction}\label{section:grid-construction}

In the previous two sections, we showed that if $\Theta$ is a sparse PSD matrix with a high-dimensional kernel satisfying a robust density property, then preconditioned Lasso necessarily fails on a perturbation $\Theta + \epsilon I$ (for suitably small $\epsilon$, depending on the condition number of $\Theta$), unless the sample complexity is also high.

In this section, we construct a matrix supported on (a variant of) the grid graph and satisfying those conditions. Specifically, we define the following variant of the grid graph:

\begin{definition}
Let $N \in \NN$. The $N \times N$ \emph{up/right-simplicized} grid graph is the result of adding an up/right edge to every cell of the $N \times N$ grid graph (see Figure~\ref{fig:simplicized-grid} for an example).
\end{definition}

We'll prove the following theorem:

\begin{corollary}\label{cor:grid-dense-kernel}
Let $N \in \NN$ be sufficiently large, and let $n = N^2$. Let $G$ be the $N \times N$ up/right-simplicized grid graph on vertex set $[n]$. Then there is a $G$-sparse PSD matrix $\Theta$ and a subset $U \subseteq [n]$ with the following properties:
\begin{itemize}
\item $\dim \ker(\Theta) = \Omega(n^{1/4})$
\item For any $x \in \ker(\Theta)$, we have $$\dist_{|U|/100, U}(x) \geq \Omega(n^{-1/2}) \norm{x}_2$$
\item $\norm{\Theta}_F \leq O(n^{1/2})$
\item The least nonzero eigenvalue of $\Theta$ satisfies $\lambda \geq \Omega(1/n^2)$
\end{itemize}
\end{corollary}

At the end of the section, we then put together this result with the results of the previous sections to prove a concrete lower bound.

\paragraph{Proof overview.} In broad strokes, we define a subset $\mathcal{X}$ of the first row of the $N \times N$ simplicized grid $G$, and a subset $\mathcal{Y}$ of the last row, where $p := |\mathcal{X}| = |\mathcal{Y}| = O(\sqrt{N})$. The goal is that $\ker(\Theta)$ bijects with $\RR^\mathcal{X}$ (in the natural way), and for every vector $v \in \ker(\Theta)$, either $v_\mathcal{X}$ or $v_\mathcal{Y}$ is robustly dense. The construction proceeds by defining a constraint matrix $M \in \RR^{n - p \times n}$ with the desired kernel, and then defining $\Theta = M^T M$.

The kernel of $\Theta$ then consists of vectors $v \in \RR^{N \times N}$ satisfying the system of equations defined by $M$. Specifically, each row of $M$ defines an equation on the vertices of $G$. There will be one equation for every vertex of $G$, except for the vertices of $\mathcal{X}$. Identify the vertex set of $G$ with $\{0,\dots,N-1\}^2$, so that vertex $(0,j)$ lies on the top row of the grid. For the rest of this proof, we will let $v$ be a generic vector of $\ker(\Theta) = \ker(M)$, and we think of the entries $v(i,j)$ as variables which we will linearly constrain. To ensure that $\Theta$ is supported on the simplicized grid, the equation for vertex $(i,j)$ will have one of the following forms:
\begin{enumerate}[label = (\roman*)]
\item $v(i, j) = av(i-1, j) + bv(i,j-1)$
\item $v(i,j) = av(i-1,j) + bv(i-1,j+1)$
\item $v(i,j) = av(i-1,j+1) + bv(i,j+1)$
\item $v(i,j) = av(i+1,j-1) + bv(i+1,j)$
\end{enumerate}
Indeed, suppose $\Theta_{(i,j),(u,v)} \neq 0$. Then columns $(i,j)$ and $(u,v)$ of $M$ have a nonzero inner product, so there is some equation involving both $v(i,j)$ and $v(u,v)$. For each of the above equation types, this means that $(i,j)$ and $(u,v)$ are adjacent in $G$.

How do we use these types of equations to enforce that every solution is dense in either $\mathcal{X}$ or $\mathcal{Y}$? The main idea is to define an arithmetic circuit (with inputs $\mathcal{X}$ and outputs $\mathcal{Y}$) such that $x_\mathcal{Y} = Ax_\mathcal{X}$, for an appropriate matrix $A$. Essentially, we want to implement a complete bipartite graph between $\mathcal{X}$ and $\mathcal{Y}$, where the edge from the $i$th vertex of $X$ to the $j$th vertex of $\mathcal{Y}$ has weight $A_{ji}$. See Figure~\ref{fig:circuit} for a schematic of this implementation. Obviously this does not quite work, since the edges (unavoidably) cross. However, we can replace each crossing in the bipartite graph with a ``swap gadget" that can simulate the crossing (see Figure~\ref{fig:swap}) via the XOR/addition swapping trick $$x := x+y; \qquad y := x-y; \qquad x := x-y.$$

With these techniques, we can implement the constraint $x_\mathcal{Y} = Ax_\mathcal{X}$ for \emph{any} matrix $A$. For our purposes, $A$ will be a Gaussian random matrix; we elaborate on this later in the section. First, we describe the details of the construction (for general $A$) in the following theorem: 


\begin{theorem}\label{theorem:grid}
Let $p \in \NN$ and let $N = 100p^2$ and $n = N^2$. Let $G$ be the $N \times N$ up/right-simplicized grid graph on vertex set $[n]$, and let $A \in \RR^{p \times p}$ be an arbitrary matrix. There are subsets $\mathcal{X},\mathcal{Y} \subseteq [n]$ where $\mathcal{X}$ is contained in the first row of $G$, $\mathcal{Y}$ is contained in the last row, and $|\mathcal{X}| = |\mathcal{Y}| = p$, with the following property: there is a $G$-sparse PSD matrix $\Theta \in \RR^{n \times n}$ such that
\begin{itemize}
\item for any $x \in \ker(\Theta)$, we have $x_\mathcal{Y} = A\,x_\mathcal{X}$ and $\norm{x}_2 \leq O((\sqrt{n} + \norm{A}_F)\norm{x_\mathcal{X}}_2)$,
\item $\norm{\Theta}_F \leq O(\sqrt{n} + \norm{A}_F)$,
\item If $\lambda_1 \leq \dots \leq \lambda_N$ are the eigenvalues of $\Theta$, then $0 = \lambda_1 = \dots = \lambda_p$ and $\lambda_{p+1} \geq \Omega(1/(n(n+\norm{A}_\infty^2)))$.
\end{itemize}
\end{theorem}

\begin{figure}
\centering
\begin{subfigure}[b]{0.3\textwidth}
\centering
\begin{tikzcd}
& & \, \ar[d, dashed] & \\
& & x \ar[d,"1"]& \\
\, \ar[r, dashed] & y \ar[r,"1"] \ar[d,"1"] & x+y \ar[r,"1"] \ar[d,"1"] & y \ar[r, dashed] & \, \\
& y \ar[r,"-1"] & x \ar[d, dashed] \ar[ru, "-1"]& \\
& & \, &
\end{tikzcd}
\caption{Down/Right Gadget}\label{fig:swapdr}
\end{subfigure}
\hspace{0.1\textwidth}%
\begin{subfigure}[b]{0.3\textwidth}
    \centering
    \begin{tikzcd}
    &   &   \, \ar[d, dashed]   & &   \\
    &   &   x  \ar[r,"1"]\ar[ld,"-1"']  & x \ar[ld,"1"] & \\
    \, & y \ar[l,dashed] & x+y \ar[l,"1"] \ar[d,"1"] & y \ar[l,"1"] \ar[ld,"-1"] & \, \ar[l, dashed] \\
    & & x \ar[d,dashed] & & \\
    & & \, & &
    \end{tikzcd}
    \caption{Down/Left Gadget}\label{fig:swapdl}
\end{subfigure}
\caption{The ``swap gadgets" which are used in place of each path crossing. For a gadget centered at vertex $(i,j)$, the goal is equations which enforce that $v(i-1,j) = v(i+1,j)$ and $v(i,j-1) = v(i,j+1)$, and fall into the allowable equation types (i) -- (iv). Gadget (a) is used when a path heading down intersects a path heading right (i.e. $v(i-1,j)$ and $v(i,j-1)$ are defined in terms of previous vertices on their paths), and gadget (b) is used when a path heading down intersects a path heading left (i.e. $v(i-1,j)$ and $v(i,j+1)$ are defined in terms of previous vertices on their paths). In gadget (a), the constraint on vertex $(i,j)$ is that $v(i,j) = v(i-1,j) + v(i,j-1)$; the constraint on vertex $(i,j+1)$ is that $v(i,j+1) = v(i,j) - v(i+1,j)$, and so forth.}\label{fig:swap}
\end{figure}
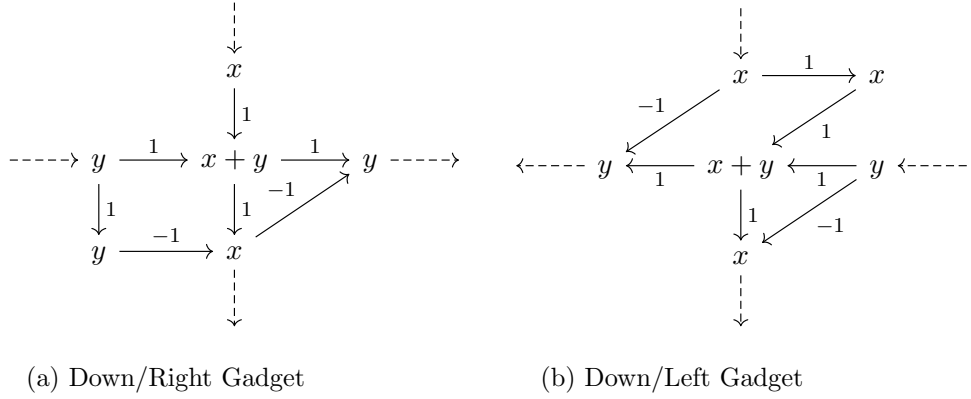

\begin{proof} 
Identify the vertex set of $G$ with $\{0,\dots,N-1\} \times \{0,\dots,N-1\}$ so that $\{0\} \times \{0,\dots,N-1\}$ is the top row. Define $$\mathcal{X} = \{(0, 6pj): 0 \leq j < p\}$$
$$\mathcal{Y} = \{(1+6p^2+6p, 6pj+3): 0 \leq j < p\}.$$

In the second row, we have constraints
$$v(1, 6pj) = v(0, 6pj) \qquad \forall 0 \leq j < p$$
$$v(1, k) = v(1, k-1) \qquad \forall k \not \equiv 0 \bmod{6p}.$$
Thus, we have $6p$ copies of each vertex of $X$. In the next $6p^2$ rows, we want to define constraints so that $v(1, 6pj + 6k) = v(1+6p^2, 6pk + 6j + 3)$. Once this is done, the last $6p$ rows can be used to implement the constraint $$v(1+6p^2+6p, 6pk + 3) = \sum_{j=0}^{p-1} A_{kj} v(1+6p^2, 6pk + 6j + 3)$$ which precisely means that $v(1+6p^2+6p, 6pk + 3) = A_kv_X.$
It remains to show how to implement the $p^2$ constraints $v(1,6pj+6k) = v(1+6p^2,6pk+6j+3)$ using the above equation types.

First, let's ignore the issue of crossings (i.e. interference between different constraints). To every pair $(j,k)$ we assign a row $2+6pj+6k$. We define a path of equal vertices
\begin{align*}
v(1,6pj+6k) 
&= \dots = v(2+6pj+6k, 6pj+6k) \\
&= \dots = v(2+6pj+6k, 6pk+6j+3) \\
&= \dots = v(1+6p^2, 6pk+6j+3).
\end{align*}
See Figure~\ref{fig:circuit} for a depiction of these constraints. Note that these paths are edge-disjoint but there are many vertices which lie at the crossing of two paths, which is a problem. As it stands, there are vertices $(i,j)$ such that we have both of the following constraints:
$$v(i-1,j) = v(i,j) = v(i+1,j)$$
$$v(i,j-1) = v(i,j) = v(i,j+1).$$
This causes interference between the two paths (i.e. the vertices on one path must now be all equal to the vertices on the other path, which is a constraint we do not want), and moreover the constraints no longer define a circuit. However, this is fixable. Every such crossing has a $3 \times 3$ neighborhood which is disjoint from all other $3\times 3$ neighborhoods of crossings. Thus, to deal with the crossings, we replace the equations of the $3 \times 3$ grid around each crossing with a ``swap gadget" (either as in Figure~\ref{fig:swapdr} or Figure~\ref{fig:swapdl}, depending on whether the horizontal path in the crossing is heading right or left).

Formally, to implement gadget (a) at vertex $(i,j)$, we have the following equations:
\begin{align*}
v(i,j) &= v(i-1,j) + v(i,j-1) \\
v(i+1,j-1) &= v(i,j-1) \\
v(i+1,j) &= v(i,j) - v(i+1,j-1) \\
v(i,j+1) &= v(i,j) - v(i+1,j).
\end{align*}

This defines a circuit with inputs $v(i-1,j)$ and $v(i,j-1)$ and outputs $v(i+1,j)$ and $v(i,j+1)$. It simulates the desired constraints that $v(i+1,j) = v(i-1,j)$ and $v(i,j+1) = v(i,j-1)$ while only using the allowable equation types (i)--(iv), and without introducing unwanted constraints. Gadget (b) is analogous, except the inputs are $v(i-1,j)$ and $v(i,j+1)$ and outputs are $v(i+1,j)$ and $v(i,j-1)$. These gadgets complete the implementation of the $p^2$ constraints $v(1,6pj+6k) = v(1+6p^2,6pk+6j+3)$, and therefore the constraint that $v_\mathcal{Y} = Av_\mathcal{X}$.

Finally, for every unconstrained vertex $(i,j)$, introduce a constraint $v(i,j) = 0$. Now, every vertex in $[n] \setminus \mathcal{X}$ is defined by exactly one equation in terms of previous vertices (i.e. there is an ordering of vertices such that this holds). Moreover, there are indeed $n-p$ constraints.

\paragraph{Properties of $\ker(\Theta)$} By construction, if $x \in \ker(\Theta) = \ker(M)$ then it holds that $x_\mathcal{Y} = Ax_\mathcal{X}$. Moreover, $\dim \ker(\Theta) = p$. Finally, let $x \in \ker(\Theta)$. For every $(i,j) \in [n]$, the value $x(i,j)$ is either zero, or equal to some element of $x_X$, or the sum of two elements of $x_\mathcal{X}$, or a partial sum of the inner product $A_kx_\mathcal{X}$ for some $k$. Altogether, $\norm{x}_2^2 \lesssim (N + \norm{A}_F^2)\norm{x_\mathcal{X}}_2^2$.

\paragraph{Bounding the Frobenius norm $\norm{\Theta}_F$.} Every entry of $M$ is in $\{0,-1,1\}$ except for $p^2$ entries which are each entry of $A_{ij}$. As a consequence, every column of $M$ has norm $O(1)$, except for $p^2$ columns with squared norms $O(1) + A_{ij}^2$. Therefore $$\norm{\Theta}_F^2 = \sum_{i,j} \langle M^T_i, M^T_j \rangle^2 \leq \sum_{i,j} \norm{M^T_i}_2^2 \norm{M^T_j}_2^2 \leq O(n + \norm{A}_F^2).$$

\paragraph{Least nonzero singular value.} Let $x \in \RR^n$ and define $\epsilon = \norm{Mx}_2$. By construction, the directed graph defined by $M$ is acyclic, so there is an ordering of the vertices such that the value at any vertex $(i,j)$ is determined by the values of previous vertices in the ordering. We can therefore inductively define a vector $y \in \ker(\Theta)$ satisfying $y_\mathcal{X} = x_\mathcal{X}$. We wish to bound $\norm{y-x}_2$. Let $E(i,j)$ be the signed error of $x$ in the equation at vertex $(i,j)$. Obviously, for any vertex $(i,j)$ defined by $v(i,j) = v(i-1,j)$ we have $$|y(i,j) - x(i,j)| \leq  |y(i-1,j) - x(i-1,j)| + |E(i,j)|,$$
and a similar identity holds if the vertex is defined by any other two-variable equation.
For a down/right swap gadget centered at vertex $(i,j)$, we have $y(i+1,j) = y(i-1,j)$ and $y(i,j+1) = y(i,j-1)$. Moreover
\begin{align*}
x(i+1,j)
&= -x(i+1,j-1) + x(i,j) + E(i+1,j) \\
&= -x(i,j-1) + x(i,j-1) + x(i-1,j) + E(i+1,j) - E(i+1,j-1) + E(i,j)
\end{align*}
so that 
\begin{align*}
|x(i+1,j) - y(i+1,j)| 
&\leq |x(i-1,j) - y(i-1,j)| + |x(i+1,j) - x(i-1,j)| \\
&\leq |x(i-1,j) - y(i-1,j)| + |E(i+1,j)| + |E(i+1,j-1)| + |E(i,j)|
\end{align*}
Similarly,
$$|x(i,j+1) - y(i,j+1)| \leq |x(i,j-1) - y(i,j-1)| + |E(i,j+1)| + |E(i+1,j)| + |E(i+1,j-1)|.$$
Analogous bounds hold for the down/left swap gadget. Finally, if vertex $(i,j)$ is constrained by $v(i,j) = 0$, then $|x(i,j) - y(i,j)| \leq |E(i,j)|.$
Thus, inducting over the vertex ordering and then summing over all vertices $(1+6p^2, 6pk+6j+3)$, we have that $$\sum_{j,k} |y(1+6p^2, 6pk+6j+3) - x(1+6p^2, 6pk+6j+3)| \leq 2\sum_{i,j \in [N]} |E(i,j)| \leq 2\norm{Mx}_1$$
since every error appears in at most $2$ of the paths $v(1,6pj+6k) \to v(1+6p^2, 6pk+6j+3)$. The same bound holds not only for vertex $(1+6p^2, 6pk+6j+3)$ but also for previous vertices in its path, so $$\sum_{i,j \in [N]: i \leq 1+6p^2} |y(i,j) - x(i,j)| \leq 4N\norm{Mx}_1.$$
Now we must handle the last $6p$ rows. For any vertex $(i,j)$ in the last $6p$ rows, for any path from $(i,j)$ to row $1+6p^2$, the product of path weights is either $1$ or $A_{kl}$ for some $k,l$. Therefore $$\sum_{i,j \in [N]: j > 1+6p^2} |y(i,j) - x(i,j)| \leq 6p\norm{Mx}_1 + \norm{A}_\infty \sum_{j,k} |y(1+6p^2,6pk+6j+3) - x(1+6p^2, 6pk+6j+3)|.$$
In total, we conclude that $$\norm{y-x}_1 \leq (4N+6p+2\norm{A}_\infty)\norm{Mx}_1$$
and thus $$\norm{y-x}_2 \leq O(N(N+\norm{A}_\infty)\norm{Mx}_2).$$
The least nonzero singular value bound follows from Lemma~\ref{lemma:least-nonzero-singular}, below.
\end{proof}

\begin{lemma}\label{lemma:least-nonzero-singular}
Let $\Theta = M^T M$ be a positive semi-definite matrix with eigenvalues $$0 = \lambda_1 = \dots = \lambda_k < \lambda_{k+1} \leq \dots \leq \lambda_n.$$
Suppose that for any $x \in \RR^n$ there is some $y \in \ker(\Theta)$ with $\norm{x-y}_2 \leq \alpha \norm{Mx}_2$. Then $\lambda_{k+1} \geq 1/\alpha^2$.
\end{lemma}

\begin{proof}
Note that $$\lambda_{k+1} = \inf_{v \in \rspan(\Theta) \setminus \{0\}} \frac{\norm{Mv}_2^2}{\norm{v}_2^2}.$$
Let $v \in \rspan(\Theta)$. Let $y \in \ker(\Theta)$ be such that $\norm{v-y}_2 \leq \alpha \norm{Mv}_2$. By the Pythagorean Theorem, $\norm{v - y}_2^2 = \norm{v}_2^2 + \norm{y}_2^2$ and so in particular we have $\norm{v-y}_2 \geq \norm{v}_2$, so $\norm{Mv}_2 \geq \alpha^{-1} \norm{v}_2$ as desired.
\end{proof}

To prove Corollary~\ref{cor:grid-dense-kernel} from Theorem~\ref{theorem:grid}, it remains to exhibit a matrix $A \in \RR^{p \times p}$ such that for all $x \in \RR^p$, either $x$ or $Ax$ is robustly dense. The following lemma shows that a Gaussian random matrix $A$ satisfies this version of the uncertainty principle; for any vector $x$, at least one of $x$ or $A x$ must be far from sparse. Furthermore this principle is quite strong as it holds even for a sum of sparsities linear in the dimension $M$; this is not true for $A$ the Fourier transform, for example, due to the existence of the Dirac comb which has sparsity $\sqrt{M}$ in both position and frequency bases, see e.g. the discussion in \cite{candes2006robust}.
\begin{lemma}\label{lemma:random-uncertainty}
Let $A \in \RR^{M \times M}$ have independent $N(0,1/M)$ entries. Then with probability $1 - \exp(-cM)$ it holds that $$\dist_{M/50}\begin{bmatrix} x \\ Ax \end{bmatrix} \geq \frac{1}{33} \norm{x}_2.$$
(Recall from Definition~\ref{def:dist} that $\dist_k(v) = \inf_{w \in B_0(k)} \norm{v-w}_2$.)
\end{lemma}

\begin{proof}
First, $\Pr[\norm{A}_\text{op} > 3] \leq \exp(-cM)$ by Theorem~\ref{thm:rmt}. By a second application of the same theorem, for $k \leq M/50$ and sets $S,T \subseteq [M]$ with $|S| = M-k$ and $|T| = k$, $$\Pr[\sigma_\text{min}(A_{ST}) \leq 1/8] \leq \exp(-0.2M).$$
Recall from standard estimates\footnote{Explicitly, by considering $X \in \{0,1\}^M$ uniformly distributed over indicator vectors of sets with at most $k$ elements and using $\log \sum_{i \le k} {M \choose i} = H(X) \le \sum_i H(X_i) \le M \cdot h(k/M)$ where $H(X) = \EE[-\log \mu(X)]$ denotes the Shannon entropy of random variable $X \sim \mu$, see \cite{cover1999elements}.} that the number of sets $T$ with at most $k$ elements is at most $e^{M \cdot h(k/M)}$ where $h(p) = -p \log p - (1 - p) \log(1 - p)$, and the same estimate upper bounds the number of possible sets $S$. Hence, union bounding over subsets $S,T$ of those sizes, we get $$\Pr[\exists S,T: \sigma_\text{min}(A_{ST}) \leq 1/2] \leq \exp(2Mh(k/M)-0.2M) = \exp(-\Omega(m)).$$
We union bound and condition on no bad events, i.e. no such $S$ or $T$ exists and $\norm{A}_\text{op} \leq 3$. Now let $x \in \RR^M$ be an arbitrary vector with $\|x\|_2 = 1$ and consider two cases:
\begin{enumerate}
    \item There exists a $k$-sparse vector $y$ with $k = M/50$ such that $\norm{x - y}_2 \leq 1/33$. In this case, $\dist_k(A y) \ge (1/8) \norm{y} \geq 4/33$ by the event we conditioned on, and by the triangle inequality
    \[ |\dist_k(A x) - \dist_k(A y)| \le \norm{A x - A y}_2 \le \norm{A}_{\text{op}} \norm{x - y}_2 \le 1/11 \]
    so $\dist_k(A x) \ge 4/33 - 1/11 = 1/33$.
    \item No such $y$ exists. Then by definition $\dist_{k} x \ge 1/33$.
\end{enumerate}
In either case, we have
\[ \dist_k \begin{bmatrix} x \\ A x \end{bmatrix} \geq 1/33\]
as desired.
\end{proof}

Instantiating Theorem~\ref{theorem:grid} with the matrix $A$ guaranteed by the above lemma yields a precision matrix $\Theta$ with all the desired properties:

\begin{proof}[\textbf{Proof of Corollary~\ref{cor:grid-dense-kernel}}]
Let $A \in \RR^{p \times p}$ have independent $N(0,1/p)$ entries. By standard concentration results, it holds that $\norm{A}_F \leq O(p)$ and $\norm{A}_\infty \leq O(p\log p)$ with probability $1 - \exp(-\Omega(p))$. By Lemma~\ref{lemma:random-uncertainty}, it holds that $$\dist_{p/50}\begin{bmatrix} x \\ Ax \end{bmatrix} \geq \frac{1}{33} \norm{x}_2$$
for all $x \in \RR^p$, with probability $1 - \exp(-\Omega(p))$. As a result, a matrix $A$ satisfying all of these properties exists. Apply Theorem~\ref{theorem:grid} to this matrix. Let $U = \mathcal{X} \cup \mathcal{Y}$, where $\mathcal{X},\mathcal{Y} \subseteq [n]$ are as defined in the theorem, and let $\Theta$ be the given PSD matrix. We have that $\dim \ker(\Theta) = p = \Omega(n^{1/4})$. Moreover, for any $x \in \ker(\Theta)$, we have from Theorem~\ref{theorem:grid} that $x_\mathcal{Y} = Ax_\mathcal{X}$ and $\norm{x}_2 \leq O(\sqrt{n})\cdot \norm{x_\mathcal{X}}_2$, so $$\dist_{|U|/100, U}(x) = \dist_{p/50} \begin{bmatrix} x_\mathcal{X} \\ x_\mathcal{Y} \end{bmatrix} \geq \frac{1}{33}\norm{x_\mathcal{X}}_2 \geq \Omega(n^{-1/2} \norm{x}_2).$$
Finally, Theorem~\ref{theorem:grid} also guarantees that $\norm{\Theta}_F \leq O(n^{1/2})$ and that the least nonzero eigenvalue is $\Omega(1/n^2)$.
\end{proof}

We can now show that there exists a positive-definite precision matrix supported on the simplicized grid graph such that $S$-preconditioned Lasso fails on some sparse signal with high probability, for any preconditioner. This is immediate from Corollary~\ref{cor:grid-dense-kernel} and Theorem~\ref{theorem:lower-bound}:

\begin{corollary}\label{corollary:grid-lower-bound}
Let $n \in \NN$ and let $G$ be the $N \times N$ up/right simplicized grid graph with $N = \sqrt{n}$. Then there is some $k = O(\log n)$ and some $G$-sparse positive-definite matrix $\Thetat$ with $\cond(\Thetat) \leq \poly(n)$ with the following property: for any preconditioner $S \in \RR^{n \times s}$, there is a $k$-sparse signal $w^*$ such that $S$-preconditioned Lasso with $m$ samples fails at exact recovery with probability at least $1 - O(m/n^{1/4}) - \exp(-\Omega(m))$, over the randomness of covariates $X_1,\dots,X_m \sim N(0,\Thetat^{-1})$ and with noiseless responses $Y_i = \langle w^*, X_i \rangle$.
\end{corollary}

\begin{proof}
Let $\Theta$ be the $G$-sparse PSD matrix guaranteed by Corollary~\ref{cor:grid-dense-kernel}. Let $k > 300\log(n)$. Since $G$ has maximum degree $6$, we know that $\Theta$ is $k$-sparse. Since $\dim \ker(\Theta) = \Omega(n^{1/4})$, if $\dim\ker(\Theta) \leq 2m$ then the corollary statement is vacuous, so we may assume that $\dim \ker(\Theta) > 2m$. We know that $$\eta = \min_{x \in \ker(\Theta) \setminus \{0\}} \frac{\dist_{|U|/100, U}(x)}{\norm{x}_2} \geq \Omega(n^{-1/2}),$$
and $\norm{\Theta}_F \leq O(n^{1/2})$, and the least nonzero eigenvalue of $\Theta$ satisfies $\lambda \geq \Omega(1/n^2)$. Hence, we can find some $\epsilon = \Omega(n^{-11})$ such that $$\epsilon < \frac{\eta^2\lambda^3}{16200n^3\norm{\Theta}_F^2}.$$
Let $\Thetat = \Theta + \epsilon I$. Note that $\Thetat$ has polynomial condition number, since $\lambda_\text{min}(\Thetat) = \epsilon = \Omega(n^{-11})$ and $\norm{\Thetat}_F \leq O(n^{1/2})$. Appealing to Theorem~\ref{theorem:lower-bound} completes the proof.

\end{proof}

\section{Bootstrapping to high-treewidth graphs}\label{section:unminoring}

In this section, we extend the lower bound proof from the grid graph to any high-treewidth graph. There are three components to this bootstrapping process:

\begin{enumerate}
    \item The Grid Minor Theorem \cite{robertson1986graph, chekuri2016polynomial, chuzhoy2021towards}, which states that any graph with treewidth $t$ contains a polynomial-sized (square) grid minor
    \item A proof that for any graph $G$ and minor $H$, and any precision matrix $\Gamma$ supported on $H$, there is a precision matrix $\Theta$ supported on $G$, and a subset $Y \subseteq V(G)$ such that the Schur complement $\Theta/\Theta_{YY}$ approximates $\Gamma$ to any fixed accuracy
    \item A proof that if preconditioned Lasso succeeds with some probability $p$ on $\Theta$, then there is a preconditioner such that preconditioned Lasso succeeds with probability nearly $p$ on any matrix near $\Theta/\Theta_{YY}$.
\end{enumerate}

The first component, the Grid Minor Theorem, is a celebrated and well-known result. The third component is fairly simple: for any preconditioner $S$ for $\Theta$, we can define $T = S_{[n]\setminus Y}$. Then for any signal $w^* \in \RR^n$, the success probability $$\Pr_{X_1,\dots,X_m \sim N(0,\Theta^{-1})}\left[w^* \in \argmin_{w \in \RR^n: Xw = Xw^*} \norm{S^T w}_1 \right]$$
is at most
$$\Pr_{X'_1,\dots,X'_m \sim N(0,(\Theta/\Theta_{YY})^{-1})}\left[v^* \in \argmin_{v \in \RR^{[n]\setminus Y}: X'v = X'v^*} \norm{T^T v}_1\right]$$
if we define $v^* = (w^*)_{[n]\setminus Y}$, since $X'_1,\dots,X'_m$ are distributed exactly as the restrictions of $X_1,\dots,X_m$ to $[n]\setminus Y$. This means that if $T$-preconditioned Lasso fails on precision matrix $\Theta/\Theta_{YY}$, then $S$-preconditioned Lasso fails on precision matrix $\Theta$. Moreover, if the precision matrix is not exactly $\Theta/\Theta_{YY}$ but near it, then we can bound the discrepancy in success probability by a total variation argument.

The second component is more involved. To sketch the proof, let $G$ be a graph and let $H$ be a minor of $G$. Let $\Gamma$ be a precision matrix supported on $H$. Since $H$ is a minor of $G$, for every vertex $v \in H$ there is a connected component $Z_v$ in $G$, such that for every edge $(u,v)$ of $H$ there is an edge between $Z_u$ and $Z_v$, the corresponding components of $G$. For every component $Z_v$ we can identify a representative $x_v$, and we want to construct a GGM (with precision matrix $\Theta$ supported on $G$) such that the covariance of the representatives approximates $\Gamma^{-1}$. 

The difficulty is that even if there is an edge $(u,v) \in E(H)$, there may be no edge in $G$ between the corresponding representatives $x_u$ and $x_v$, so we cannot create a dependency between $x_u$ and $x_v$ just by choosing $\Theta_{x_ux_v}$ appropriately. Instead, all we know is that some vertex in $Z_u$ is adjacent to some vertex in $Z_v$. To deal with this difficulty, the main idea is to enforce that the variables in any one component $Z_v$ are all highly correlated. This essentially (additively) collapses the dependencies, so that any dependency between two vertices in different components induces approximately the same dependency between the components' representatives.

To be more precise, there are three pieces to the construction of $\Theta$:
\begin{itemize}
    \item For every component $Z_i$, we set $\Theta_{Z_iZ_i}$ to be a large multiple of the Laplacian $L_{G[Z_i]}$ of the induced subgraph $G[Z_i]$. This creates a \emph{Gaussian free field} \cite{sheffield2007gaussian} on each component, making the variables in each component highly positively correlated.
    \item For every edge $(i,j) \in E(H)$, we distribute the dependency $\Gamma_{ij}$ arbitrarily across the edges between $Z_i$ and $Z_j$. This ensures that the covariance between the variables on $x_i$ and $x_j$ is approximately $\Gamma_{ij}$.
    \item To fix the diagonal of the Schur complement (i.e. the variances of the variables on each $x_i$), we update each $\Theta_{x_ix_i}$ appropriately. These updates are completely independent and do not affect the off-diagonal, so it is always possible to fix the variance of $x_i$ to $\Gamma_{ii}$ for all $i$.
\end{itemize}

The following lemma describes this construction more formally and proves correctness. Notably, the ``additivity" property is not exact, and proving that it approximately holds requires expanding the Schur complement as a power series, approximating the first few terms, and truncating the rest.

\begin{lemma}\label{lemma:unminoring}
Let $\epsilon > 0$. Let $G, H$ be graphs with $n = |V(G)|$ and suppose that $H$ is a minor of $G$. Let $\Gamma$ be a positive semi-definite matrix supported on $H$. Then there is a positive semi-definite matrix $\Theta$ supported on $G$ such that $\norm{\Theta}_F \leq n^{21/2}\norm{\Gamma}_F/\epsilon$, and a Schur complement $\tilde{\Gamma}$ of $\Theta$ satisfying $\norm{\tilde{\Gamma} - \Gamma}_F \leq \epsilon\norm{\Gamma}_F$.

Moreover, if $\lambda_\text{min}(\Gamma) \geq 2\epsilon\norm{\Gamma}_F$ then $\Theta$ is positive-definite with $\lambda_\text{min}(\Theta) \geq O(\epsilon \norm{\Gamma}_F/n)$.
\end{lemma}

\begin{proof}
Assume that $V(H) = [h]$. Let $Z_1,\dots,Z_h$ be a partition of $V(G)$ with the following properties:
\begin{itemize}
\item The induced subgraph $G[Z_i]$ is connected, for every $i$
\item For any $(i,j) \in E(H)$ there is some $(x,y) \in E(G)$ with $x \in Z_i$ and $y \in Z_j$.
\end{itemize}
Pick arbitrary representatives $x_i \in Z_i$ for each $i \in [h]$, and let $Y_i = Z_i \setminus \{x_i\}$. Let $Y = \cup_i Y_i$ and let $X = \{x_1,\dots,x_h\}$. Set $t = 2\epsilon^{-1}n^6 \norm{\Gamma}_F$. Define $\Theta$ as follows:
\begin{itemize}
\item $\Theta_{Z_i Z_i} = tL^{(i)} + \delta_{x_ix_i} e_{x_ix_i}$ where $L^{(i)}$ is the Laplacian of $G[Z_i]$, and $\delta_{x_ix_i}$ will be chosen later
\item For any $i \neq j$ with $(i,j) \in E(H)$, define $$\Theta_{Z_i Z_j} = \Gamma_{ij}(e_{xy} + e_{yx})$$ where $(x,y) \in E(G)$ is some arbitrary edge between $Z_i$ and $Z_j$.
\item Otherwise, $\Theta_{Z_i Z_j} = 0$.
\end{itemize}

Now, we choose $\delta_{x_ix_i}$ so that $\Theta_{x_ix_i} - \Theta_{x_iY}\Theta_{YY}^{-1} \Theta_{Yx_i} = \Gamma_{ii}$. By construction, $\Theta$ is both symmetric and supported on $G$.

Let $A = \Theta/\Theta_{YY}$ be the Schur complement obtained by conditioning out $Y$. Define $B = (\Theta_{YY})^{-1}$.
Define the block-diagonal matrix $$L = \begin{bmatrix} L^{(1)} & & & \\ & L^{(2)} & & \\ & & \ddots & \\ & & & L^{(h)}\end{bmatrix}$$
and note that $\lambda_\text{min}(L_{YY}) \geq 1/n$ since each $L^{(i)}$ is the Laplacian of a connected graph, and $L_{YY}$ removes a vertex from each Laplacian. Moreover, define $R = tL_{YY} - \Theta_{YY}$. Observe that $\norm{R}_F \leq \norm{\Gamma}_F$ (since $x_1,\dots,x_h \not \in Y$, every nonzero entry of $R$ is some $\Gamma_{ij}$, and no entry of $\Gamma$ is repeated). 

\paragraph{Estimate $A$ via power series.} Observe that $B = (tL_{YY} - R)^{-1}.$ 
We can rewrite $B$ as $t^{-1}L_{YY}^{-1/2}(I - t^{-1}L_{YY}^{-1/2}RL_{YY}^{-1/2})^{-1}L_{YY}^{-1/2}$. Note that $$\norm{t^{-1}L_{YY}^{-1/2}RL_{YY}^{-1/2}}_\text{op} \leq t^{-1}\norm{R}_\text{op}\lambda_\text{min}(L_{YY})^{-1} \leq t^{-1}\norm{\Gamma}_F n \leq 1/2.$$ This means that $\norm{B}_\text{op} \leq 2t^{-1} \lambda_\text{min}(L_{YY}^{1/2})^2 \leq 2t^{-1}n$, which will be needed later. Moreover it means that we can expand $I - t^{-1}L_{YY}^{-1/2}RL_{YY}^{-1/2}$ as a power series: $$B = \sum_{k=0}^\infty (tL_{YY})^{-1/2} (t^{-1} L_{YY}^{-1/2} R L_{YY}^{-1/2})^k (tL_{YY})^{-1/2}.$$
For any $k$, term $k$ in this series is bounded in operator norm as 
\begin{align*}
\norm{t^{-k-1} L_{YY}^{-1/2}(L_{YY}^{-1/2}R^k L_{YY}^{-1/2})^kL_{YY}^{-1/2}}_\text{op} 
&\leq t^{-(k+1)}\norm{L_{YY}^{-1/2}}_\text{op}^{2(k+1)} \norm{R}_\text{op}^k \\
&\leq t^{-(k+1)} \lambda_\text{min}(L_{YY})^{-(k+1)} \norm{R}_\text{op}^k \\
&\leq t^{-(k+1)} n^{k+1} \norm{\Gamma}_F^k.
\end{align*}
In particular, since $t \geq 2\epsilon^{-1}n^6 \norm{\Gamma}_F$, we can approximate $B$ by the first two terms, with error 
\begin{align*}
\norm{B - t^{-1}L_{YY}^{-1} - t^{-2}L_{YY}^{-1}RL_{YY}^{-1}}_\text{op} 
&\leq \sum_{k=2}^\infty t^{-(k+1)}n^{k+1}\norm{\Gamma}_F^k \\
&\leq t^{-2}n^{-3} \sum_{k=2}^\infty (\epsilon/2)^{k-1} n^{3-6(k-1)} \norm{\Gamma}_F^{-(k-1)} n^{k+1}\norm{\Gamma}_F^k \\
&\leq \frac{1}{2} \epsilon t^{-2}n^{-3}\sum_{k=2}^\infty 2^{2-k}n^{10-5k} \norm{\Gamma}_F \\
&\leq \epsilon t^{-2} n^{-3} \norm{\Gamma}_F.
\end{align*}
We use this to approximate $A = \Theta_{XX} - \Theta_{XY}B\Theta_{YX}$:
\begin{align*}\norm{A - (\Theta_{XX} - \Theta_{XY} (t^{-1}L_{YY}^{-1} + t^{-2}L_{YY}^{-1}RL_{YY}^{-1}) \Theta_{YX})}_F
&\leq \norm{\Theta_{XY}}_F (\epsilon/t^2) \norm{\Theta_{YX}}_F \\
&\leq O(\epsilon \norm{\Gamma}_F).
\end{align*}
since $\norm{\Theta_{XY}}_F^2 \leq \norm{\Gamma}_F^2 + t^2\sum_i \norm{L^{(i)}}_F^2 \leq O(t^2n^3)$.

\paragraph{Show that estimate is near $\Gamma$.} Pick $i,j \in [h]$ with $i\neq j$. We must estimate $$e_{x_i}^T(\Theta_{XX} - \Theta_{XY} (t^{-1}L_{YY}^{-1} + t^{-2}L_{YY}^{-1}RL_{YY}^{-1}) \Theta_{YX})e_{x_j}.$$ 
To do so, we estimate each of the three terms.
\begin{enumerate}
\item The term $e_{x_i}^T \Theta_{XX} e_{x_j}$ is simply $\Theta_{x_ix_j}.$

\item The term $e_{x_i}^T \Theta_{XY} t^{-1}L_{YY}^{-1} \Theta_{YX} e_{x_j}$ can be written
$$\Theta_{x_iY} t^{-1} L_{YY}^{-1} \Theta_{Yx_j} = t^{-1} \sum_{k \in [h]} \Theta_{x_iY_k} (L_{Y_kY_k})^{-1} \Theta_{Y_kx_j}.$$

If $k \not \in \{i,j\}$, then both $\Theta_{x_iY_k}$ and $\Theta_{Y_kx_j}$ avoid the block-diagonal of $\Theta$, so their entries are from $\Gamma$. Thus, such terms have total operator norm at most $t^{-1}\sum_{k\not \in \{i,j\}} n\norm{\Theta_{x_iY_k}}_2\norm{\Theta_{x_jY_k}}_2 \leq O(t^{-1}n\norm{\Gamma}_F^2)$.

If $k \in \{i,j\}$ then the resulting term can be simplified: we have $L_{Y_i Z_i} 1_{Z_i} = 0$, since $L_{Z_iZ_i}$ is a Laplacian, so $L_{Y_i Y_i} 1_{Y_i} = -L_{Y_i x_i} = -t^{-1}\Theta_{Y_ix_i}$. Thus $\Theta_{x_iY_i}L_{Y_iY_i}^{-1} = t1_{Y_i}^T$. Similarly $L_{Y_jY_j}^{-1}\Theta_{Y_jx_j} = t1_{Y_j}$.

Thus, we have
\begin{align*}
\Theta_{x_iY} t^{-1} L_{YY}^{-1} \Theta_{Yx_j} 
&= t^{-1} \sum_{k \in [h]} \Theta_{x_iY_k} (L_{Y_kY_k})^{-1} \Theta_{Y_kx_j} \\
&= t^{-1} \Theta_{x_iY_i}L_{Y_iY_i}^{-1} \Theta_{Y_ix_j} + t^{-1}\Theta_{x_iY_j}L_{Y_jY_j}^{-1} \Theta_{Y_jx_j} + O(t^{-1} n \norm{\Gamma}_F^2) \\
&= - 1_{Y_i}^T \Theta_{Y_ix_j} - \Theta_{x_iY_j} 1_{Y_j} + O(t^{-1} n \norm{\Gamma}_F^2) \\
&= -\sum_{y \in Y_i} \Theta_{yx_j} - \sum_{y \in Y_j} \Theta_{x_iy} + O(t^{-1} n \norm{\Gamma}_F^2).
\end{align*}

\item The term $e_{x_i}^T t^{-2}L_{YY}^{-1}RL_{YY}^{-1} \Theta_{YX} e_{x_j}$ can be written
$$t^{-2}\Theta_{x_iY} L_{YY}^{-1} R L_{YY}^{-1} \Theta_{Yx_j} = t^{-2}\sum_{k,l \in [h]} \Theta_{x_iY_k} L_{Y_kY_k}^{-1} R_{Y_k Y_l} L_{Y_lY_l}^{-1} \Theta_{Y_lx_j}.$$
If $k\neq i$ and $l = j$ then (using the equality $\Theta_{x_jY_j}L_{Y_jY_j}^{-1} = t1_{Y_j}^T$, the corresponding term has operator norm at most $$t^{-2}n \norm{\Theta_{x_iY_k}}_2 \norm{R_{Y_kY_l}}_\text{op}\norm{t1_{Y_j}}_2 \leq t^{-1} n^{3/2} \norm{\Gamma}_F^2.$$
The same bound holds if $k=i$ and $l\neq j$. If $k\neq i$ and $l\neq j$, then the corresponding term has operator norm at most
$$t^{-2}n^2 \norm{\Theta_{x_iY_k}}_2\norm{R_{Y_kY_l}}_\text{op}\norm{\Theta_{x_jY_l}}_2 \leq t^{-2}n^2 \norm{\Gamma}_F^3 \leq t^{-1}n^{-4} \norm{\Gamma}_F^2.$$
Thus, the only significant term is from $k=i$ and $l=j$. Summing this term with the above error terms,
\begin{align*}
t^{-2}\Theta_{x_iY}L_{YY}^{-1} R L_{YY}^{-1} \Theta_{Yx_j} 
&= t^{-2} \Theta_{x_iY_i}L_{Y_iY_i}^{-1} R L_{Y_jY_j}^{-1} \Theta_{Y_jx_j} + O(t^{-1} n^{5/2} \norm{\Gamma}_F^2) \\
&= 1_{Y_i}^T R 1_{Y_j} + O(t^{-1} n^{5/2} \norm{\Gamma}_F^2)\\
&= \sum_{y \in Y_i \land z \in Y_j} R_{yz} + O(t^{-1} n^{5/2} \norm{\Gamma}_F^2) \\
&= -\sum_{y \in Y_i \land z \in Y_j} \Theta_{yz} + O(t^{-1} n^{5/2} \norm{\Gamma}_F^2).
\end{align*}
\end{enumerate}
All together,
\begin{align*}
& e_{x_i}^T(\Theta_{XX} - \Theta_{XY} (t^{-1}L_{YY}^{-1} + t^{-2}L_{YY}^{-1}RL_{YY}^{-1}) \Theta_{YX})e_{x_j} \\
&\qquad = \Theta_{x_ix_j} + \sum_{y \in Y_i} \Theta_{yx_j} + \sum_{y \in Y_j} \Theta_{x_iy} + \sum_{y\in Y_i \land z \in Y_j} R_{yz} + O(t^{-1} n^{5/2} \norm{\Gamma}_F^2) \\
&\qquad = \sum_{y \in Z_i \land z \in Z_j} \Theta_{yz} + O(t^{-1} n^{5/2} \norm{\Gamma}_F^2) \\
&\qquad = \Gamma_{ij} + O(t^{-1}n^{5/2}\norm{\Gamma}_F^2)
\end{align*}
since the sum of edge weights from $Z_i$ to $Z_j$ is precisely $\Gamma_{ij}$ by construction. Hence, in Frobenius norm, we have (since we chose the diagonal of $\Theta_{XX}$ so that $A_{x_ix_i} = \Gamma_{ii}$)
\begin{align*}
\norm{A - \Gamma}_F^2
&= \sum_{i,j \in [h]: i\neq j} (A_{x_ix_j} - \Gamma_{ij})^2 \\
&\leq 2\norm{A - (\Theta_{XX}-\Theta_{XY}(t^{-1}L_{YY}^{-1}+t^{-2}L_{YY}^{-1}RL_{YY}^{-1})\Theta_{YX})}_F^2 \\
&\qquad + 2\sum_{i\neq j} (\Gamma_{ij} - (\Theta_{x_ix_j} - \Theta_{x_iY}(t^{-1}L_{YY}^{-1} + t^{-2}L_{YY}^{-1}RL_{YY}^{-1})\Theta_{Yx_j}))^2 \\
&\leq O(\epsilon^2 \norm{\Gamma}_F^2) + O(t^{-2} n^7\norm{\Gamma}_F^4) \\
&\leq O(\epsilon^2 \norm{\Gamma}_F^2) + O(\epsilon^2 n^{-5} \norm{\Gamma}_F^2).
\end{align*}
Thus, we get that $\norm{A - \Gamma}_F = O(\epsilon \norm{\Gamma}_F)$.

\paragraph{Bounding $\norm{\Theta}_F$.} We have that $$\norm{\Theta - \sum \delta_{x_ix_i}e_{x_ix_i}}_F^2 \leq t^2 \norm{L}_F^2 + \norm{\Gamma}_F^2 \leq t^2 n^3 + \norm{\Gamma}_F^2 \leq 2t^2 n^3.$$
Moreover
$$\norm{\Theta_{XY} \Theta_{YY}^{-1} \Theta_{YX}}_F \leq \norm{\Theta_{XY}}_F^2 \norm{B}_F \leq O(t^2 n^3) \cdot O(t^{-1} n^{3/2}) \leq O(tn^{9/2}).$$
Now since $\delta_{x_ix_i} = \Gamma_{ii} + \Theta_{x_iY}\Theta_{YY}^{-1}\Theta_{Yx_i} - tL_{x_ix_i}$, $$\norm{\sum \delta_{x_ix_i}e_{x_ix_i}}_F \leq \norm{tL}_F + \norm{\Gamma}_F + \norm{\Theta_{XY}\Theta_{YY}^{-1}\Theta_{YX}}_F \leq O(tn^{9/2}).$$ Therefore $\norm{\Theta}_F \leq O(tn^{9/2}) = O(\epsilon^{-1}n^{21/2}\norm{\Gamma}_F)$. 

\paragraph{Smallest singular value.} It's well known that if a principal submatrix and its Schur complement are positive definite, then the whole matrix is positive definite as well. There is in fact a quantitative version of this fact, which is proven below (Lemma~\ref{lemma:schur-min-eigenvalue}). We use it to lower bound $\lambda_\text{min}(\Theta)$. We've shown that $\norm{A-\Gamma}_F \leq \epsilon\norm{\Gamma}_F$ and $\norm{\Theta}_F \leq O(n^{21/2}\norm{\Gamma}_F/\epsilon)$. If $\lambda_\text{min}(\Gamma) \geq 2\epsilon\norm{\Gamma}_F$ then $\lambda_\text{min}(A) \geq \epsilon\norm{\Gamma}_F$, so by Lemma~\ref{lemma:schur-min-eigenvalue} and $\lambda_\text{min}(\Theta_{YY}) \geq t/(2n)$ and $\norm{\Theta_{XY}}_F^2 \leq O(t^2 n^3)$, we have 
\begin{align*}
\lambda_\text{min}(\Theta) 
&\geq \min\left(\frac{\lambda_\text{min}(\Theta_{YY})}{8}, \frac{\lambda_\text{min}(A)}{1 + 4\lambda_\text{min}(\Theta_{YY})^{-2} \norm{\Theta_{YX}}_F^2}\right) \\
&\geq \min\left(\frac{t}{16n}, \frac{\epsilon\norm{\Gamma}_F}{1 + 4(t/(2n))^{-2}\cdot O(t^2 n^3)}\right) \\
&\geq O(\epsilon\norm{\Gamma}_F/n)
\end{align*}
as claimed.
\end{proof}

We now state and prove the lemma used above to bound the smallest singular value of $\Theta$.

\begin{lemma}\label{lemma:schur-min-eigenvalue}
Let $$M = \begin{bmatrix} A & B \\ B^T & C \end{bmatrix}$$ be a positive-definite $n \times n$ matrix. Then $$\lambda_\text{min}(M) \geq \min\left(\frac{\lambda_\text{min}(A)}{8}, \frac{\lambda_\text{min}(M/A)}{1 + 4\lambda_\text{min}(A)^{-2}\norm{B}_\text{op}^2}\right).$$
\end{lemma}

\begin{proof}
Let $X \sqcup Y = [n]$ be the decomposition so that $A = M_{XX}$ and so forth. For any $x \in \RR^n$, we have $$x^TMx = x_X^T A x_X + 2x_Y^T B^T x_X + x_Y^T C x_Y$$ and $$x_Y^T (M/A) x_Y = x_Y^T (C - B^T A^{-1} B) x_Y.$$
Hence
\begin{align*}
x^TMx
&= x_Y^T(M/A)x_Y + x_X^T Ax_X + 2x_Y^T B^T x_X + x_Y^T B^T A^{-1}Bx_Y \\
&= x_Y^T(M/A)x_Y + (x_X + A^{-1}Bx_Y)^T A (x_X + A^{-1} B x_Y) \\
&\geq \lambda_\text{min}(M/A) \norm{x_Y}_2^2 + \lambda_\text{min}(A) \norm{x_X + A^{-1}Bx_Y}_2^2.
\end{align*}
Let $k = \max(1, \norm{A^{-1}B}_\text{op})$. If $\norm{x_Y}_2 \leq \norm{x_X}_2/(2k)$ then $$x^T Mx \geq \lambda_\text{min}(A) \cdot \frac{1}{4} \norm{x_X}_2^2 \geq \frac{\lambda_\text{min}(A)}{8} \norm{x}_2^2.$$
Conversely, if $\norm{x_Y}_2 \geq \norm{x_X}_2/(2k)$, then $$x^T Mx \geq \frac{\lambda_\text{min}(M/A)}{1+4k^2}\norm{x}_2^2 \geq \frac{\lambda_\text{min}(M/A)}{1+4\lambda_\text{min}(A)^{-2} \norm{B}_\text{op}^2} \norm{x}_2^2.$$
The lemma follows.
\end{proof}

We now use Lemma~\ref{lemma:unminoring} to black-box bootstrap a hard example for preconditioned Lasso supported on $H$ to a hard example for preconditioned Lasso supported on $G$, for any graph $G$ containing $H$ as a minor. One technical difficulty is that Lemma~\ref{lemma:unminoring} only shows that a precision matrix $\Gamma$ supported on $H$ can be \emph{approximated} by the Schur complement of a precision matrix $\Theta$ supported on $G$. Thus, if we sample $X_1,\dots,X_m \sim N(0,\Theta^{-1})$ and restrict to the appropriate subset of coordinates, the covariance is not quite $\Gamma^{-1}$. To argue that this does not significantly impact the probability that preconditioned Lasso succeeds, we simply show that the restriction of $X_1,\dots,X_m$ is close in total variation distance to the hard instance $N(0,\Gamma^{-1})$. For this purpose we recall the following lemma:

\begin{lemma}[\cite{devroye2018total}]\label{lemma:normal-tv}
Let $\Sigma_1,\Sigma_2$ be $n \times n$ covariance matrices. Let $X \sim N(0, \Sigma_1)$ and $Y \sim N(0, \Sigma_2)$ with $$\norm{\Sigma_1^{-1/2}\Sigma_2\Sigma_1^{-1/2} - I}_F \leq \epsilon.$$ Then $d_{TV}(X,Y) \leq 1.5 \epsilon$.
\end{lemma}

\begin{theorem}\label{theorem:unminoring-lower-bound}
Let $G, H$ be graphs with $V(G) = [n]$, with $V(H) = [h]$, and with $H$ a minor of $G$. Let $s, m, k \in \NN$. Suppose that there exists an $H$-sparse precision matrix $\Gamma$ such that for every $S \in \RR^{h \times s}$ there is some $k$-sparse $w^* \in \RR^h$ such that $$\Pr_{X_1,\dots,X_m \sim N(0, \Gamma^{-1})}\left[w^* \in \argmin_{Xw = Xw^*} \norm{S^T w}_1\right] \leq p.$$
Let $\epsilon,\delta > 0$. Suppose that $\lambda_\text{min}(\Gamma) \geq \epsilon$ and $\delta \leq 1/(2\norm{\Gamma}_F)$. Then there is a $G$-sparse precision matrix $\Theta$ such that for every $T \in \RR^{n \times s}$ there is some $k$-sparse $v^* \in \RR^n$ such that $$\Pr_{X_1,\dots,X_m \sim N(0,\Theta^{-1})}\left[v^* \in \argmin_{Xv = Xv^*} \norm{T^T v}_1\right] \leq p + 1.5\delta nm.$$
Moreover, $c\epsilon\delta n^{-1} \norm{\Gamma}_F I \preceq \Theta \preceq C\delta^{-1}\epsilon^{-1} n^{21/2}\norm{\Gamma}_F I.$
\end{theorem}

\begin{proof}
Let $\Theta$ be the $G$-sparse matrix obtained by applying Lemma~\ref{lemma:unminoring} to $\Gamma$ with parameter $\delta\epsilon$, and let $Y$ be the subset of vertices of $G$ so that $\norm{\Theta/\Theta_{YY} - \Gamma}_F \leq \delta$. Let $A = \Theta/\Theta_{YY}$. Then $\norm{\Theta}_F \leq Cn^{21/2}\norm{\Gamma}_F/(\delta\epsilon)$, and since $\lambda_\text{min}(\Gamma) \geq \epsilon \geq 2\epsilon\delta\norm{\Gamma}_F$, it follows that $\lambda_\text{min}(\Theta) \geq c\epsilon\delta\norm{\Gamma}_F/n$.

Let $T \in \RR^{n \times s}$. Define $S = T_{[n] \setminus Y} \in \RR^{h \times s}$. Let $w^* \in \RR^h$ be the resulting $k$-sparse vector for which $S$-preconditioned Lasso fails under design covariance $\Gamma^{-1}$. Let $v^* \in \RR^n$ be the $0$-extension of $w^*$. We have $$\Pr_{X_1,\dots,X_m \sim N(0,\Theta^{-1})}\left[v^* \in \argmin_{Xv = Xv^*} \norm{T^T v}_1\right] \leq \Pr_{X_1,\dots,X_m \sim N(0, A^{-1})}\left[w^* \in \argmin_{Xw = Xw^*} \norm{S^T w}_1\right].$$
Note that since $\norm{\Gamma-A}_F \leq \delta\epsilon$, we have $$\norm{I - \Gamma^{-1/2}A\Gamma^{-1/2}}_F \leq \norm{\Gamma^{-1/2}}_F\norm{\Gamma-A}_F\norm{\Gamma^{-1/2}}_F \leq \delta \epsilon n\norm{\Gamma^{-1}}_\text{op} \leq \delta n.$$ By Lemma~\ref{lemma:normal-tv}, we have $d_{TV}(N(0, A^{-1}), N(0, \Gamma^{-1})) \leq 1.5 \delta n$. Hence, if $X_1,\dots,X_m \sim N(0,A^{-1})$ and $Y_1,\dots,Y_m \sim N(0,\Gamma^{-1})$ then $d_{TV}(X,Y) \leq 1.5\delta nm$. As a result, $$\Pr_{X_1,\dots,X_m \sim N(0, A^{-1})}\left[w^* \in \argmin_{Xw = Xw^*} \norm{S^T w}_1\right] \leq p + 1.5\delta nm$$ as desired.
\end{proof}

It remains to note that every high-treewidth graph contains the simplicized grid as a minor. Indeed, it is not hard to see that the up/right simplicized $N \times N$ grid graph is a minor of the $2N \times 2N$ grid graph:

\begin{lemma}\label{lemma:simplicized-minor}
Let $N \in \NN$, and let $H$ be the $N \times N$ up/right simplicized grid graph. Let $G$ be the $2N \times 2N$ grid graph. Then $H$ is a minor of $G$.
\end{lemma}

\begin{proof}
It suffices to exhibit a map $\phi: V(H) \to \mathcal{P}(V(G))$ such that each subset $\phi(h)$ induces a connected subgraph of $G$, and each edge $(h_1,h_2)$ corresponds to an edge $(x,y)$ between some $x \in \phi(h_1)$ and $y \in \phi(h_2)$.

Identify the vertex set of $H$ with $[N] \times [N]$, and the vertex set of $G$ with $[2N] \times [2N]$. Pick a vertex $(i,j) \in V(H)$. We define $$\phi(h) = \begin{cases} \{(2i-1, 2j), (2i-1, 2j+1), (2i, 2j-1), (2i, 2j)\} & \text{ if } j < N \\ \{(2i-1, 2j), (2i, 2j-1), (2i, 2j)\} & \text{ otherwise}.\end{cases}.$$
The vertical edge $(i,j) \leftrightarrow (i-1, j)$ in $H$ corresponds to the edge $(2i-1, 2j) \leftrightarrow (2i-2, 2j)$ in $G$. The diagonal edge $(i,j) \leftrightarrow (i-1, j+1)$ corresponds to $(2i-1, 2j+1) \leftrightarrow (2i-2, 2j+1)$. And the horizontal edge $(i,j) \leftrightarrow (i,j+1)$ corresponds to $(2i-1, 2j+1) \leftrightarrow (2i-1, 2j+2)$. See Figure~\ref{fig:minor} for a depiction of the construction when $N = 3$.
\end{proof}

\begin{figure}
    \centering
    \includegraphics[width=0.3\textwidth,trim={0 0 14cm 23cm}]{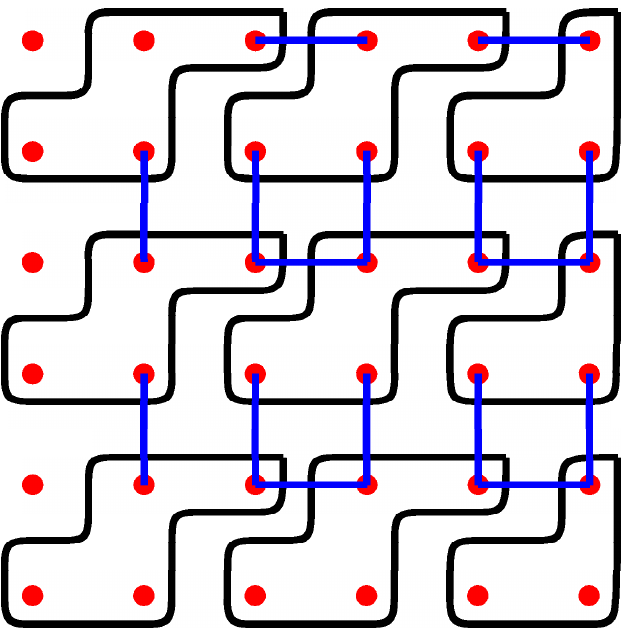}
    \caption{A model of the $3 \times 3$ up/right simplicized grid graph in the $6 \times 6$ grid graph.}
    \label{fig:minor}
\end{figure}

Finally, we appeal to the Grid Minor Theorem, for which the following is the current tightest bound:

\begin{theorem}[Theorem 1.1 of \cite{chuzhoy2021towards}]\label{theorem:grid-minor}
There are constants $c_1,c_2 > 0$ with the following property. Let $G$ be a graph with treewidth $t \geq c_1 g^9 \log^{c_2}(g)$. Then $G$ contains the $g \times g$ grid as a minor.
\end{theorem}

We can now prove Theorem~\ref{theorem:lower-bound-intro}, the lower bound for general high-treewidth graphs.

\begin{theorem}
Pick $n, t \in \NN$, and suppose that $G$ is a graph on $[n]$ with treewidth at least $t$. Then there exists $k = O(\log n)$ and a $G$-sparse precision matrix $\Theta$ with condition number $\poly(n)$ with the following property: for every preconditioner $S \in \RR^{n \times s}$, there is some $k$-sparse signal $w^* \in \RR^n$ such that the $S$-preconditioned basis pursuit exactly recovers $w^*$ with probability at most $O(m/t^{1/19}) + \exp(-\Omega(m))$, from covariates $X_1,\dots,X_m \sim N(0,\Theta^{-1})$ and noiseless responses $Y_i = \langle w^*, X_i \rangle$.
\end{theorem}

\begin{proof}
Let $g$ be the largest integer such that $t \geq c_1(2g)^9 \log^{c_2}(2g)$, with constants as in Theorem~\ref{theorem:grid-minor}. Let $H$ be the $g \times g$ up/right simplicized grid graph. By Lemma~\ref{lemma:simplicized-minor}, $H$ is a minor of the $2g \times 2g$ grid graph, which by Theorem~\ref{theorem:grid-minor} is a minor of $G$. Thus, $H$ is a minor of $G$.

By Corollary~\ref{corollary:grid-lower-bound}, there is an $H$-sparse positive definite matrix $\Gamma$ with $\lambda_\text{min}(\Gamma) \geq \Omega(g^{-22})$ and $\norm{\Gamma}_F \leq O(g)$, such that for any preconditioner $S \in \RR^{g^2 \times s}$, there is $k = O(\log n)$ and some $k$-sparse $w^* \in \RR^{g^2}$ such that $$\Pr_{X_1,\dots,X_m\sim N(0,\Gamma^{-1})}\left[ w^* \in \argmin_{w: Xw = Xw^*} \norm{S^T w}_1\right] \leq \frac{Cm}{g^{1/2}} + \exp(-\Omega(m))$$ for any $m = O(g^{1/2})$.

Let $\delta = c_\delta/(g^{1/2}n^2)$ be such that $\delta \leq 1/(2\norm{\Gamma}_F)$, which holds for some absolute constant $c_\delta>0$. Then by Theorem~\ref{theorem:unminoring-lower-bound}, there is a $G$-sparse precision matrix $\Theta$ such that $\lambda_\text{min}(\Theta) \geq O(g^{-45/2}n^{-3}\norm{\Gamma}_F)$ and $\norm{\Theta}_F \leq O(g^{45/2}n^{25/2}\norm{\Gamma}_F)$. Rescaling by $1/\norm{\Gamma}_F$ provides the desired condition number bound. Moreover, Theorem~\ref{theorem:unminoring-lower-bound} guarantees that for every preconditioner $T \in \RR^{n \times s}$, there is some $k$-sparse $v^* \in \RR^n$ such that $$\Pr_{X_1,\dots,X_m \sim N(0,\Theta^{-1})}\left[v^* \in \argmin_{v:Xv=Xv^*}\norm{T^Tv}_1\right] \leq \frac{Cm}{g^{1/2}} + \exp(-\Omega(m)) + \frac{1.5c_\delta}{g^{1/2}}$$ for any $m = O(g^{1/2})$. Since $g = \Omega(t^{2/19})$, the claim follows.
\end{proof}

\section{Discussion}\label{section:discussion}
Our results give an answer to the question of when preconditioning the Lasso can make sparse linear regression problems tractable.
For future work, it would be interesting to prove lower bounds for sparse linear regression against an even larger class of algorithms, and we expect some of the tools developed in this work may be useful in this direction. One natural candidate framework is the Statistical Query (SQ) model \cite{kearns1998efficient}, which considers algorithms that have a limited form of access to data. As discussed in \cite{vempala2019gradient},
some care must be taken in formulating the precise SQ model for real-valued/regression problems.

Conversely, it would be equally interesting if sparse linear regression is in fact tractable on the random designs we constructed. There is a notable lack of algorithms which succeed outside the regime of preconditioned Lasso, so it seems likely that this would require developing new algorithmic techniques.

\paragraph{Acknowledgements.} We thank Ankur Moitra, Pablo Parrilo, Arsen Vasilyan, Philippe Rigollet, Guy Bresler, Dylan Foster, Tselil Schramm, and Matthew Brennan for valuable conversations on related topics.

\bibliographystyle{amsalpha}
\bibliography{bib}

\appendix
\section{An example with large subgaussian constant}\label{apdx:example}
In this section, we discuss an example application of our theory to a problem with a very large subgaussian constant. This helps clarify some connections between our upper bound in the special case of the simple random walk and some other work in the signal processing literature.

We consider a random-design version of a 1d signal processing problem where 
\[ Z_0 \sim Uni\,\{1,\ldots,n\}, \quad (X_0)_i = \bone(Z_0 \le i) - i/n \]
and 
\[ Y_0 = \langle w^*, X \rangle \]
with $w^*$ $k$-sparse. This is the same as a (nonlinear) regression problem where $Y_0 = f(Z_0)$, we assumed the problem is centered so that $\EE Y_0 = 0$, and the function $f$ is assumed to be piecewise constant with at most $k$ change points.
In this case the population covariance matrix is
\[ \Sigma_{ij} = \EE \bone(Z_0 \le i) \bone(Z_0 \le j) - ij/n^2 = \min\{i,j\}/n - ij/n^2. \]
If $\Theta_1$ is the precision matrix of simple random walk with steps scaled by $1/n$ and $u_i = i/n$  then by the Sherman-morrison formula,
\[ \Theta = \Sigma^{-1} = (\Theta_1^{-1} - uu^T)^{-1} = \Theta_1 + \frac{\Theta_1 u u^T \Theta_1}{1 - u^T \Theta u} \]
and observe $\Theta_1 u = (1/n) e_n$ so the precision matrix is sparse and supported on a path. (It's the covariance matrix of a discretized Brownian bridge.) 

The covariates in this example have a subgaussian constant which is $poly(n)$. Therefore, the straightforward subgaussian generalization of our upper bound result (Theorem~\ref{thm:upper-bound-intro}) says that given $m = \Omega(poly(n))$ the preconditioned Lasso obtains a statistical rate of $O(\sigma^2 k \log^3(n)/m)$ which is close to optimal. Since the underlying graph is a path, the preconditioner from our Theorem essentially rewrites the problem to estimating the signal $f$ in the Haar wavelet basis. 

In the signal processing literature, a variant of this problem with fixed instead of random design has been extensively studied. In this variant, we are given a noisily observed version of the signal at each timestep, instead of at random times; this problem is a \emph{1-d denoising problem} and the fact that Haar wavelets can be used with the Lasso is known: see \cite{needell2013stable,hutter2016optimal}. 
In fact, it's also known in this setting that the \emph{unpreconditioned} Lasso (i.e. regression with a total-variation penalty) actually achieves optimal statistical rates \cite{dalalyan2017prediction}. 

It may appear surprising that the unpreconditioned Lasso is known to perform well in the 1-d denoising example, which seems related to the path example from the Introduction, since for the latter problem which we have an explicit lower bound (Theorem~\ref{thm:random-walk-lb-intro}) showing the suboptimality of the Lasso. A key difference between this setup and the path example is the very large subgaussian constant: in the random design regression example described above, it's clear that exact recovery of the signal $f$ is not information-theoretically possible from $o(n)$ samples, and so the fact that the Lasso fails from $o(n)$ samples does not correspond to a gap in performance versus the information-theoretic optimal estimator.

\section{Oracle Inequality}\label{apdx:oracle}
In this section, we will consider the following a general setup for sparse linear regression which allows for misspecification, i.e. situations where the response only approximately follows a sparse linear prediction rule.
It follows straightforwardly by combining our Theorem~\ref{thm:sst} with the oracle inequality for the Lasso from \cite{bickel2009simultaneous}. One motivation for considering this general setup is its consistency with the relatively gradual improvement of the Preconditioned BP error in Figure~\ref{fig:experiment}, vs. the sharp ``phase transition'' type behavior observed with BP; after preconditioning, the signal ends up spread across multiple coordinates and the oracle inequality reflects a tradeoff where the Lasso competes with an approximation of the signal using fewer coordinates (this corresponds to selecting $k$ smaller than the ground truth sparsity in the bound below).

We first formally describe the usual oracle inequality setup.
The algorithm is given as input a matrix $X : m \times n$ with rows sampled independently as $X_i \sim N(0,\Sigma)$, as well as a noisy \emph{response vector} $Y \in \mathbb{R}^m$ generated as
\[ Y = Y^* + \xi \]
where $Y^* \in \mathbb{R}^m$ is an arbitrary vector, unknown to the algorithm, and $\xi$ is a random vector which, conditional on $X$, is both mean-zero (i.e. $\EE[\xi | X] = 0$) and $(\sigma^2,I)$-subgaussian. 
The goal in this section is to establish upper bounds for recovering $Y^*$, more specifically to show an \emph{oracle inequality} of the form
\[ \|Y^* - X w\|_2^2 \le (1 + \alpha)\min_{w^* \in B_0(k)} \left[\|Y^* - X w^*\|_2^2 + \epsilon \right] \]
where $\alpha > 0$ can be chosen arbitrarily,
 $w$ is the output of the algorithm, and the goal is to make $\epsilon$ (which will depend on $\alpha$) as small as possible. This bound is called an oracle inequality because it shows that our algorithm achieves an error comparable with a sparse \emph{oracle} $w^*$ which has access to the true $Y^*$ (equivalently, an infinite amount of data), up to a constant factor and a small error term $\epsilon = o(1)$. See \cite{rigollet2015high,tsybakov2008introduction} for further background on oracle inequalities. It's also possible to state results for more of a learning-theoretic misspecification model (see e.g. \cite{shalev2014understanding}) where the Bayes predictor $\EE[Y_i | X_i] = f(X)$ is given by an approximately linear function $f$ and the goal is to minimize $\EE[(f(X_0) - \langle w^*, X_0 \rangle)^2]$, an objective which includes the cost of generalization to fresh samples, but for simplicity we stick to the oracle inequality setup described above.

To prove guarantees for the Lasso combined with our preconditioner, we just need to combine our Theorem~\ref{thm:sst} with the oracle inequality for the Lasso proved in \cite{bickel2009simultaneous}.
\begin{theorem}[Theorem 6.1 of \cite{bickel2009simultaneous}]\label{thm:lasso-oracle}
For any $\alpha > 0$ the following result holds. Suppose $X : m \times n$
is an arbitrary matrix with columns of $\ell_2$ norm at most $\sqrt{m}$ and $\hat \Sigma = \frac{1}{m} X^T X$ satisfies $RE(k,3 + 4/\alpha)$. Suppose $Y = Y^* + \xi$
where $Y^* \in \mathbb{R}^m$ and $\xi$ is a random vector which is mean-zero and $(\sigma^2,I)$-subgaussian. Then the Lasso with regularization parameter $\lambda = A\sigma \sqrt{\frac{\log n}{m}}$ outputs $w$ satisfying the oracle inequality
\[ \frac{1}{m}\|Y^* - Xw\|_2^2 \le (1 + \alpha) \min_{\|w^*\|_0 \le k} \left[\frac{1}{m} \|Y^* - X w^*\|_2^2 + \frac{C(\alpha) A^2 \sigma^2}{\kappa^2(s,3 + 4/\alpha)} \frac{\|w^*\|_0 \log(n)}{m}\right]\]
where $C(\alpha) > 0$ depends only on $\alpha$, $\|w^*\|_0$ denotes the sparsity (i.e. number of nonzeros) of $w^*$, and the result holds
with probability at least $1 - n^{1 - A^2/8}$ over the randomness of the noise. 
\end{theorem}
We note that in the original paper \cite{bickel2009simultaneous} it was assumed the noise is Gaussian distributed, but the sub-Gaussian generalization stated above follows by essentially the same proof (see e.g. \cite{rigollet2015high}). We also have used notation consistent with the rest of this paper: e.g. in the language of \cite{bickel2009simultaneous} the columns of $X$ would be referred to as dictionary elements and denoted $f_1,\ldots,f_n$.
\begin{theorem}
Provided $m = \Omega(k \tw(\Theta) \log^2(n)\log(n/\delta))$, the output $\hat{w}$ of the preconditioned Lasso, using the preconditioner from Theorem~\ref{thm:sst}, satisfies
\[ \frac{1}{m}\|Y^* - Xw\|_2^2 \le 2 \min_{w^* : \|w^*\|_0 \le k} \left[\frac{1}{m} \|Y^* - X w^*\|_2^2 + C \frac{\sigma^2 \|w^*\|_0 \tw(\Theta) \log^2(n/\delta)}{m}\right]\]
with probability at least $1 - \delta$.
\end{theorem}
\begin{proof}
This follows by combining Theorem~\ref{thm:sst} and Theorem~\ref{thm:lasso-oracle} as in Theorem~\ref{thm:lasso-recovery}.
\end{proof}

\section{Variant of IHT Analysis}\label{apdx:iht}
In this Appendix, we prove the variant of the IHT guarantee we stated
in Lemma~\ref{lem:iht}. The proof we give is a variant of the IHT analysis
from \cite{blumensath2009iterative}.\\\\
\noindent Algorithm~\textbf{IHT}($\mathcal{S},T,X,Y$):
\begin{enumerate}
    \item Set $w_0 = 0$.
    \item For $t = 1$ to $T$:
    \begin{enumerate}
        \item Set $u_{t} = w_{t - 1} + \frac{1}{m} X^T(Y - X w_{t - 1})$.
        \item Set $w_{t} = \Proj_{w : \supp(w) \in \mathcal{S}}[u_t]$.
    \end{enumerate}
    \item Return $w_T$.
\end{enumerate}
\begin{lemma}\label{lem:projection-fact}
Suppose $\mathcal{S}$ is a family of subsets of $[n]$ and let $\Proj_{\mathcal S}(v) = \arg\min_{\supp(v) \subset S \in \mathcal{S}} v$ be the corresponding projection map onto $\mathcal{S}$-sparse vectors (where among minimizers, the result is chosen by a fixed but arbitrary rule). Then for any $v \in \mathbb{R}^n$, let $w = \Proj_{\mathcal S}(u)$ be supported on $S$ and suppose $w'$ is any vector supported on $S' \in \mathcal{S}$. Then for any $T \supset S \cup S'$ we have
\[ \|w - v_T\| \le \|w' - v_T\| \]
where $v_T$ denotes the projection of $v$ onto vectors supported on $T$. (Informally, this means $w = \Proj_{\mathcal{S}} v_T$.)
\end{lemma}
\begin{proof}
By the Pythagorean Theorem,
\[ \|w - v\|^2 = \|w - v_T\|^2 + \|v_T - v\|^2 \]
since $v_T - v$ is supported outside of $T$ and $w - v_T$ is supported inside of $T$. The analogous identity also holds for $w'$. Hence
\[ \|w - v_T\|^2 = \|w - v\|^2 - \|v - v_T\|^2 \le \|w' - v\|^2 - \|v - v_T\|^2 = \|v - v_T\|^2 \]
where the inequality follows from the projection property, and this proves the result. 
\end{proof}
Next we recall a deviation inequality for the norm of a sub-Gaussian random vector, which follows from a net argument.
\begin{lemma}[Theorem 1.19 of \cite{rigollet2015high}]\label{lem:norm-bound}
Suppose that $X$ is a mean-zero random vector in $\mathbb{R}^n$ satisfying the sub-Gaussianity inequality
\[ \max_{w : \|w\|_2 = 1} \log \EE \exp \lambda \langle w, X \rangle \le \sigma^2 \lambda^2/2 \]
for some $\sigma > 0$.
Then with probability at least $1 - \delta$,
\[ \|X\|_2 \lesssim \sigma \sqrt{n} + \sigma \sqrt{\log(2/\delta)}. \]
\end{lemma}
\begin{lemma}\label{lem:process-bound}
Suppose $X : m \times n$ with $\hat{\Sigma} = \frac{1}{m} X^T X$ is $(kr,\beta)$-RIP and $\xi$ is a random vector in $\mathbb{R}^m$ with independent $\sigma^2$-sub-Gaussian entries, then with probability at least $1 - \delta$
\[ \max_S \frac{1}{\sqrt{m}} \|(X^T \xi)_S\|_2 \lesssim (1 + \beta) \sigma \sqrt{kr + \log(2/\delta)} \]
where the maximum ranges over all $k$-group-tree-sparse sets $S$, where the groups corresponding to the nodes of the tree have size at most $r$.
\end{lemma}
\begin{proof}
The vector $(X^T \xi)_S$ is a sub-Gaussian vector  with covariance matrix $\sigma^2 (X^TX)_{SS}$. By the RIP property we know $\frac{1}{m} (X^T X)_{SS} \preceq 1 + \beta$ and so
with probability at least $1 - \delta$
\[ \frac{1}{\sqrt{m}} \|(X^T \xi)_S\|_2 \lesssim (1 + \beta) \sigma \sqrt{kr} + \sigma\sqrt{\log(2/\delta)} \]
by Lemma~\ref{lem:norm-bound}.
Taking the union bound over the $e^{O(k)}$ possible supports $S$ given by Lemma~\ref{lem:num-tree-sparse} gives the result.
\end{proof}
\begin{lemma}\label{lem:iht-apdx}
Suppose $X : m \times n$ and $\hat \Sigma = \frac{1}{m} X^T X$ is $(3kr,\beta)$-RIP 
with $\beta < 1/3$ and $\mathbb{Y} = X w^* + \xi$ where $w^*$ is $k$-group-tree-sparse with respect to a tree with node groups of size at most $r$, 
and $\xi$ is a random vector in $\mathbb{R}^m$ with independent $\sigma^2$-subGaussian entries. Then Iterative Hard Thresholding with projection onto the set of $k$-tree-sparse vectors and $T = O(\log(\|w^*\|nm/\sigma))$ succeeds to recover $w$ such that
\[ \|w^* - w_t\| \lesssim \sigma \sqrt{\frac{kr + \log(2/\delta)}{m}} \]
with probability at least $1 - \delta$ over the randomness of $\xi$.
\end{lemma}
\begin{proof}
Given Lemma~\ref{lem:process-bound}, this result follows by adapting the analysis of IHT from \cite{blumensath2009iterative}; we include the detailed proof here for completeness. The main step in the analysis is a per-timestep inequality. Define $S$ to be $\supp(w_t) \cup \supp(w^*)$ and $v_t := (u_t)_S$, and observe that
\begin{equation}\label{eqn:projection-fact}
\|w_t - v_t\| \le \|w^* - v_t\| 
\end{equation}
as otherwise $w^*$ would be a $k$-tree-sparse vector closer to $u_t$ than the projection $w_t$: this is due to the Pythagorean Theorem identity
\[ \|w_t - u_t\|^2 = \|w_t - v_t\|^2 + \|v_t - u_t\|^2 \]
which follows from the fact that $w_t - v_t$ and $v_t - u_t$ have disjoints supports,
and the corresponding identity for $w^*$ which follows for the same reason.
Hence by the triangle inequality and \eqref{eqn:projection-fact},
\[ \|w^* - w_t\| \le \|w^* - v_t\| + \|v_t - w_t\| \le 2 \|w^* - v_t\|. \]
Since
\[ u_t = w_{t - 1} + \frac{1}{m} X^T X (w^* - w_{t - 1}) + \frac{1}{m} X^T \xi = w^* + (\frac{1}{m} X^T X - I)(w^* - w_{t - 1}) + \frac{1}{m}X^T \xi \]
we see that
\[ v_t - w^* = \left[(\frac{1}{m} X^T X - I)(w^* - w_{t - 1})\right]_S + \frac{1}{m}(X^T \xi)_S\]
we conclude from the RIP property applied to the set $S \cup \supp(w_{t - 1})$ of tree size at most $3k$ 
that
\[ \|w^* - w_t\|_2 \le 2 \|w^* - v_t\| \le 2 \beta \|w^* - w_{t - 1}\| + \frac{2}{m}\|(X^T \xi)_S\|.\]
Using $\beta < 1/3$ and $\frac{2}{m} \|(X^T \xi)_S\| \le \sigma\sqrt{k + \log(2/\delta)}$ by Lemma~\ref{lem:process-bound} and applying this guarantee inductively guarantees we achieve the desired error guarantee $\|w^* - w_T\|_2 = O(\sigma\sqrt{\frac{k + \log(2/\delta)}{m}})$.
\end{proof}
\section{Sparse Linear Regression with a Sparse Covariance}\label{sec:sparse-covariance}
In this setting, we consider the setting where the \emph{covariance} matrix $\Sigma$ is $d$-sparse instead of the \emph{precision} matrix, so that many pairs of coordinates of $X_0 \sim N(0,\Sigma)$  are uncorrelated with each other. In this setting, we use a randomized preconditioner based on a site percolation process on the support of $\Sigma$ and show it succeeds with probability $1/d^{O(k)}$ where $k$ is the sparsity of the weight vector, so running multiple times gives a polynomial time and sample efficient algorithm when $k \log d = O(\log n)$.
\begin{theorem}[Theorem 1 of \cite{krivelevich2015phase}]\label{thm:site-percolation}
There exists absolute constants $C,\epsilon_0 > 0$ such that the following is true.
Suppose that $G$ is a graph on vertex set $[n]$ where every vertex has maximum degree $d$.
Suppose that $U \subset [n]$ is formed by including each vertex of $G$ independently with probability $p$ (this is called site percolation). For any $\epsilon \in (0,\epsilon_0)$  and $p = (1 - \epsilon)/d$, the probability that the largest connected component of $U$ has more than $4\log(n)/\epsilon^2$ vertices is at most\footnote{The explicit upper bound on the probability is not stated but can easily be extracted from the proof.}  $C n^{-1/6}$.
\end{theorem}
Observe that under the site percolation process of Theorem~\ref{thm:site-percolation}, the probability that any fixed set $S$ with $|S| = k$ is contained in $U$ is $p^k$. Take $\epsilon = \epsilon_0/2$ so $p^{-k} = (cd)^k$ for some $c > 1$.
Therefore if $n^{1/6} = \Omega((cd)^k)$, it will take $O(\log(1/\delta) (c d)^k)$ many percolations to sample a $U$ containing $S$ and with largest component size $O(\log n)$.

Now suppose we want to solve a sparse linear regression where $\Sigma$ is $d$-sparse, $w$ is $k$-sparse, and $S$ is the support of $w$. Observe that for $U$ satisfying the conclusion of Theorem~\ref{thm:site-percolation} $\Sigma_{UU}$ is block-diagonal with blocks of size $O(\log n)$ corresponding to the connected components. This means that if $w$ is supported on $U$, we can recover $w$ by whitening the connected components and running the Lasso on data which now has isotropic covariance.
Using the analysis of the Lasso as above, we see that performing repeated site percolations and running the preconditioned Lasso gives a $poly((c d)^k, n,\log(1/\delta))$ time algorithm which achieves a rate of $\|\hat w - w^*\|_{\Sigma}^2 = O(\sigma^2 k^2\log(n/\delta)/m)$.
\end{document}